\documentclass[ms,sglanonrev]{informs4}
\RequirePackage{tgtermes}
\RequirePackage{newtxtext}
\RequirePackage{newtxmath}
\RequirePackage{bm}
\RequirePackage{endnotes}
\usepackage{shortcuts}
\usepackage{booktabs,tabularx,array}
\usepackage{enumitem}

\OneAndAHalfSpacedXII %

\usepackage{algorithm}
\usepackage{algpseudocode}
\usepackage{tikz}

\usepackage{amssymb,amsmath,mathtools}
\EquationsNumberedThrough    %
\TheoremsNumberedThrough     %
\ECRepeatTheorems  %
\newtheorem{setting}[theorem] {Setting}
\usepackage{cleveref}
\newcommand{\tspace}{\mathcal{T}}
\newcommand{\sumrt}{\sum_{t\in\tspace, r\in\{0,1\}}}
\newcommand{\ptmidrx}{p_{t\mid r}(X)}

\newcommand\norm[1]{\lVert#1\rVert}
\newcommand{\Xov}{\mathcal{X}_{\mathrm{ov}}}
\newcommand{\Xnov}{\mathcal{X}_{\mathrm{nov}}}

\newcommand{\indov}[1]{\mathbb{I}_{\mathrm{ov}}(#1)}
\newcommand{\indnov}[1]{\mathbb{I}_{\mathrm{nov}}(#1)}
\crefname{assumption}{assumption}{assumptions}
\usepackage{natbib}
 \bibpunct[, ]{(}{)}{,}{a}{}{,}%
 \def\bibfont{\small}%
\usepackage{booktabs}
\usepackage{todonotes}

\MANUSCRIPTNO{MNSC-0001-2024.00}

\begin{document}


\RUNTITLE{Optimal and Equitable Encouragement Policies}

\TITLE{Mind the Gap: Optimal and Equitable Encouragement Policies}

\ARTICLEAUTHORS{%
\AUTHOR{Angela Zhou}
\AFF{Department of Data Sciences and Operations, University of Southern California, \EMAIL{zhoua@usc.edu}}

} %

\ABSTRACT{%
In consequential domains, it is often impossible to compel individuals to take treatment, so that optimal policy rules are merely suggestions in the presence of human non-adherence to treatment recommendations. 
We study personalized decision problems in which the planner controls recommendations into treatment rather than treatment itself. Under a covariate-conditional no-direct-effect model of encouragement, policy value depends on two distinct objects: responsiveness to encouragement and treatment efficacy. This modeling distinction makes induced treatment take-up, rather than recommendation rates alone, the natural fairness target and yields tractable policy characterizations under budget and access constraints. In settings with deterministic algorithmic recommendations, the same model localizes overlap-robustness to the recommendation-response model rather than the downstream outcome model.
We illustrate the methods in case studies based on data from reminders of SNAP benefits recertification, and from pretrial supervised release with electronic monitoring. While the specific remedy to inequities in algorithmic allocation is context-specific, it requires studying both take-up of decisions and downstream outcomes of them.
}%

\FUNDING{}


\maketitle

\section{Introduction}

Data-driven decisions leverage individual-level data to learn heterogeneous treatment effects and can improve upon uniform decision rules for the population. Personalized targeting of promotions to customers can improve profits and longer-term revenues beyond one-size-fits-all campaigns \citep{simester2020targeting}. A crucial operational lever comprises of \textit{nudges and reminders}, \citep{hirano2000assessing,karlan2016getting}, 
and recent work shows that \textit{targeting} nudges can further improve outcomes under a lower budget \citep{von2025smart}.
A rapidly growing research literature studies personalized decision rules based on treatment effect heterogeneity to prescribe who \textit{ought} to receive different treatments, deployed in healthcare, e-commerce, and social and societal settings \citep{athey17,kt15,manski05,zhao2012estimating}. 

However, in many important settings, we cannot \textit{compel} individuals into taking treatment. Such mismatches between treatment assignment and realization are called non-compliance or non-adherence in the causal literature. In fact, non-adherence is inherent in \textit{encouragement designs} that randomize over \textit{encouragements}, or {recommendations} into treatment, when it is impossible to randomize the treatment itself due to a human-in-the-loop who makes the final decision, or other ethical concerns. 
For example, in e-commerce, although companies can assign different users to different visual interfaces of a website, they cannot {compel} users to sign up or subscribe. Instead, they can \textit{nudge} or \textit{encourage} users to sign up via promotions and offers. Encouragement designs are widely used in e-commerce to avoid withholding access to new features \citep{spotifyengineering2023} and in social settings because ``when evaluating a noncompulsory, entitlement program, researchers and practitioners cannot and should not force individuals to take up the program nor deny eligible individuals access to the program" \citep{heard2017real}.

Throughout, the decision-maker controls an encouragement or recommendation, not treatment uptake itself. 
Our analysis relies on a covariate-conditional exclusion restriction/assumption of no direct effect of encouragement on outcomes. Conditional on observed covariates, encouragement affects outcomes only through its effect on treatment take-up. This assumption encodes the ``nudge" nature of encouragements, and it enables us to conceptually separate the effects of nudges on treatment take-up from the effects of treatments on downstream outcomes. 
To interpret that decomposition in terms of heterogeneous treatment effects, we additionally invoke treatment unconfoundedness; under truly randomized encouragement, alternative IV-style interpretations are also available.
We therefore study policies that choose who to encourage so as to maximize policy value under the induced treatment behavior, optionally subject to budget and fairness constraints on induced treatment take-up or other groupwise impacts. 

Optimal personalized policies often are optimizing in the space of encouragements, not treatments. %
Marginal encouragement effects are often enough for operational decision-making.
For example, revenue management and pricing maximizes the expected profit, corresponding to the \textit{intention-to-treat} effect. 
Operational goals to make better decisions differ from scientific goals to estimate the specific causal effect of receiving treatment, which gives rise to more challenging instrumental variables estimands.
On the other hand, especially in the same consequential domains where it is impossible to compel treatment due to actions on and effects on people, decision-makers are interested separately in tracking \textit{who ultimately receives treatment} (take-up) and \textit{who benefits from receiving treatment} (treatment efficacy).  
For example, recent papers optimize for personalized allocation of job training services \citep{kallus2023treatment,athey2021policy,bertsimas2019optimal}, leveraging data from large randomized experiments of job training programs \citep{imbens2024lalonde}. However, recent deployments of algorithmic job profiling in Europe have exposed key fairness concerns. 
A job profiling tool in Austria gave women and non-Austrian citizens lower predictions (hence lower priority for job training) and was ultimately scrapped by the data protection authority \citep{derstandardDatenschutzbehrdeKippt}. The federal court overturned this decision, but the data had already been destroyed. The labor agency emphasized the predictive scores would only inform caseworker decisions, not dictate them. Would disparities in displayed predictive scores translate in practice also into discriminatory decisions for women, and worse labor market outcomes? Answering this requires separating how recommendations affect downstream human decisions from how those decisions affect outcomes.

Encouragement policy design is therefore inherently multi-objective: decision-makers may care about downstream outcomes, treatment take-up, encouragement burden, and disparities across groups. These tradeoffs may be represented either through penalized objectives or through constrained formulations. 

The distinction between encouragements versus the treatments they nudge is ubiquitous in practice, arising whenever treatment assignment relies on intermediate human decisions. 
For example, price affects both demand and revenue. Demand fairness could be relevant \citep{cohen2022price,kallus2021fairness}, while personalized pricing could expand access \citep{karlan2008credit}.
In healthcare, nonadherence is a major issue and can be significantly affected by social determinants of health \citep{miao2024identifying}. We later discuss a case study of targeted outreach for take-up of social services, so we discuss it here in more detail.


\begin{example}[Takeup of social services and digital outreach]
In public programs such as Supplemental Nutrition Assistance Program (SNAP), many eligible individuals do not receive or retain benefits because of operational frictions in information, outreach, and benefits renewal (re-certification). These frictions are naturally modeled as barriers to take-up rather than as differences in treatment efficacy, which is fixed. Encouragements such as reminders, outreach, or navigation support are typically the only operational lever available to nonprofits and agencies to improve program enrollment among the likely-eligible. But positive average treatment effects hide groupwise heterogeneity where worse-off individuals may be less likely to act on nudges and enroll into beneficial treatment enrollment \citep{finkelstein2019take}.
\end{example}
Our SNAP case study later returns to this setting and asks whether disparities under naive targeting are driven primarily by treatment benefit or by responsiveness to encouragement. We show that naive personalization can generate substantial access fairness gaps because responsiveness to encouragement is less aligned with treatment benefit for nonwhite recipients.

Thus far, we have focused on modeling  \textit{physically randomized} encouragements such as those we have just described. However, our model can be extended to model the impacts of \textit{algorithmic advice} on \textit{humans-in-the-loop}. In this algorithmic advice setting, the covariate-conditional exclusion restriction is still plausible, but the ``overlap" assumption is violated: Randomized (or as-if randomized) algorithmic encouragement or nudge recommendations are typically fixed functions of covariates. Crucially, our model's decomposition of treatment take-up and effects on outcomes renders recent approaches based on robust optimization \citep{ben2021safe} less conservative. 

\begin{example}[Algorithmic advice]
    Doctors prescribe treatment on the basis of algorithmic recommendations \citep{lin2021does}, managers and workers combine their expertise to act on the basis of algorithmic decision support \citep{bastani2021improving}, and in the social sector, caseworkers assign individuals to benefit programs on the basis of recommendations from risk scores that support triage \citep{de2020case,green2019disparate,yacoby2022if}. 
\end{example}

A central managerial point is that disparities in encouragement settings can arise from different causal channels, including treatment efficacy and responsiveness to encouragement. Because these channels imply different operational responses, we advocate an audit-to-remedy workflow: first audit unconstrained encouragement policies by separately examining take-up and efficacy, then choose the appropriate intervention.
We study personalized decision problems in which the planner controls recommendations into treatment rather than treatment itself. Under a covariate-conditional no-direct-effect model of encouragement, policy value depends on two distinct objects: responsiveness to encouragement and treatment efficacy. This modeling distinction makes induced treatment take-up, rather than recommendation rates alone, the natural fairness target and yields tractable policy characterizations under budget and access constraints. In settings with deterministic algorithmic recommendations, the same model localizes overlap-robustness to the recommendation-response model rather than the downstream outcome model.
In the supervised-release application, local changes to simple interpretable scorecards can more than halve current disparities in supervision at comparatively little value loss. Data-driven encouragement policies offer additional degrees of freedom for improving fairness upon the status quo.

\section{Related Work}
In the main text, we briefly highlight the most relevant methodological and substantive work and defer additional discussion to the appendix.

\paragraph{Fairness analysis of optimal policy learning and responses to algorithmic advice}

\citet{chohlas2021learning} puts forth a consequentialist framework for parity-aware policy design, emphasizing elicitation of preferences on efficiency and resource parity in a contextual bandit problem. In our work, the planner controls recommendations rather than treatment itself and disparities in induced take-up introduce another objective. The elicitation mechanisms of \citet{chohlas2021learning} could similarly be used here to guide trade-offs on multiple objectives. \citet{ge2025rethinking} studies how to robustly improve the algorithmic fairness of classifiers under downstream human decisions without any information on compliance. In contrast, our later extension to algorithmic advice introduces differential causal effects of receiving treatment, and throughout we assume observational data access to prior human-made decisions in response to classifiers. \citet{mclaughlin2022algorithmic,gillis2021fairness} study the fairness of machine-assisted human decisions and develop a principal-agent model studying how algorithmic advice might shift beliefs in risk and preferences. In contrast, we are not micro-founded as to why disparities in compliance arise, which is an important direction of future work. Further, our assumptions are complementary and assume decision-makers are ``set in their ways" responding to algorithmic risk labels: empirically supported by prior studies \citep{pruss2023ghosting,imai2020experimental} in the criminal justice setting of our later case study. A rapidly growing literature in operations studies \textit{algorithmic advice}, at times from a behavioral focus (but not specialized to causal effects). \citet{grand2024best} also shares a focus on leveraging adherence to optimize \textit{advice} rather than \textit{decisions} - however they adopt a dynamic model with updating. 

\paragraph{Optimal encouragement designs/policy learning with constraints.}
There is extensive literature on off-policy evaluation and learning, empirical welfare maximization, and optimal treatment regimes \citep{athey2021policy,zhao2012estimating,manski05,kt15}.
 The most closely related work in terms of problem setup is the formulation of ``optimal encouragement designs" in \citet{qiu2021optimal}, but they focus on knapsack resource constraints. %
\citet{sun2021treatment} has studied uniform feasibility in constrained resource allocation, but without encouragement or fairness. \citet{ben2021safe} studies robust extrapolation in policy learning from algorithmic recommendation, exactly implicitly optimizing in the space of encouragements, but not fairness concerns. By introducing separate consideration of treatment take-up and efficacy, we can substantially reduce the conservativity of robust formulations and limit extrapolation to behavioral response to recommendations rather than unknown treatment effects. 

\paragraph{Fair pricing in operations: take-up (demand) vs. revenue}
We view our perspective on fairness constraints as translating the normatively important role of \textit{demand fairness} from fair pricing analysis in operations management to settings beyond revenue outcomes alone and our \textit{treatment access fairness} framework. \cite{cohen2022price} study revenue and other implications of different potential fairness notions in pricing, including price, demand, consumer surplus, and no-purchase valuation under a linear demand model. \cite{kallus2021fairness} study fairness in pricing, from a nonparametric, contextual perspective. Although fair pricing analysis often hinges specifically on the functional form of revenue and demand, its importance highlights the analogous importance of considering demand fairness in other decision-impact settings. 

\paragraph{Algorithmic fairness in operations}
Another rich line of work studies fairness in dynamic resource allocation \citep{manshadi2021fair,salem2026algorithmic,freund2023group,banerjee2023online}. Our encouragements model points out that allocating resources need not ensure that individuals benefit: potentially due to concerning persistent access or take-up gaps. \citet{garg2020dropping} study how access gaps to test score information can change normative policy recommendations, \citet{andrews2022modeling} studies access gaps in resource allocation, and \citet{garg2025heterogeneous,herd2019administrative} documents access gaps and administrative burdens in public-facing mechanisms.

\section{Encouragement model and identification}\label{sec-problemsetup}

We introduce the setup, key assumptions, and the resulting characterization of encouragement policies that underlies the analysis below, before discussing extensions via sensitivity analysis. The central modeling assumption is that recommendation affects outcomes only through treatment take-up. Under this encouragement model, policy value depends jointly on responsiveness to encouragement and treatment efficacy.

We work in the potential outcomes framework for causal inference \citep{rubin1980randomization}, which posits that each individual is endowed with a vector of potential outcomes, one for each value of treatment. 
We define the following: 
\begin{itemize}
    \item recommendation flag
$R \in\{0,1\}$, where $R=1$ means \textit{encouraged}/\textit{recommended} (we use the terms \textit{encouragement }and \textit{recommendation} interchangeably);  
\item  treatment $T(r) \in \{0,1\}$, where $T(r)=1$ indicates that the treatment decision was $1$ when the recommendation was $r$; and 
\item outcome $Y(t(r))$, the potential outcome under encouragement $r$ and treatment $t$. 
\end{itemize}

We let $\pi$ denote a \textit{personalized encouragement recommendation rule}. In general, for theoretical characterizations of optimal policies, we allow $\pi$ to depend on both covariates and protected attributes, and write
\[
\pi(x,a):=P(R=1\mid X=x,A=a), \qquad \pi_r(x,a):=P(R=r\mid X=x,A=a),
\]
When discussing practically implementable policy learning, however, we typically restrict attention to covariate-only policies $\pi(x)$. We seek to evaluate and optimize the value under different candidate policies $\pi$, which differ from the distribution of encouragements in our observed data.

Finally, for the purposes of assessing algorithmic fairness, we define the group-level protected attribute information $A \in\{a,b\}.$ The protected attribute $A$ designates group information by which it may be relevant to decompose groupwise performance disparities, such as race, gender, or age. 
\paragraph{Our multi-objective criteria: treatment effectiveness and take-up.}
The average encouragement effect (AEE) is the difference in average outcomes if we refer everyone vs. no one, while the encouragement policy value $V(\pi)$ is the population expectation induced by the outcomes and treatment  with recommendations following the policy distribution. 
\[ AEE \coloneqq \E[Y(T(1)) - Y(T(0))], 
\qquad Y(\pi) \coloneqq \sum_{r \in\{0,1\}} \pi_r(X, A) Y(T(r)), 
\qquad 
V(\pi) \coloneqq \E[Y(\pi)]
\]

Regarding fairness, we are concerned about disparities in policy value and treatment take-up (resources or burdens) across different groups, denoted $A \in\{a,b\}$. (For notational brevity, we may generically discuss identification/estimation without additionally conditioning on the protected attribute.) Because the policy controls encouragement recommendations $R$ rather than treatment $T$ directly, fairness can be
evaluated on several margins. Later, our primary focus is introducing awareness of disparities in access, take-up or demand for treatment.
\begin{equation}
    \Delta_T(\pi)\coloneqq
\E[T(\pi) \mid A=a]-  \E[T(\pi)\mid A=b]
\end{equation}
This is the natural access disparity in encouragement settings, because the planner controls recommendation rather than treatment directly. Other disparity measures that are useful additional diagnostics include: groupwise differences in final outcomes,
treatment take-up, encouragement rates, and improvement from targeting --- we discuss these in the context of our case studies later on.



In this paper, we first recommend auditing unconstrained policy optimization for disparities. Auditing can surface undesirable disparities, and we also consider multi-objective and budget-constrained policy optimization problem, where $B$ is a budget fraction of the population that can receive treatment:  
\begin{align*}
&\textrm{Auditing: } &&\Delta_T(\pi^*), \text{ where } \pi^* \in \arg\max \E[V(\pi)]\\
&\textrm{Fair policy optimization: }    &&\max \left\{ \E[Y(\pi)] - \lambda \abs{\E[T(\pi) \mid A=a]-  \E[T(\pi)\mid A=b]}  \colon \E[\pi(X)] \leq B \right\} 
\end{align*}
The budget on encouragement is part of our default formulation since resource-constrained settings benefit the most from personalized targeting.
Though the fairness-penalized and multi-objective version is practical, for theoretical analysis we often study the fairness-constrained problem. Since fairness is highly contextual, we generally don't recommend trying to impose all disparity measures as constraints simultaneously. Instead, we use disparities as diagnostic measures to identify different gaps and to inform potentially different operational remedies, as illustrated in our later case studies.

\paragraph{Causal identification and estimation: key assumptions}
We discuss identification and estimation on the basis of the following recommendation propensity $e_r$, treatment propensity $p_{t\mid r}$, and outcome model $\mu_t$: 
\begin{align*}
e_r(X,A) &\coloneqq P(R=r\mid X,A),
\qquad 
p_{t\mid r}(X,A)\coloneqq P(T=t\mid R=r,X,A),\;\;\\
\mu_t(X,A) &\coloneqq \E[Y\mid T=t, X,A] \text{  (asn.} \ref{asn-exclusionrestriction})
\end{align*}

An important object for interpretation, appearing in the policy learning characterizations, is the effect of encouragement on treatment take-up. 
$$\kappa(x,a)
\coloneqq  p_{1\mid 1}(X,A) - p_{1\mid 0}(X,A) $$

The following causal assumptions comprise our primary model, and our later extension to algorithmic advice is based on sensitivity analysis to some of them. 

\begin{assumption}[Consistency and stable unit treatment values (SUTVA)]\label{asn-consistency}
    $Y_i = Y_i(T_i(R_i)).$
\end{assumption}

\begin{assumption}[Exclusion / no direct effect of recommendation]\label{asn-exclusionrestriction}
$Y_i(t(r)) = Y_i(t), \forall i,t,r$
\end{assumption} 

\begin{assumption}[Randomized recommendations]\label{asn-randomized-recommendations}
$R \indep \{T(0),T(1),Y(t(0)),Y(t(1)) : t\in\mathcal{T}\}\mid X,A.$
\end{assumption}

\begin{assumption}[Treatment unconfoundedness given recommendation]\label{asn-unconfoundedness}
$(Y(1),Y(0)) \indep T \mid R,X,A.$
\end{assumption}

\begin{assumption}[Stable responsivities under new recommendations]\label{asn-stableresponsivity}
$P(T=t\mid R=r,X)$ remains fixed from the observational to the future dataset. 
\end{assumption}
Our key assumption beyond standard causal inference assumptions is \cref{asn-exclusionrestriction}, the conditional exclusion restriction/no direct effect assumption that the encouragement/recommendation has no causal effect on the outcome beyond its effect on the treatment. This assumption is typically mild for encouragements, although it requires that all of the covariate information that is informative of downstream outcomes is measured. 

We begin with \Cref{asn-unconfoundedness} first to highlight the two distinct mechanisms at play, compliance and treatment response, and discuss potential violations later on. 

\Cref{asn-stableresponsivity} is a structural assumption that limits our method to most appropriately reoptimize over small changes to existing algorithmic recommendations. For example, $p_{0\mid 1}(x)$ (disagreement with algorithmic recommendation) could be a baseline algorithmic aversion. Not all settings are appropriate for this assumption. We do not assume microfoundations on how or why human decision-makers deviate from algorithmic recommendations; rather, we take these patterns as given. 

Under these assumptions, we obtain causal identification of the policy value via outcome modeling (sometimes called regression adjustment). Causal identification rewrites the causal estimand in terms of probability distributions estimable from data. The argument follows by applying the conditional exclusion restriction and consistency but, crucially, does not rely on overlap. 

\begin{proposition}[Regression adjustment identification]\label{prop-identification}
\begin{align*}
\E[Y(\pi)]
&\textstyle = \sum_{t \in \tspace, r\in\{0,1\}}
\E[\pi_r(X) \mu_{t}(X)p_{t\mid r}(X) ] .
\end{align*} 
\end{proposition}

We first also assume overlap in recommendations and treatment. 
\begin{assumption}[Overlap]\label{asn-overlap}
    $\rho_r \leq  e_r(X,A) \leq 1-\rho_r; \;\; \rho_t \leq  p_{t\mid r}(X,A) \leq 1-\rho_t \text{ and } \rho_r, \rho_t > 0 $.
\end{assumption}

We consider two problem settings, which model different situations and differ based on the strength of the overlap assumptions. 
\begin{setting}[Setting 1: Randomized encouragement]\label{setting-encouragement}
    The encouragement $R$ is (as-if) randomized and satisfies overlap (\Cref{asn-overlap}). 
\end{setting}Then, $R$ can be interpreted as the ITT or prescription, whereas $T$ is the actual realization thereof. \Cref{setting-encouragement} models nonadherence situations where decision-makers can target encouragements but not direct receipt of the treatment itself. 

\paragraph{Our perspective: Fairness constraints in causal intention-to-treat (ITT) analyses.}
We consider covariate-conditional intention-to-treat causal effect under the assumption of treatment unconfoundedness. We evaluate encouragement policies by their population-level value and by
groupwise impacts under the treatment behavior they induce. Thus, treatment
take-up is central for fairness constraints, while utility is assessed through
policy value; we do not pursue IV or principal-stratum estimands.


\subsection{Policy characterizations}

Next, we focus on basic characterizations of the optimal encouragement policy under a constrained resource budget, where targeting becomes particularly useful. 
We begin by considering \textit{naive targeting on the intention-to-treat effect}, i.e. viewing the encouragement as a treatment itself. The optimal policy thresholds on what we call the \textit{encouragement score}: 
\begin{align*}
    s(X, A)&\coloneqq\E[Y \mid R=1, X, A]- \E[Y \mid R=0, X, A]\\
    &=\E[Y(T(1))-Y(T(0)) \mid X, A]. 
    \tag{under \cref{asn-consistency,asn-exclusionrestriction}} \\
        &=\E[(Y(1)-Y(0))(T(1)-T(0))  \mid X, A]. 
    \tag{ \cref{asn-exclusionrestriction}}
\end{align*}
Define a generic measure of treatment efficacy $\theta(X,A)$, akin to a conditional Wald estimand, that can be used to define a generic decomposition of the encouragement score as follows:
$$\theta(X,A) \coloneqq \frac{s(X,A)}{\E[T(1)-T(0)  \mid X, A]}, \;\; s(X,A) \coloneqq \theta(X,A)\,(p_{1\mid 1}(X,A) -p_{1\mid 0}(X,A)).$$ 

We define the heterogeneous treatment effect $\tau(X,A) = \E[Y(1) - Y(0) \mid X,A].$ Crucially, \Cref{asn-unconfoundedness} enables interpretation of the treatment efficacy part of the decomposition of $s(X,A)$ as a heterogeneous treatment effect, i.e. $\theta(X,A) = \tau(X,A)$, so that 
\begin{align*} s(X,A) = \E[Y(1)-Y(0)\mid X,A]\,(p_{1\mid 1}(X,A) -p_{1\mid 0}(X,A))=\tau(X,A)\kappa(X,A)
\tag{\Cref{asn-unconfoundedness}, \Cref{asn-exclusionrestriction}}
\end{align*} We study decompositions of the encouragement score to interpret optimal policies, where naive ITT targeting by default optimizes it.

\begin{proposition}[Quantile-threshold optimality]Assume the distribution of $s(X,A)$ is continuous (no atoms). Consider optimizing the ITT effect of encouragements under a global encouragement budget: 
$\pi^*_\beta(x,a) \in \arg\max_\pi \{\E[Y(\pi)] \colon \E[\pi(X,A)] \leq \beta \} $
Let $t_\beta \coloneqq \max\!\left\{\inf\{z : P(s(X,A)\le z)\ge 1-\beta\},\,0\right\}$.
Then the optimal policy is a threshold policy at the top-$\beta$ quantile of the encouragement score (so long as it is positive):
\begin{equation}
    \pi^*_\beta(x,a) \;=\; \mathbf{1}\!\left\{\, s(x,a) \ge t_\beta\,\right\}\label{eqn-thresh-policies}
\end{equation}
\end{proposition}

Threshold policies on encouragement score solve budget-constrained naive targeting; later we study properties of such threshold policies.

\subsection{Beyond the base assumptions }

\paragraph{Beyond recommendation overlap (\Cref{asn-overlap})}
Although we begin with these base assumptions, we can handle violations of certain of these assumptions with separate robustness checks. The most central of these robustness analyses is our extension to \textit{algorithmic advice} by allowing for violations of \Cref{asn-overlap}, for which we discuss robust optimization 
\begin{setting}[Setting 2: Algorithmic recommendation]\label{setting-algorec}
       The algorithmic recommendation $R$ is the output of a predictive model and does not satisfy \Cref{asn-overlap}. 
\end{setting}
We later extend our methods to the second setting, where $R$ does not satisfy overlap in recommendation but there is sufficient randomness in human decisions to satisfy overlap in treatment. For example, such recommendations arise from binary \textit{high risk}/\textit{low risk} labels of classifiers since decisions in consequential domains are rarely automated but rather are used to inform humans-in-the-loop, who decide whether to assign treatment. 

\paragraph{Beyond treatment unconfoundedness (\Cref{asn-unconfoundedness})}
It may look strange to reference the heterogeneous treatment effect in a paper on encouragement designs/adherence. The key is our \Cref{asn-unconfoundedness}, on unconfounded treatments. Other papers making similar assumptions include \citet{rahier2021individual} in e-commerce, \citet{imai2010identification} in mediation analysis, and \citet{rudolph2020defining}. 

\begin{table}[t]
\centering
\small
\setlength{\tabcolsep}{5pt}
\renewcommand{\arraystretch}{1.08}
\begin{tabularx}{\linewidth}{
>{\raggedright\arraybackslash}p{2.8cm}
>{\centering\arraybackslash}p{3.2cm}
>{\raggedright\arraybackslash}X}
\toprule
Framework & Score decomposition ($s(x,a)=$) & Interpretation of multiplier \\
\midrule
Unconfounded treatments (A4)
&
$\displaystyle
\kappa(x,a)\tau(x,a)$
&
Under \Cref{asn-unconfoundedness}, $\displaystyle \tau(x,a)=\E[Y(1)-Y(0)\mid X=x,A=a]$,
the conditional treatment effect.
\\[0.2em]

IV + monotonicity
&
$\displaystyle
\kappa(x,a)\tau_c(x,a)$
&
$\displaystyle \tau_c(x,a)=\E[Y(1)-Y(0)\mid T(1)>T(0),X=x,A=a]$,
the conditional LATE / CLATE.
\\[0.2em]

Wang--Tchetgen Tchetgen 
&
$\displaystyle 
\kappa(x,a)\theta(x,a)$
&
$\displaystyle \delta_W(x,a)=\Delta_Y(x,a)/\kappa(x,a)$,
the conditional Wald estimand; under A5.a or A5.b it coincides with the conditional ATE.
\\
\bottomrule
\end{tabularx}
\caption{Alternative interpretations of $\theta(x,a)$ in the encouragement score
$s(x,a)$. 
Full assumptions are given in the appendix.}
\label{tab:alt-identification-interpretations}
\end{table}


\Cref{asn-unconfoundedness} may be violated due to the role of private information governing how humans make decisions. If one objects to our treatment unconfoundedness assumption, there are a few options. An alternative route available when the recommendation is truly randomized is to interpret the recommendation \(R\) as an instrumental variable for treatment \(T\). In that case, a similar decomposition holds but each of the treatment-take-up vs. treatment-effectiveness objects has a slightly different interpretation, outlined in \Cref{tab:alt-identification-interpretations}. See \Cref{apx-addldisc-identification} in the Appendix for more details.

In our Setting 2 (Algorithmic advice), such alternative identification approaches are not available. However, our paper focuses on data- and covariate-rich settings. It may be plausible that observed covariates explain most, but not all of selection into treatment - so that sensitivity analysis approaches \citep{kallus2021minimax,yadlowsky2018bounds} can robustly optimize over ``last-mile" unobserved confounders. We illustrate this approach in our algorithmic advice case study on supervised release in \Cref{subsec-supervised-release} and how to conduct such a sensitivity analysis in the Appendix.

\section{Audit: are disparities driven by take-up or efficacy?}\label{sec-fairness-analysis}

In encouragement designs, value is generated through two distinct causal channels:
encouragement changes treatment uptake, and treatment uptake changes outcomes.
As a result, a policy that targets on the individual encouragement effect can generate
group disparities for two substantively different reasons. First, groups may differ in
treatment benefit. Second, groups may differ in responsiveness to encouragement. These mechanisms have different normative
interpretations and different operational remedies, so we analyze them separately.

We proceed in three steps. 
First, we discuss how the relationship between treatment benefit and compliance is central to quantifying the improvements of personalized targeting, and therefore why separately analyzing treatment take-up and efficacy is crucial to analyze mechanisms underlying disparities in downstream improvements. 
Second, we show that even when treatment effects are distributionally identical across
groups, naive ITT targeting can still generate persistent disparities through
compliance-only covariates. Finally, we illustrate the value of auditing both take-up and efficacy separately in a real-world case-study where we highlight that such phenomena do appear empirically.

\subsection{The joint dependence of treatment effects and compliance governs policy value improvements from targeting.}


To see why it is important to distinguish between treatment efficacy and treatment take-up, consider a healthcare setting where patients are prescribed an overall beneficial medication, whose effects may nonetheless vary, but is potentially burdensome to obtain and administer. Suppose that a personalized prescription policy yields lower encouragement value for a particular group. This disparity could arise because the treatment is less effective for that group, or because patients are less likely to adhere to the prescription when encouraged. These two mechanisms imply fundamentally different responses: differences in treatment efficacy call for revisiting the underlying science of the treatment, whereas differences in take-up point to operational interventions such as adherence support, outreach design, or access improvements.

Targeting is most effective when individuals who benefit most are also those most responsive to encouragement. But there can be groupwise heterogeneity in both. We breakdown the \textit{groupwise-improvement from personalized ITT targeting} to highlight different sources of disparities. Choose $\pi_0 = 0$, i.e. encourage no-one. For example, this could be the status quo without any encouragements.
Then, for a group $A=a$, we can define the improvement and a decomposition into marginal effects and the \textit{conditional covariance} between compliance and treatment efficacy scores:
\begin{align*}
&{Imp}_a(\pi^*;0)
\coloneqq 
\E[Y(\pi)-Y(0) \mid A=a] 
=\mathbb{E}\!\left[s_a\,\mathbb{I}\{s_a\ge t^*\}\mid A=a\right] \nonumber
\\
&= 
P\left(s_a \geq t^* \mid A=a\right)\left(E\left[\tau_a \mid s_a \geq t^*, A=a\right] E\left[\kappa_a \mid s_a \geq t^*, A=a\right]+\operatorname{Cov}\left(\tau_a, \kappa_a \mid s_a \geq t^*, A=a\right)\right)
\end{align*}



The conditional covariance term reflects the contribution of the joint dependence between treatment effects and compliance. It is zero within the encouraged stratum if treatment benefit and compliance are uncorrelated, and positive when individuals who benefit more from treatment also tend to be more responsive to encouragement. Thus, targeting is especially valuable when treatment efficacy and responsiveness are aligned. More generally, groups can differ in the gains they derive from personalized targeting because of differences either in the marginal levels of treatment benefit and compliance or in the alignment between them.


\paragraph{A stylized access-driven model of encouragement disparities.}
Observed disparities in encouragement effects or improvements from personalized targeting can arise under very different mechanisms. We give a \textit{possibility result} of low value-fairness trade-off arising due to naive ITT targeting exploiting ``compliance proxies", covariates that affect compliance alone but not treatment effects. 

To concretize, consider a simple model with covariates partitioned as $(X^s, X^\kappa, X^\tau)$ where $X$ is a shared factor affecting both compliance and heterogeneous treatment effect, $X^\kappa$ modifies compliance alone, and $X^\tau$ modifies the heterogeneous treatment effect alone. The protected attribute $A$ has no direct effect on compliance or the heterogeneous treatment effect, but $A$ impacts the distribution of a compliance-relevant covariate $X^\kappa$, i.e. $X^\kappa \mid A$ differs. 
  
To interpret this example, $X^\kappa$ could be zip code that affects location and therefore distance to medical services and adherence. With other health information about a person, zip code wouldn't plausibly biologically affect treatment effectiveness, but could affect adherence/compliance. Such compliance differences alone can generate persistent disparities in encouragement and treatment take-up. 
Thus, naive targeting may reward ease of nudging rather than treatment benefit, especially when it exploits covariates that proxy for access or adherence rather than efficacy. The resulting disparities can be normatively concerning when lower responsiveness reflects structural access barriers rather than lower need. Naive ITT targeting favors individuals who are easier to reach, not necessarily those for whom treatment is most valuable, and therefore could entrench structural disparities. 

In \Cref{apx-addldisc-fairnessanalysis} we analytically characterize the favorable fairness-value tradeoff in this instance. A simple proxy-blind targeting approach, i.e. optimizing over rules that do not use compliance proxy information, nearly eliminates group-wise encouragement disparity from $0.1232$ to $0.0088$ with a modest value loss of $0.0031$ (1.4\%). 

This stylized example shows that disparities under naive ITT targeting can arise entirely from differential compliance, even when treatment benefits are identical across groups, and therefore that fairness adjustments could nearly eliminate such disparities at potentially low tradeoff in overall value. Next, we apply such diagnostics on a real encouragement setting, where we see similar phenomena show up in real-world data. 
\subsection{Case study: Text message reminders for SNAP recertification}
We illustrate the above diagnostic analysis of compliance and treatment benefit in a case study drawing on a randomized-controlled trial of text message reminders for SNAP recertification. We have two purposes: first, we ask whether racial disparities in the encouragement score \(s(X)=\tau(X)\kappa(X)\) are driven primarily by treatment benefit or by compliance. Second, we illustrate how these differences translate into disparities under the actual threshold policies induced by naive ITT targeting.
\paragraph{Background}
Each year, households receiving SNAP (food assistance) must complete a short \emph{recertification interview} to verify continued eligibility. Missing this interview results in termination of benefits, even for households that otherwise remain eligible. We study a pilot program in San Francisco in which clients could \emph{opt in} to receive a text-message reminder about the recertification deadline \citep{homonoff2021program}. The reminder encouraged participants to complete the required interview—our effective \emph{treatment}—so that their benefits would continue. Since many who fail to recertify are in fact eligible, increasing interview attendance often improves benefit continuity and household welfare. 

We reinterpret the text message reminder as an \emph{encouragement} \(R\) affecting the take-up of treatment \(T\) (interview completion), which in turn influences the outcome \(Y\) (next-year SNAP benefits). We focus on heterogeneity in both the reminder’s compliance effect and the treatment effect of interview completion, with particular attention to racial disparities in their alignment. 



\paragraph{Step 1: Descriptive diagnostics: heterogeneity in treatment benefit and compliance, and their alignment.}
We begin by comparing the groupwise distributions within each race (binarized to white $A=0$ or nonwhite $A=1$) in \Cref{fig:sfhsa-density-race} of the heterogeneous treatment effect $\tau(X)$ the heterogeneous compliance effect $\kappa(X)$ and the encouragement score which is their product and which gives the ranking for optimal budget-constrained encouragement targeting rules. 

\Cref{fig:sfhsa-density-race} shows these descriptive diagnostics. The first plot shows the treatment effect heterogeneity: The distribution is wider for nonwhite than for white beneficiaries; i.e., it has more mass at higher magnitudes of $\tau$ and less at lower magnitudes of $\tau$. However, we see the \textit{opposite} effect on compliance: While the results are relatively weak, they are bimodal, with a substantially greater probability density at higher compliance treatment effects for white than for nonwhite beneficiaries. One reason for this could be the language barrier for non-English-speaking applicants. The last plot shows the heterogeneity in the encouragement effect: the distribution of the \textit{heterogeneous encouragement effect} for white beneficiaries is wider than that for nonwhite ones. These diagnostics show that group-level differences in nudgeability dominate in disparities in the encouragement score. 


Another source of potential disparities in improvement from personalized targeting arises from \textit{differential alignment} between treatment take-up and efficacy. We quantify this alignment with the Spearman's $\rho_{SP}$ rank correlation coefficient. 
Among white beneficiaries only, the rank correlation coefficient between heterogeneous treatment and compliance effects is $0.68$, while it is $0.45$ for nonwhite beneficiaries. While both correlations are statistically significantly positive, the relationship is \textit{weaker} for nonwhite than for white beneficiaries. This means that non-whites who benefit the most from treatment may not be the most responsive to nudges into it. Spearman correlation is not itself the conditional covariance term in the improvement decomposition---it is rank-based, scale-free, and not conditioned on selection---but it is a useful threshold-free diagnostic proxy for disparities in policy value and/or improvement due to differential alignment.

\begin{figure}
    \centering
    \includegraphics[width=0.24\linewidth]{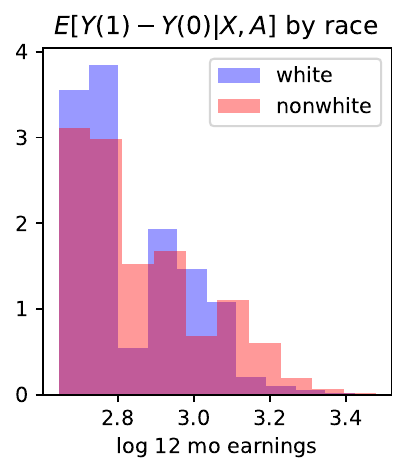}\includegraphics[width=0.33\linewidth]{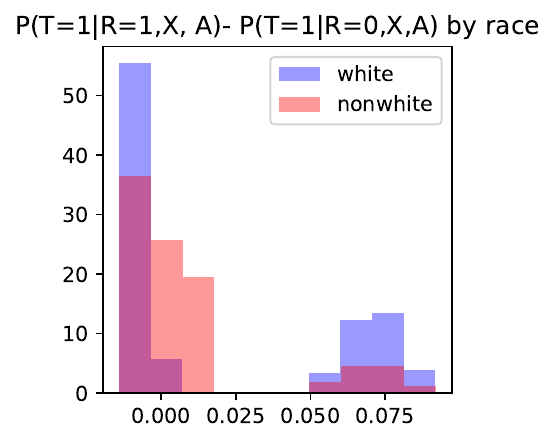}\includegraphics[width=0.33\linewidth]{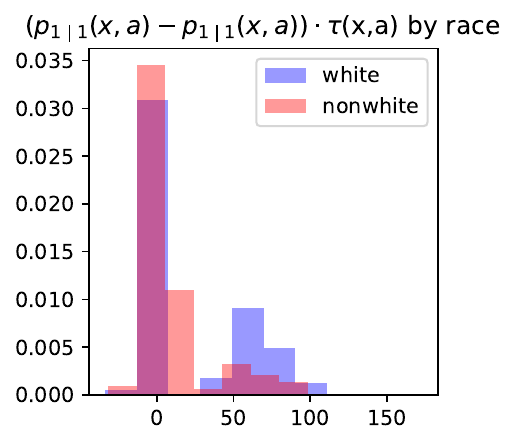}
    \caption{SNAP recertification case study: Heterogeneous treatment effects $\tau$, compliance, and their product (heterogeneous encouragement effects), density distribution plots by race.}
    \label{fig:sfhsa-density-race}
\end{figure}

\begin{figure}
    \centering
      \includegraphics[width=0.33\linewidth]{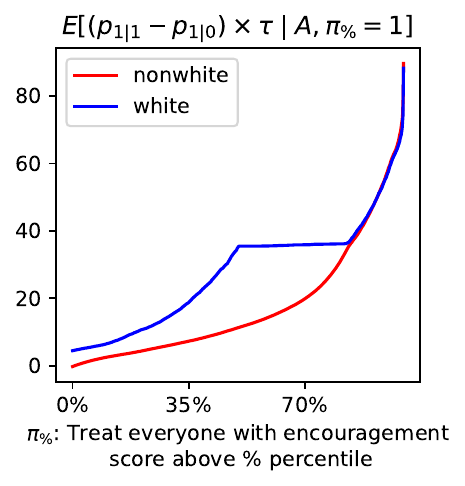}\includegraphics[width=0.33\linewidth]{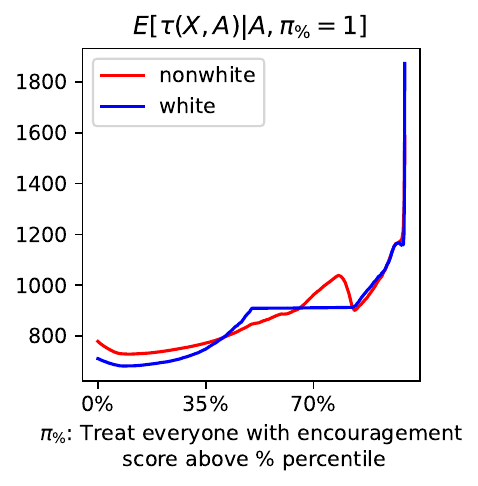}\includegraphics[width=0.33\linewidth]{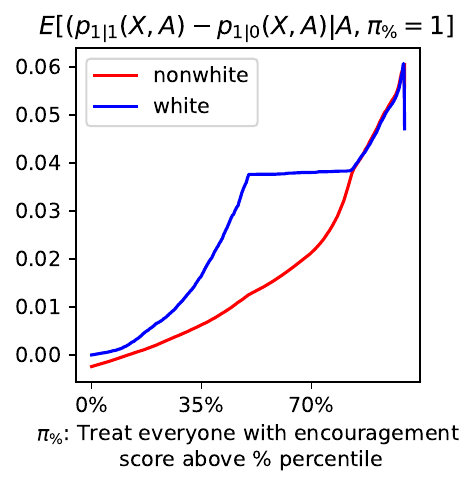}
    \caption{SNAP recertification case study. Each figure indicates the performance metric (conditional encouragement effect, i.e., compliance score $\times$ heterogeneous treatment effect; heterogeneous treatment effect; and compliance score induced by a threshold policy that thresholds on $\left(p_{1 \mid 1}-p_{1 \mid 0}\right) \times \tau$).} 
    \label{fig:sfhsa-cdfplot-race}
\end{figure}

\paragraph{Step 2: Investigate disparities under naive ITT targeting.}
We next investigate how these differences result in disparities in the actual allocations induced by naive ITT targeting. To do so, we consider threshold policies as in \Cref{eqn-thresh-policies} of the form $\pi_{\%}(X)=\mathbb{I}\{s(X)\ge \% \text{percentile of } s\},$ since they solve naive-ITT budgeted targeting. %

Whereas Step 1 uses unconditional, threshold-free diagnostics to assess whether treatment benefit and compliance are differently aligned across groups, Step 2 evaluates the conditional means of these same quantities within the sets actually selected by naive ITT threshold rules, which are exactly (proportional to) the objects governing the groupwise improvement decomposition. \Cref{fig:sfhsa-cdfplot-race} connects the previous mechanism-based diagnostic to investigate \textit{who is selected} by naive ITT targeting under different policies. 

For each threshold, we look within each group at those who would be selected and plot the average encouragement score, treatment benefit, and compliance (y-axis, respectively). We sweep over score thresholds on the x-axis to sweep over all possible budgets.  

The left panel shows the average encouragement score among selected individuals by group. For large ranges of moderate budgets, the selected white applicants have larger encouragement score than non-white applicants. The middle plot illustrates the treatment benefit $\tau$, which is broadly similar within both group's selected sets, so disparities in average outcome are not driven by differences in how much individuals stand to gain from interview completion. By contrast, the right panel shows a clearer gap, persisting across budgets, in the average compliance effect $\E[ (p_{1\mid 1}(X,A) - p_{1\mid 0}(X,A))\mid \pi^*=1,A]$.


Comparing these panels, we see that disparities in improvement naive ITT-based allocation are in fact driven almost entirely by disparities in access, which would persist under a large range of budgets. 


\textit{Step 3: Interpretation of potential disparities, if present.}

A central concern is that worse-off groups may benefit more from treatment yet be less responsive to low-touch encouragement because of administrative and resource barriers. In SNAP, the former is plausible given diminishing marginal utility, while the latter is consistent with the burdens of recertification, scheduling, childcare, and transportation. This distinction matters because responsiveness is potentially actionable: organizations can redesign outreach intensity, navigation support, or interview logistics. \citet{koenecke2023popular} likewise find public support for more equitable SNAP outreach, including greater spending to reach harder-to-reach groups.

By separating treatment efficacy from responsiveness, our audit identifies access gaps that would persist even under outcome-efficient allocations. 
next step. 
Knowing whether limits to improvements from personalized targeting stem from disparities in compliance or treatment efficacy clarifies relevant next steps. 
In SNAP, that next step lies in outreach and administrative redesign rather than the underlying benefit itself. However local administrators and nonprofits have far more room to design supplemental outreach processes affecting compliance, than they do to modify the centralized policy eligibility rules that underlie treatment efficacy.

\paragraph{Step 4: 
From audit to redesign: quantifying the gains and limits of personalized targeting in SNAP.}

We use 
off-policy evaluation and optimization to quantify the improvements and limits of personalized targeting in SNAP, illustrating our how prior diagnosed disparities 
result in potential equity concerns in group-differential limits to improvement from targeting, due to persistent access barriers.

\begin{figure}
    \centering
    \includegraphics[width=0.8\linewidth]{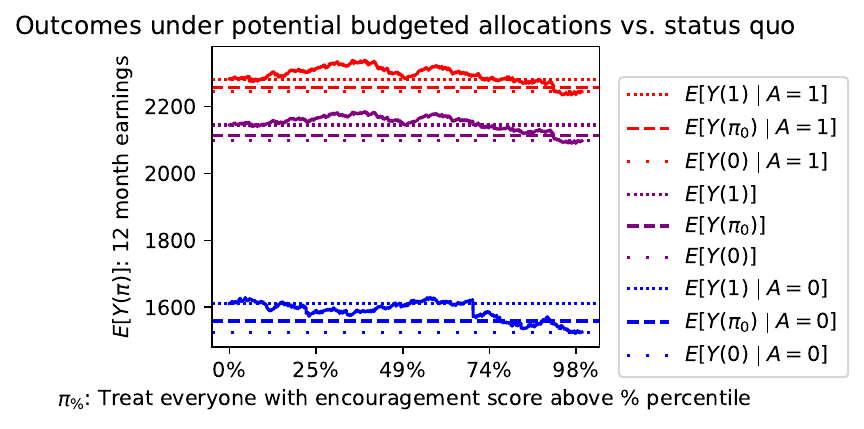}
    \caption{Comparison of average group outcomes under budgeted allocations (solid lines) for targeted treatment of different beneficiary shares ($E[Y(\pi_{\%})\mid A=a]$) with self-selection ($E[Y(\pi_0)\mid A=a]$) or no-reminder ($E[Y(0)\mid A=a]$) status quo.}
    
    \label{fig:sfhsa-budgeted-allocation-comparison}
\end{figure}

We quantify the extent of potential improvements from automatic enrollment in reminders (recommend to all) or, under potential budget constraints, budget-optimal targeted allocations. 
In \Cref{fig:sfhsa-budgeted-allocation-comparison}, the y-axis measures different groupwise or average downstream values: 
red for the group $A=1$ (nonwhite), $\E[Y(\pi)\mid A=1]$; purple for the entire population, $\E[Y(\pi)]$; and blue for $A=0$ (white), $\E[Y(\pi)\mid A=0]$. Next, within each set of colored lines, we compare the following: the \textit{solid} line for budget-optimal policies $\E[Y(\pi_{\%})]$, the \textit{densely dotted} line treating all ($E[Y(1)\mid A=a]$), the \textit{dashed} line indicating the observed self-selection, and the \textit{sparsely dotted} line treating no one ($E[Y(0)\mid A=a]$). The x-axis ranges from treating everyone to no-one 
by ranging over encouragement score percentile thresholds, hence sweeping over all possible budget restrictions. 
This figure also connects back to the earlier fairness decomposition. For each group \(a\), the groupwise improvement
$\E[Y(\pi_b)\mid A=a]-\E[Y(0)\mid A=a]$ can be read off the plot as the difference between solid and sparsely-dotted lines. Observe that the magnitude of these improvements differs across groups, exactly due to our previously observed disparities in compliance and alignment between treatment take-up and efficacy. 

Overall, additional reminders weakly improve outcomes relative to no reminders, and relative to currently observed self-selection. The main exception is the white subgroup, for whom self-selection modestly improves outcomes. By contrast, for budgets covering at least roughly $25\%$ of the population, targeted allocation substantially improves group-average outcomes relative to both current self-selection and the no-reminder baseline. Thus, in terms of group-level outcomes, targeted reminder allocation improves upon the status quo. Nonetheless, limits persist as to the extent of gains from low-touch outreach alone, suggesting that additional investment in outreach resources beyond text reminders to ensure that those who may benefit the most from program enrollment can overcome access barriers or burdens to do so. Additional research would be needed to ascertain root causes and whether the observed disparities in compliance reflect behavioral inattention, trust in government, or structural barriers to compliance with administrative burdens matters for policy implications.

\section{Policy Optimization under Access Constraints}\label{sec-methodd}
Here we move from diagnosis to decision rules. We first characterize the structure of optimal fairness-aware encouragement policies in the population—showing they remain threshold rules on appropriate scores or Lagrangian scores—and then use those moment representations to motivate the estimation and constrained policy-learning procedures that follow.


Although our motivating decision problem is multi-objective, the formal statistical analysis below focuses on the primary outcome criterion $V(\pi)=\E[Y(\pi)]$. Fairness, access, and budget considerations enter through separately estimated constraint moments, and application-specific scalarizations can be formed from those same moments when desired.

\begin{proposition}[Group-specific thresholds under demographic parity. ]\label{prop-threshold-policy-groupfair}
Suppose the distribution of $s(x,a)$ is continuous (no atoms). Consider optimizing encouragement effects under a $\beta$-fraction budget and groupwise budget parity: 
\begin{equation}
\pi_{DP}^*(x,a) 
\in \max_{\pi(x,a)} \{\E[Y(\pi)] \colon \E[\pi(X,A)] \leq \beta, \abs{\E[\pi(X,A)\mid A=a]-\E[\pi(X,A)\mid A=b]} = 0 \}\label{eqn-demographic-parity-budget} 
\end{equation}
    Define the $\beta$-quantile for group $a$'s score: $t_{\beta}^a \coloneqq \max\!\left\{\inf\{z : P(s(X,a)\le z)\ge 1-\beta\},\,0\right\}$.
    Then the optimal policy comprises group-dependent thresholds at the top-$\beta$ quantile of the encouragement score within each group (so long as it is positive):
$$\pi_{DP}^*(x,a) \;=\; \mathbb{I}[ s(x,a) \ge t_{\beta}^a].$$
\end{proposition}

That is, ensuring parity in encouragement budget by group is achieved by ranking within the score distributions in each group. 

\paragraph{Treatment parity--constrained optimal decision rules} We consider an access/resource/burden parity fairness constraint: 
\begin{align}%
\max_{\pi} \; 
\{
\E[Y(\pi)] \colon \E[T(\pi)\mid A=a]- \E[T(\pi)\mid A=b]  \leq \epsilon \}
\label{eq-opt-treatment-parity-constraint}
\end{align}
We discuss the one-sided constraint in the analysis for simpler discussion; enforcing absolute values, etc., follows in the standard way. Not all values of $\epsilon$ may be feasible, and the amount of room for reducing disparities in treatment depends on the heterogeneity in treatment response. The feasible range of treatment-parity disparities is an interval
\([\inf_\pi \abs{\Delta_T},\sup_\pi \abs{\Delta_T}]\) that can be pre-calculated, so that a given treatment disparity constraint \(\Delta_T\le \epsilon\) is feasible iff $\epsilon$ lies within the interval.

We first characterize a threshold solution when the policy class is unconstrained. 
\begin{proposition}[Threshold solutions under treatment-parity constraints]\label{prop-threshold-treatment-parity}
Let 
$\kappa(X,A):=p_{1\mid 1}(X,A)-p_{1\mid 0}(X,A)$, $c_0:=\E[p_{1\mid 0}(X,A)(\tfrac{\indic{A=a}-\indic{A=b}}{p(A)})]$, and
$$\ell(\lambda,X,A):=\kappa(X,A)\bigl\{\tau(X,A)-\lambda \tfrac{\indic{A=a}-\indic{A=b}}{p(A)}\bigr\}.$$
Suppose \eqref{eq-opt-treatment-parity-constraint} is strictly feasible. Then there exists $\lambda^*\in\arg\min_{\lambda\ge 0}\bigl\{\E[\ell(\lambda,X,A)_+]+\lambda(\epsilon-c_0)\bigr\}$ such that $\pi^*(x,u)=\mathbb{I}\{\ell(\lambda^*,x,u)>0\}$.
If instead $\pi$ is restricted to covariates only, $\lambda^*\in\arg\min_{\lambda\ge 0}\bigl\{\E[\E[\ell(\lambda,X,A)\mid X]_+]+\lambda(\epsilon-c_0)\bigr\}$ and $\pi^*(x)=\mathbb{I}\{\E[\ell(\lambda^*,X,A)\mid X=x]>0\}$.
\end{proposition}
Establishing this threshold structure (which follows by duality of infinite-dimensional linear programming) allows us to provide a generalization bound argument. We make a standard assumption that the functions we use for estimation are learnable, e.g. have finite VC dimension.

\begin{assumption}[Learnable nuisance functions]\label{asn-learnability}
    Assume that the nuisance models $\eta = [p_{1\mid 0},p_{1\mid 1}, \mu_{1},\mu_0]^\top, \eta \in \mathcal{F}_\eta$ are consistent and well specified with finite Vapnik--Chervonenkis (VC) dimension $v_\eta$ over the product function class $\mathcal{F}_\eta$.
\end{assumption}

\begin{proposition}[Policy value generalization]\label{prop-polgen-unconstr}

Let $\Pi = \{ \mathbb{I}\{\E[\ell(\lambda,X,A; \eta)\mid X]>0 \colon \lambda \in \mathbb{R}; \eta \in \mathcal{F}_\eta \}.$ 
    $$\textstyle \sup_{\pi \in \Pi, \lambda\in\mathbb{R}}\left|
(\E_n[ \pi \ell(\lambda,X,A) ] - \E[ \pi \ell(\lambda,X,A) ])  
    \right| = O_p(n^{-\frac 12} ) $$
\end{proposition}
This bound is stated for known nuisance functions; verifying stability under estimated nuisance functions further requires rate conditions.
\paragraph{Doubly robust estimation}
We may improve the statistical properties of the estimation by developing \textit{doubly robust} estimators, which can achieve faster statistical convergence when both the probability of recommendation assignment (when it is random) and the probability of outcome are consistently estimated or can otherwise protect against misspecification of either model. We first consider the ideal setting when algorithmic recommendations are randomized so that $e_r(X) = P(R=r\mid X)$. 
\begin{proposition}[Variance-reduced estimation]\label{prop-doublerobustness}
\begin{align*}  \textstyle V(\pi)&=
\sum_{t\in \tspace, r\in\{0,1\}}\E\left[ 
\pi_r(X)
\left\{ \frac{\indic{R=r}}{e_r(X)}
(\mathbb{I}[T=t]  Y -\mu_{t}(X)\ptmidrx ) +\mu_{t}(X)\ptmidrx
\right\} 
\right] \\
\E[T(\pi)] &= \textstyle \sum_{r\in\{0,1\}} \E\left[ 
\pi_r(X)
\left\{ \frac{\indic{R=r}}{e_r(X)}
(T  -p_{1\mid r}(x) ) +p_{1\mid r}(x) 
\right\} 
\right] 
\end{align*} 
\end{proposition}
We retain the full expression rather than simplifying \citep[as appears in][]{qiu2021optimal} since the doubly robust estimation of constraints changes the Lagrangian.
For example, for regression adjustment, it is clearer in \Cref{prop-threshold-treatment-parity} how constraints affect the optimal decision rule. %

\subsection{Extension: Algorithmic advice with treatment, but not recommendation overlap}\label{subsec-overlap-robust}

When the recommendations are, e.g., the \textit{high risk}/\textit{low risk} labels from binary classifiers, the overlap assumption may not be satisfied since the algorithmic recommendations are deterministic functions of covariates. However, note that identification in \Cref{prop-identification} requires only SUTVA, consistency, and the exclusion restriction.

A naive approach based on parametric extrapolation is to estimate $p_{1\mid 1}(X)$, treatment responsivity, on the observed data and simply use the parametric form to extrapolate to the full dataset. Though this is generally unsatisfactory by itself, it is the starting point for our robust extrapolation methods.
\paragraph{Robust extrapolation under violations of overlap}
We next describe methods for robust extrapolation under structural assumptions about the smoothness of the outcome models. Under violations of overlap, the only unknown quantity is $p_{t\mid r}( X)$ in regions of no overlap in recommendation; however, a plausible assumption is that the underlying function is smooth in covariates. A robust approach obtains worst-case bounds on policy value under all functions compatible with a particular smoothness assumption. On the other hand, we assume that overlap holds with respect to $T$ given covariates $X$, so our finer-grained approach via \Cref{asn-exclusionrestriction} yields milder penalties due to robustness since we need robustly extrapolate only the treatment response to recommendations, $p_{t\mid r}( X),$ rather than the outcome model's $\mu_t(X).$
Define the regions of no overlap as the following: 
 Let $\mathcal{X}^{\text{nov}}_r = \{ x: P(R=r\mid x) = 0\}$; in this region, we do not jointly observe all potential values of $(t,r,x)$. In addition, let $\mathcal{X}^{\text{nov}} = \bigcup_{r} \mathcal{X}^{\text{nov}}_r$. Correspondingly, define the overlap region as $\mathcal{X}^{\text{ov}} = (\mathcal{X}^{\text{nov}})^c.$ We consider uncertainty sets for ambiguous treatment recommendation probabilities. 
For example, one plausible structural assumption is \textit{monotonicity} of treatment in recommendation, or other global shape restrictions.
We could assume uniform bounds on unknown probabilities; more refined bounds, such as Lipschitz smoothness with respect to some distance metric $d$; or boundedness. 
\begin{align*}
\textstyle 
\mathcal{U}_{\text{lip}} &\coloneqq  \left\{ q_{1\mid r}(x') \colon 
d(q_{1\mid r}(x'), p_{1\mid r}(x))\leq L d(x',x),\;(x',x)\in(\Xnov\times \Xov)
\right\}\\
\mathcal{U}_{\text{bnd}} &\coloneqq \left\{ q_{1\mid r}(x') \colon 
\underline{b}(x)\leq q_{1\mid r}(x')  \leq \overline{b}(x)
\right\} 
\end{align*}

We write $\indov{X}:=\mathbb{I}\{X\in\Xov\}$ and $\indnov{X}:=\mathbb{I}\{X\in\Xnov\}$ for the overlap and no-overlap indicators.
Define $V_{ov}(\pi)\coloneqq \sumrt 
 \E[ \pi_r(X)\ptmidrx \mu_t(X) 
 \indov{X}
 ]
$. Let $\mathcal{U}$ denote the uncertainty set including any custom constraints, e.g., $\mathcal{U} = \mathcal{U}_{bnd} \cap\mathcal{U}_{\text{lip}} $. %
We define the pessimistic robust lower bound $\underline{V}_{\mathrm{rob}}(\pi)\coloneqq V_{\mathrm{ov}}(\pi)+\underline{V}_{\mathrm{nov}}(\pi)$, where
\begin{align*}\textstyle
\underline{V}_{\mathrm{nov}}(\pi)\coloneqq \min_{q_{1\mid r}(X,A)\in\mathcal{U}}\sum_{r\in\{0,1\}}\E\!\left[\pi_r(X,A)\bigl\{\mu_0(X,A)+q_{1\mid r}(X,A)\tau(X,A)\bigr\}\indnov{X}\right].
\end{align*}
Here $\underline{V}_{\mathrm{rob}}(\pi)$ is the pessimistic robust lower bound; the corresponding optimistic bound replaces the minimum by a maximum.

In the binary-outcome, constant-bound case, the worst-case no-overlap term admits a simple closed form based on the sign of the relevant conditional mean; see \Cref{apx-opt-overlap}. We therefore state only the more general interval-bounded robust LP here.
We consider the case of continuous-valued outcomes in the example setting of the simple treatment parity--constrained program of \Cref{eq-opt-treatment-parity-constraint}. We first study simple uncertainty sets, such as intervals, to deduce insights about the robust policy, with a more general reformulation in the appendix. 
For interval bounds, it is convenient to summarize each interval by its midpoint $m_r$ and radius $\rho_r$.
We now characterize the policy that maximizes the pessimistic robust lower bound $\underline{V}_{\mathrm{rob}}(\pi)$.
\begin{proposition}[Overlap-robust linear program]\label{prop-robustlp}
Suppose that $r,t\in\{0,1\}$ and $q_{1\mid r}(\cdot,u)\in\mathcal U_{\mathrm{bnd}}$ for all $r\in\{0,1\}$, $u\in\{a,b\}$, so that $\underline B_r(x,u)\le q_{1\mid r}(x,u)\le \overline B_r(x,u)$. Define midpoint and radius $m_r(x,u):= \frac 12 ({\underline B_r(x,u)+\overline B_r(x,u)}),\;\rho_r(x,u):=
\frac 12 
({\overline B_r(x,u)-\underline B_r(x,u)})$, objective-at-midpoint $c_1(\pi):=\sum_{r\in\{0,1\}}\E\!\left[\tau(X,A)\pi_r(X,A)m_r(X,A)\indnov{X}\right]$, and disparity on the overlap region:
\[
\Delta_{\mathrm{ov}}^T(\pi):=\E[T(\pi)\indov{X}\mid A=a]-\E[T(\pi)\indov{X}\mid A=b].
\]
Then the overlap-robust policy optimization problem is:
\begin{align*}
\max_{\pi}\; &
V_{\mathrm{ov}}(\pi)
+\E[\mu_0(X,A)\indnov{X}]
+c_1(\pi)
-\sum_{r\in\{0,1\}}
\E\!\left[
|\tau(X,A)|\,\pi_r(X,A)\,\rho_r(X,A)\,\indnov{X}
\right]
\\
\text{s.t.}\; &
\sum_{r\in\{0,1\}}
\Big\{
\E[\pi_r(X,A)\overline B_r(X,A)\indnov{X}\mid A=a]
-\E[\pi_r(X,A)\underline B_r(X,A)\indnov{X}\mid A=b]
\Big\}
+\Delta_{\mathrm{ov}}^T(\pi)\le \epsilon.
\end{align*}
\end{proposition}

A key payoff of our modeling perspective is that the take-up / treatment-efficacy decomposition makes the overlap-robust problem less conservative: when recommendation overlap fails, we only need to robustly bound extrapolated treatment response to recommendation, rather than the downstream outcome model itself. Later in \Cref{sec-experiments} we illustrate how this allows us to robustly improve current recommendations. 

\subsection{Extension: variance-sensitive constrained policy learning}\label{sec-general-method}
We previously discussed policy optimization over unrestricted decision rules given estimates. 
We now turn to implementable policy learning over a restricted policy class under more general fairness constraints. Let $\Delta(\pi)\in\mathbb{R}^K$ denote a vector of identified disparity moments, and consider
\[
\max_{\pi\in\Pi}\{E[Y(\pi)] : E[\Delta(\pi)] \le d\}.
\]
This formulation also enables local Pareto improvement around a reference policy $\pi_0$, by treating allowable value loss relative to $\pi_0$ as an additional constraint. To obtain tighter out-of-sample feasibility, we solve these constrained problems with a two-stage procedure that localizes around a first-stage solution and calibrates variance-based slacks. We discuss our two-stage procedure building on a constrained optimization oracle, which we instantiate with \citep{agarwal2018reductions}, but other methods with statistical guarantees on the Lagrangian regret also suffice.

\paragraph{A constrained policy optimization oracle. }
We describe a specific procedure for constrained policy optimization that we leverage to concretize our results.

\citet{agarwal2018reductions} develops a general saddle-point procedure that optimizes such fairness-constrained problems over a hypothesis class with two key algorithmic tools: convexifying the space of policies via randomization, and saddle-point optimization over the convexified constrained problem. Such hypothesis classes include probabilistic classifiers that are good parametrizations of our causal encouragement policy.

They convexify $\Pi$ via randomized policies $Q \in \mathcal{P}(\Pi)$, where $\mathcal{P}(\Pi)$ is the set of distributions over $\Pi,$ i.e., a randomized classifier that samples a policy $\pi\sim Q$. Hence, they solve:
\begin{align*}
\max_{Q \in \mathcal{P}(\Pi)} \min_{\lambda} \mathcal{L}(Q,\lambda), \text{ where } \mathcal{L}(Q,\lambda) \coloneqq \mathbb{E}[V(Q)] - \lambda^\top(\E[ {\Delta}(Q) ]-{d})
\end{align*}
Of course, the stochastic objective and constraints need to be estimated from data. We let $\hat V, \hat\delta$ denote such estimates of the policy value and constraint scores.

We also require a feasibility margin $\epsilon_k$ that depends on concentration of the estimated constraints, so the sampled constraint vector is $\hat d_k = d_k - \epsilon_k,$ for all $k$, ensuring that the constraint holds at nominal level $d_k$ out of sample.

We introduce some additional notation. Let \(O=(X,A,R,T,Y)\sim P\) denote a fresh observation
For a deterministic policy \(\pi\), let $v_{DR}(\pi) := v_{DR}(O;\pi,\eta_0)$ denote the random doubly robust value contribution, and let $\Delta(\pi) := \Delta(O;\pi)\in\mathbb R^K$
denote the random vector of identified constraint contributions. Let $\delta(\pi) := \delta(O;\pi,\eta)$ denote the corresponding observed-data estimating score. We overload notation and discuss randomized policies via a distribution over policies $Q$, e.g. $v_{DR}(Q) := \E_{\pi\sim Q}\!\left[v_{DR}(\pi)\mid O\right],$ where \(\E_{\pi\sim Q}[\cdot\mid O]\) denotes expectation over the policy randomization only, and use analogous variants. We seek an approximate saddle point so that the constrained solution is equivalent to the Lagrangian,
$$
\widehat{\mathcal{L}}(Q, {\lambda})=
\E_n[v_{DR} (O;\eta, Q) ]  -{\lambda}^{\top}(
\E_n[\delta (O; \eta, Q) ]
-{\hat{d}}
).
$$
We simultaneously solve for an approximate saddle point over the $B$-bounded domain of $\lambda$:
\begin{align*}
&
\textstyle
\underset{Q \in \mathcal{P}(\Pi)}{\max}
\;\;
\underset{{{\lambda} \in \mathbb{R}_{+}^{K},\|{\lambda}\|_1 \leq B}}{\min}
\widehat{\mathcal{L}}(Q, {\lambda}), \qquad
\underset{{\lambda} \in \mathbb{R}_{+}^{K},\|{\lambda}\|_1 \leq B}{\min} \;\; \underset{Q \in \mathcal{P}(\Pi)}{\max} \; \widehat{\mathcal{L}}(Q, {\lambda})
\end{align*}

If convexifying via randomization is not permitted, we discuss a simple generalization in the appendix that leverages well-known convex surrogate losses for optimizing optimal treatment regimes \citep{zhao2012estimating}.
We defer specific details to \Cref{apx-pol-opt}, including how to encode fairness constraints in standard convexified LP form.\footnote{The centered policy regret can be reparametrized via a policy parameter $\beta$ as: $  J(\beta) = J(\op{sgn}(f_\beta(\cdot))) = \E[ \op{sgn}(f_\beta(X)) \left\{ \psi \right\}].$ We can apply the standard reduction to cost-sensitive classification since $\psi_i \op{sgn}(f_\beta(X_i)) = \abs{\psi_i} (1-2 \indic{\op{sgn}(f_\beta(X_i)) \neq \op{sgn}(\psi_i)}$). Then, we can use surrogate losses for the zero-one loss. Although many functional forms for $\ell(\cdot)$ are Fisher consistent, one such choice of $\ell$ is the logistic (cross-entropy) loss $\E[\abs{\psi} \ell(f_\beta(X), \op{sgn}(\psi))], \qquad  l(g,s) = 2 \log(1+\exp(g)) - (s+1)$.}

We introduce the notation $\operatorname{REDFAIR}(\mathcal D,v_{DR},\delta,\hat d;\hat\eta)$
to refer to an invocation of the algorithm of \citep{agarwal2018reductions}.
\citet[Theorem 3]{agarwal2018reductions} gives generalization guarantees on the policy value and constraint violation achieved by the approximate saddle point output by the algorithm. The analysis is generic under rate assumptions on uniform convergence of policy and constraint values.

\begin{assumption}[Uniform concentration of value and constraint estimators]
\label{asn-polopt-rate}
There exist a scalar sequence \(\varepsilon_{V,n}\) and a vector sequence
\(\varepsilon_{\Delta,n}\in \mathbb{R}_+^K\) such that, with probability at least
\(1-\zeta\), and, coordinatewise for each \(k\in[K]\),
\[
\sup_{Q\in \mathcal P(\Pi)}
\left\{
\mathbb{E}[V(Q)]-\mathbb{E}_n[v_{DR}(Q)]
\right\}
\le \varepsilon_{V,n}(\zeta),
\;\;
\sup_{Q\in \mathcal P(\Pi)}
\left\{
\mathbb{E}[\Delta_k(Q)]-\mathbb{E}_n[\delta_k(Q)]
\right\}
\le \bigl(\varepsilon_{\Delta,n}(\zeta)\bigr)_k.
\]
\end{assumption}
Under \Cref{asn-learnability} and orthogonal nuisance estimation,
\(\epsilon_n(\zeta)\) follows from standard empirical-process concentration. 
The optimization approach of \citet{agarwal2018reductions} follows the no-regret online learning paradigm of \citet{freund1997decision}, 
an online saddle point optimization alternating optimization over 
$\lambda$ and best-response updates on $Q$ (hence $\pi$). Further details are in \Cref{sec-saddle}.  We only make a small modification to use a refined no-regret \citep[second-order multiplicative weights;][]{cesa2007improved,steinhardt2014adaptivity} algorithm instead of the original hedge/exponentiated gradient algorithm \citep{freund1997decision} 
 for the optimization over $\lambda$. The second-order regret bounds surface dependence on the estimation variance and therefore the two-stage algorithm's variance reduction. 


\textbf{Two-stage variance-constrained algorithm.}

Our two-stage procedure, described formally in \Cref{alg-metaalg2}, improves upon generic constrained optimization by introducing implicit variance regularization, so that regret bounds depend more favorably on the maximal variance over small-variance \textit{slices} near the optimal policy, rather than worst-case constants over all policies \citep{chernozhukov2019semi,athey2021policy}. As a result, such policies can achieve tighter fairness constraint control. 


\begin{algorithm}[t!]
\caption{Two-stage localized constrained policy optimization}
\label{alg-metaalg2}
\begin{algorithmic}[1]
\Statex \textbf{Input:} policy class \(\Pi\), value score \(v_{DR}\), vector constraint score \(\delta\), target vector \(d\), localization radius \(\epsilon_n\)

\State Randomly split the sample into \(\mathcal D_1\) and \(\mathcal D_2\)

\State Fit nuisance estimators \(\hat\eta_1\) on \(\mathcal D_1\)

\State Compute the first-stage solution $\hat Q_1 \gets \operatorname{REDFAIR}(\mathcal D_1,v_{DR},\delta,d;\hat\eta_1)$

\State Define the nearly binding set $\hat{\mathcal I}_1
\gets
\left\{
k\in[K]:
d_k-\mathbb E_{n_1}[\delta_k(\hat Q_1)] \le \epsilon_n
\right\}$
\State Compute confidence slacks. 

For each constraint $\hat\sigma_k^2 \gets \widehat{\mathrm{Var}}_{\mathcal D_1}\!\bigl(\delta_k(\hat Q_1)\bigr),$ and set $\hat d_k \gets d_k - z_\alpha \hat\sigma_k\, n_1^{-\alpha}$
\State Define the augmented moment and threshold vectors \(\tilde\delta, \tilde d\) by
\[
\tilde\delta(Q)
\leftarrow
\left(
\delta(Q)^\top,\,
\bigl(\delta_{\hat I_1}(Q)-\delta_{\hat I_1}(\hat Q_1)\bigr)^\top,\,
v_{\mathrm{DR}}(\hat Q_1)-v_{\mathrm{DR}}(Q)
\right)^\top,
\qquad
\tilde d
\leftarrow
\left(
\hat d^\top,\,
\epsilon_n{\bf 1}_{|\hat I_1|}^\top,\,
\epsilon_n
\right)^\top.
\]


\State Fit nuisance estimators \(\hat\eta_2\) on \(\mathcal D_2\). 

On \(\mathcal D_2\), compute the second-stage solution ${\hat Q_2 \gets \operatorname{REDFAIR}(\mathcal D_2,v_{DR},\tilde\delta,\tilde d;\hat\eta_2)}$
\State \textbf{return} \(\hat Q_2\)
\end{algorithmic}
\end{algorithm}

We adapt an out-of-sample regularization scheme developed in \citet{chernozhukov2019semi}, which recovers variance-sensitive regret bounds. To summarize the procedure at a high level: we split the data into two subsets, $\mathcal{D}_1$ and $\mathcal{D}_2$. On the first data split, we estimate the optimal policy distribution $\hat Q_1$, identify the first-stage binding constraints via the index set $\hat I_1$, and estimate the objective and constraint variances at $\hat Q_1$. 
Next, we \textit{augment} the constraint matrix to optimize over policy distributions $Q$ in the second stage that are localized within $\epsilon_n$-sized neighborhoods for policy value and constraint moment values, relative to $\hat Q_1$. For example, the new constraint vector slack $\hat{d}_k\leftarrow d_k - z_\alpha \hat{\sigma}_k n_1^{-\alpha}$ is updated via the binding constraint variance $\widehat{\operatorname{Var}}_{\mathcal{D}_1}(\delta_k(\hat{Q}_1)), i \in \hat I_1$, where $z_\alpha$ is the $\alpha$-level z-score to ensure constraints hold with $\alpha$-probability out-of-sample. 
\begin{align*}
    \max_{Q \in \mathcal{P}(\Pi)}&\mathbb{E}_{n_2}[v_{DR}(Q)]
\\
\text{ s.t. }& 
\mathbb{E}_{n_2}[\delta(Q)] \le \hat d,\ 
\mathbb{E}_{n_2}\!\left[
\delta_{\hat{\mathcal I}_1}(Q)-\delta_{\hat{\mathcal I}_1}(\hat Q_1)
\right]
\le
\epsilon_{n_2} \mathbf{1}_{|\hat{\mathcal I}_1|},\ 
\mathbb{E}_{n_2}\!\left[
v_{DR}(\hat Q_1)-v_{DR}(Q)
\right]
\le \epsilon_{n_2}
\end{align*}



Next, we provide a generalization bound on the out-of-sample performance of the policy returned by the two-stage procedure. Importantly, because of our two-stage procedure, the regret of the policy depends on the worst-case variance of near-optimal policies (rather than all policies). Define the function classes
  $ \mathcal{F}_{\Pi}=\{ v_{DR}(\cdot, \pi; \eta)\colon \pi \in \Pi, \eta \in \mathcal{F}_\eta\}$,  $\mathcal{F}_{j}=\{ g_j(\cdot, \pi; \eta)\colon \pi \in \Pi, \eta \in \mathcal{F}_\eta\}$
  and the empirical entropy integral $\kappa(r, \mathcal{F})=
   {\inf _{\alpha \geq 0}\{4 \alpha+10 \int_\alpha^r \sqrt{\frac{\mathcal{H}_2(\epsilon, \mathcal{F}, n)}{n}} {d} \epsilon\}}$, where $H_2(\epsilon, \mathcal{F}, n)$ is the $L_2$ empirical entropy, i.e., log of the $\norm{\cdot}_2$ $\epsilon$-covering number.
We make a mild assumption of a learnable function class (bounded entropy integral) \citep{van1996weak}, which is satisfied by many standard function classes such as linear models, polynomials, kernel regression, and neural networks \citep{wainwright2019high}. 
\begin{assumption}
 The function classes $\mathcal{F}_{\Pi}, \{\mathcal{F}_j\}_{j\in \mathcal{J}}$ satisfy that, for any constant $r, \kappa(r, \mathcal{F}) \rightarrow 0$ as $n \rightarrow \infty$. The function classes $\{\mathcal{F}_j\}_{j\in \mathcal{J}}$ comprise $L_j$-Lipschitz contractions of $\pi.$ 
\end{assumption}
We assume that we are using doubly robust/orthogonalized estimation as in \cref{prop-doublerobustness} and, hence, state our results depending on the estimation error of nuisance vector $\eta$. The next theorem summarizes the out-of-sample performance of the two-stage algorithm of \Cref{alg-metaalg2}, $\hat Q_2$.
\begin{theorem}[Variance-based oracle policy regret]\label{thm-varregret}
Let $Q^\star
\in
\arg\max_{Q\in\mathcal P(\Pi)}
\{V(Q):\Delta(Q)\le d\}$.
Suppose that, with probability at least \(1-\delta/2\) over the nuisance function estimation sample, $\max_{\ell\in[L]}
\E\!\left[(\hat\eta_\ell-\eta_{0,\ell})^2\right]
\le
\chi^2_{n,\delta}$ (we have $\chi^2_{n,\delta}$ as a bound on nuisance MSE).
Denote $v^0_{\mathrm{DR}}(O;Q):=v_{\mathrm{DR}}(O;Q,\eta_0),
\qquad
\delta^0_k(O;Q):=\delta_k(O;Q,\eta_0)$,
and define the worst-case variance radius
$r
:=
\sup_{Q\in\mathcal P(\Pi)}
\sqrt{
\E\![
(v^0_{\mathrm{DR}}(O;Q))^2
]
}.$ For the second-stage sample size \(n_2:=|D_2|\), let $\epsilon_{n_2}
=
\Theta\!(
\kappa(r,\mathcal F_\Pi)
+
r\sqrt{\frac{\log(1/\delta)}{n_2}}
)$. Let \(\hat I_1\) denote the nearly binding set defined in Algorithm~\ref{alg-metaalg2}. Define the localized near-optimal slice
\[
\mathcal Q^\star(\epsilon;\hat I_1)
:=
\left\{
Q\in\mathcal P(\Pi):
V(Q^\star)-V(Q)\le \epsilon,\quad
\Delta_k(Q)\le d_k+\epsilon,\ \forall k\in[K],\quad
\Delta_k(Q)-\Delta_k(Q^\star)\le \epsilon,\ \forall k\in\hat I_1
\right\}.
\]
Let $\tilde\epsilon_{n_2}
=
O(\epsilon_{n_2}+\chi^2_{n,\delta})$.
Define the value-variance radius $\bar\sigma^2_{D_2}$ and the constraint-variance: $$\bar\sigma^2_{D_2}
:=
\sup_{Q,Q'\in\mathcal Q^\star(\tilde\epsilon_{n_2};\hat I_1)}
\op{Var}\!(
v^0_{\mathrm{DR}}(O;Q)-v^0_{\mathrm{DR}}(O;Q')
), \;\;
\bar\sigma^2_{k,D_2}
:=
\sup_{Q,Q'\in\mathcal Q^\star(\tilde\epsilon_{n_2};\hat I_1)}
\op{Var}\!(
\delta^0_k(O;Q)-\delta^0_k(O;Q')
)$$
For each \(k\), write
$
\mathcal F^\Delta_k
:=
\left\{
\delta_k(\cdot;\pi,\eta):
\pi\in\Pi,\ \eta\in\mathcal F_\eta
\right\}.
$
Then the policy distribution \(\hat Q_2\) returned by Algorithm~\ref{alg-metaalg2} satisfies,
with probability at least \(1-\delta\),
\[
V(Q^\star)-V(\hat Q_2)
=
O\!\left(
\kappa(\bar\sigma_{D_2},\operatorname{conv}(\mathcal F_\Pi))
+
\bar\sigma_{D_2}n_2^{-1/2}\sqrt{\log(3/\delta)}
+
\chi^2_{n,\delta}
\right),
\]
and, for each \(k\in[K]\),
\[
\bigl(\Delta_k(\hat Q_2)-d_k\bigr)_+
=
O\!\left(
\kappa(\bar\sigma_{k,D_2},\operatorname{conv}(\mathcal F^\Delta_k))
+
\bar\sigma_{k,D_2}n_2^{-1/2}\sqrt{\log(3/\delta)}
+
\chi^2_{n,\delta}
\right).
\]
\end{theorem} 
The specific benefits of the two-stage approach are that 1) the constants are improved from being absolute, structure-agnostic bounds to depending on the variance of low-regret policies, which also reflects the improved variance from the use of doubly robust estimation as in \cref{prop-doublerobustness}, and 2) it allows less conservative satisfaction of the fairness constraint out of sample.

\section{Case Studies}\label{sec-experiments}

\subsection{Why and where the two-stage constrained optimization helps}
In \Cref{sec-general-method}, we developed a general two-stage procedure for constrained formulations or Pareto improvement upon prior policies. We highlight the benefits of our two-stage approach in a simulated example, where it operates as effective variance regularization and improves out-of-sample fairness constraint feasibility and Lagrangian regret, especially in finite samples. 

\begin{figure}
    \centering
\includegraphics[width=\linewidth]{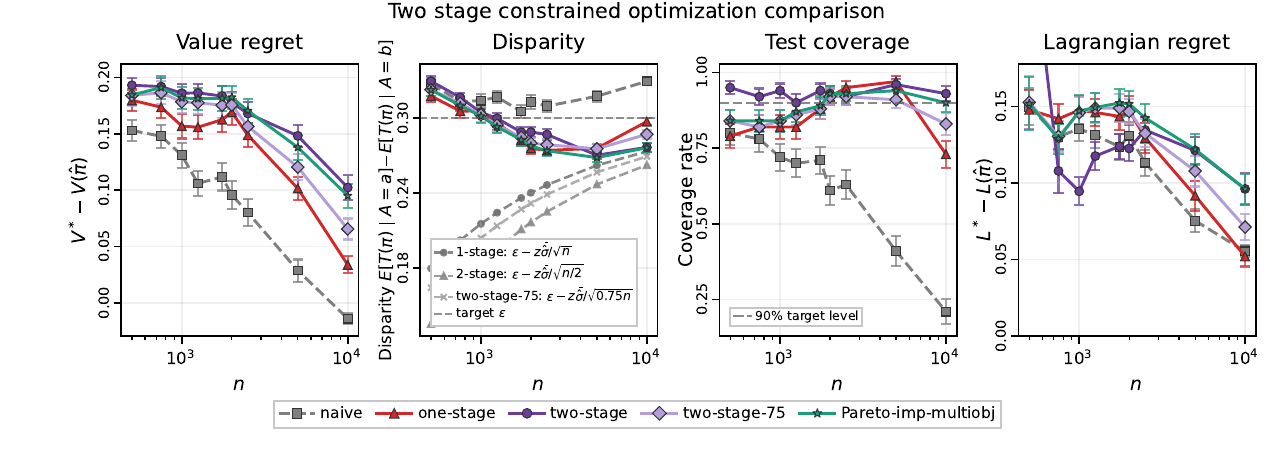}
\caption{Comparison of the two-stage method of Section~6.2 to ablations, as a function of sample size $n$. In the disparity panel, the horizontal line is the nominal fairness target $\epsilon=0.3$, while the dashed curves labeled \texttt{1-stage}, \texttt{2-stage}, and \texttt{two-stage-75} denote the corresponding slack-adjusted effective constraint levels: for the one-stage baseline, $\epsilon_{\mathrm{eff}}^{(1)}=\bigl(\epsilon-z_{0.9}\sigma_{\max}/\sqrt{n}\bigr)_+,
$
and for two-stage methods, obtained from first-stage variance estimates at the effective sample size. The last panel measures Lagrangian regret relative to each method's own effective constraint levels.}
    \label{fig:two-stage}
\end{figure}

In \Cref{fig:two-stage}, we compare several ways of solving the same constrained optimization problem. The \textit{naive} method solves the empirical constrained problem directly, with no extra slack. The \textit{one-stage} baseline adds a single worst-case variance slack, while \textit{two-stage} implements our procedure: it uses a first split to identify a candidate solution and nearly binding constraints, then re-solves a localized problem on a second split with variance-based slacks. The \textit{two-stage-75} variant uses the same localization idea with larger effective subsamples (75\% of the dataset in a split), and \textit{Pareto-imp-multiobj} denotes the local Pareto-improvement variant around a first-stage optimal policy rather than the global constrained problem. Implementation details are deferred to Appendix \Cref{apx-two-stage}. Figure 4 reports, from left-to-right, key metrics measured on an out-of-sample test set: value regret, treatment disparity, coverage of the target fairness constraint, and a variant of Lagrangian regret that accounts for different conservativities due to different effective sample complexities of the one- vs. two-stage methods. The main pattern is that the naive approach is least conservative but often violates the fairness constraint out of sample, whereas the regularized baselines improve feasibility. 

Across the simulation settings, the naive approach is least conservative but frequently violates the fairness constraint out of sample, whereas the one-stage and two-stage procedures both improve constraint satisfaction. For the value regret, in general, fairness is imposed at some value loss; there is a trade-off in general since compliance factors into the encouragement score as well. The advantage of the two-stage method is most visible in the moderate-sample regime, where it attains similar or better policy value while delivering tighter control of disparity. At larger sample sizes, this advantage narrows as the cost of sample splitting becomes less worthwhile.



\subsection{Case study: Decision-making framework for electronic monitoring.}\label{subsec-supervised-release}

We conduct a case study on a dataset of judicial decisions on \textit{supervised} release based on risk-score-informed recommendations for supervised release under an electronic-monitoring program \citep{cookcountybailreform}. The Public Safety Assessment Decision-Making Framework (PSA-DMF) uses a prediction of failure to appear for a future court date to inform pretrial decisions, including supervised release with electronic monitoring. An existing decision-making matrix recommends supervised release based on thresholds of failure-to-appear (FTA) and new-criminal-activity (NCA) predictive risk scores \citep{psa-dmf}. To our best understanding, such thresholds were not previously optimized on data.\footnote{We focus on supervised release with electronic monitoring, though the broad term  \textit{supervised release} encompasses substantially  different programs nationwide, including access to supportive services and caseworkers, which has been touted as a factor enabling bail reform and release more broadly \citep{bloomberg-bailreform}.} These thresholds on predictive risk scores need not align with optimal thresholds on encouragement scores.
There are current policy concerns about disparities in the increasing use of supervised release, given mixed evidence on its outcomes \citep{cookcountybailreform,communityrenewalsocietyLetterRegarding}; e.g., \citet{safetyandjusticechallenge} concludes that "targeted efforts to reduce racial disparities [in supervising release] are necessary". We focus on a publicly available dataset from Cook County, Illinois, with information about defendant characteristics, algorithmic recommendations for electronic monitoring, detention/release/supervised release decisions,  failure to appear and other outcomes \citep{cookcountybailreform-data}. 

 All of our analysis occurs in the released population only.
 We let $X$ denote covariates (age, top charge category, PSA failure to appear/new criminal arrest (FTA/NCA) score bucket and flag). The (binarized) protected attribute is $A$: race (nonwhite/white) or gender (female/male). The algorithmic recommendation is $R$, a recommendation from the PSA-DMF matrix for supervised release (at any intensity of supervision conditions). The treatment $T$ is whether the individual is released under supervision (at any intensity of supervision conditions). The outcome variable, $Y$, is failure to appear ($Y=1$).

In descriptive diagnostics included in \Cref{apx-psa}, we compare the heterogeneity in treatment effects and responsiveness. We find estimated heterogeneity in treatment efficacy and compliance, as well as group-level differences in distributions. We also observe disparities in how responsive decision-makers are to recommendation, conditional on
the same treatment effect efficacy. This is importantly not a claim of animus because decision-makers
didn’t have access to causal effect estimates. Nonetheless, disparities persist.

\begin{figure}[t!]
        \includegraphics[width=0.25\textwidth]{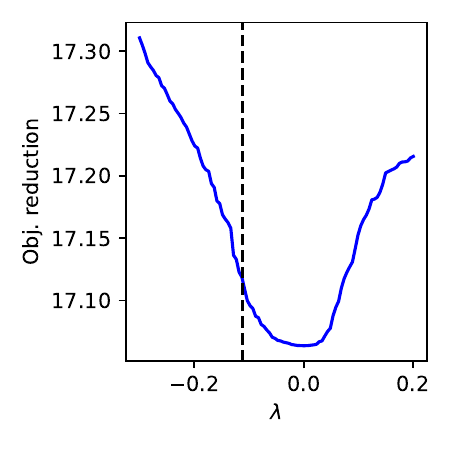}\includegraphics[width=0.25\textwidth]{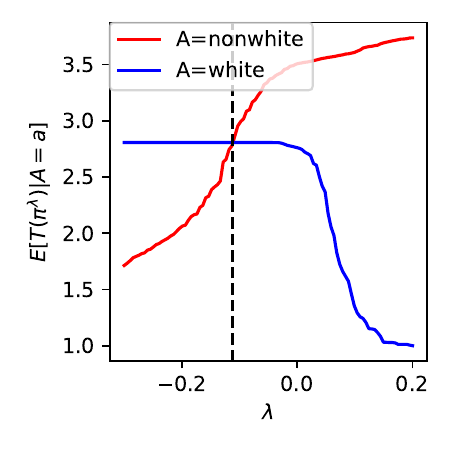}\includegraphics[width=0.25\textwidth]{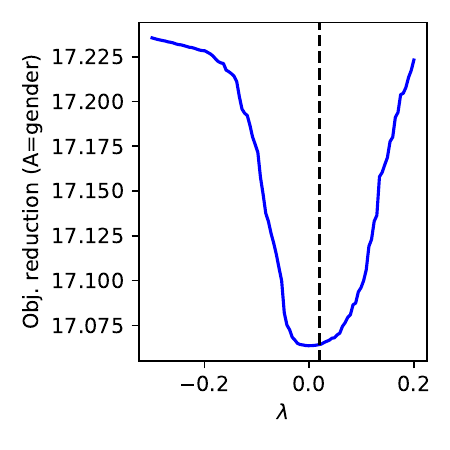}\includegraphics[width=0.25\textwidth]{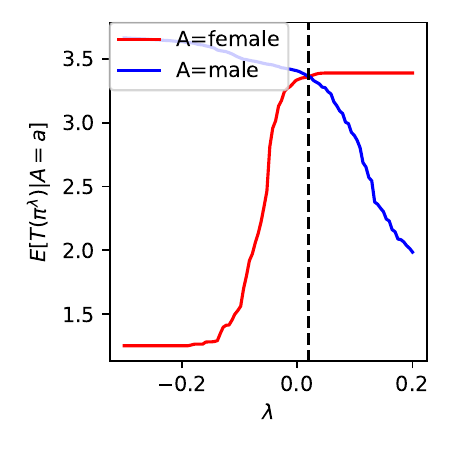}
        \vspace{-10pt}
    \caption{Pretrial supervised release case study: policy optimization via enumeration over $\lambda$ in penalty formulation. Local scalarized score $100\,\E[Y(\pi^\lambda)] + 20\,\E[T(\pi^\lambda)]$ and treatment rate $\E[T(\pi^\lambda)\mid A=a]$, for $A$ = race, gender.}
\label{fig:objconstr}
\end{figure}
In \Cref{fig:objconstr}, we highlight results from constrained policy optimization. The first two plots in each set illustrate the local scalarized score $100\,\E[Y(\pi)] + 20\,\E[T(\pi)]$ and the group-$A=a$ treatment rate for $A=$ race (nonwhite/white) and gender (female/male), respectively. This scalarization is specific to the supervised-release application. On the x-axis, we plot the penalty $\lambda$ that we use to assess the solutions of \Cref{prop-threshold-treatment-parity}. The vertical dashed line indicates the solution achieving $\epsilon = 0$, i.e., parity in treatment takeup. Near-optimal policies that reduce treatment disparity can be of interest given advocacy concerns about how the expansion of supervised release could increase the surveillance of already surveillance-burdened marginalized populations. We see that, indeed, for race, surveillance parity--constrained policies can substantially reduce disparities for nonwhite defendants while not increasing surveillance on white defendants that much: The red line decreases significantly with a low increase of the blue line (and low increases to the local scalarized score). On the other hand, for gender, the opportunity for improvement in the surveillance disparity is much smaller. 

\paragraph{Fairness-constrained optimization over interpretable scorecards.}
A key challenge in this setting is the lack of overlap in recommendations for supervised release. Prior work \citep{ben2021safe} addresses this by performing robust optimization over encouragements, effectively extrapolating outcome models. In contrast, our decomposition of encouragement into treatment take-up and treatment efficacy allows us to rely on treatment overlap while restricting robustness to the treatment response to recommendations. This is a milder extrapolation problem, as it can be guided by behavioral structure such as monotonicity.

%
We optimize over \textit{local regions} of potential improvement upon the existing status quo policy. We use \Cref{prop-robustlp} to optimize over uncertainty sets based on Lipschitz parametric extrapolation of high-overlap regions to low-overlap regions. The status quo policy is given by a decision-making framework on combinations of FTA and NCA scores, and we model $R=1$ if the highest levels of supervision are recommended. For privacy, the predictive risk scores were discretized into score buckets 1-2, 3-4, 5-6 etc. We optimize over the discrete space of \textit{interpretable scorecards}, i.e. comprising of recommendation decisions for each (FTA, NCA) pair because these are nearly the same form as the current decision-making matrix. We also consider covariate-richer \textit{age-split} scorecards that expand policy expressivity while remaining interpretable. (Age has widely been empirically shown to be a major driver of discretion in pretrial decision-making).

\begin{figure}
    \centering
    \includegraphics[width=\linewidth]{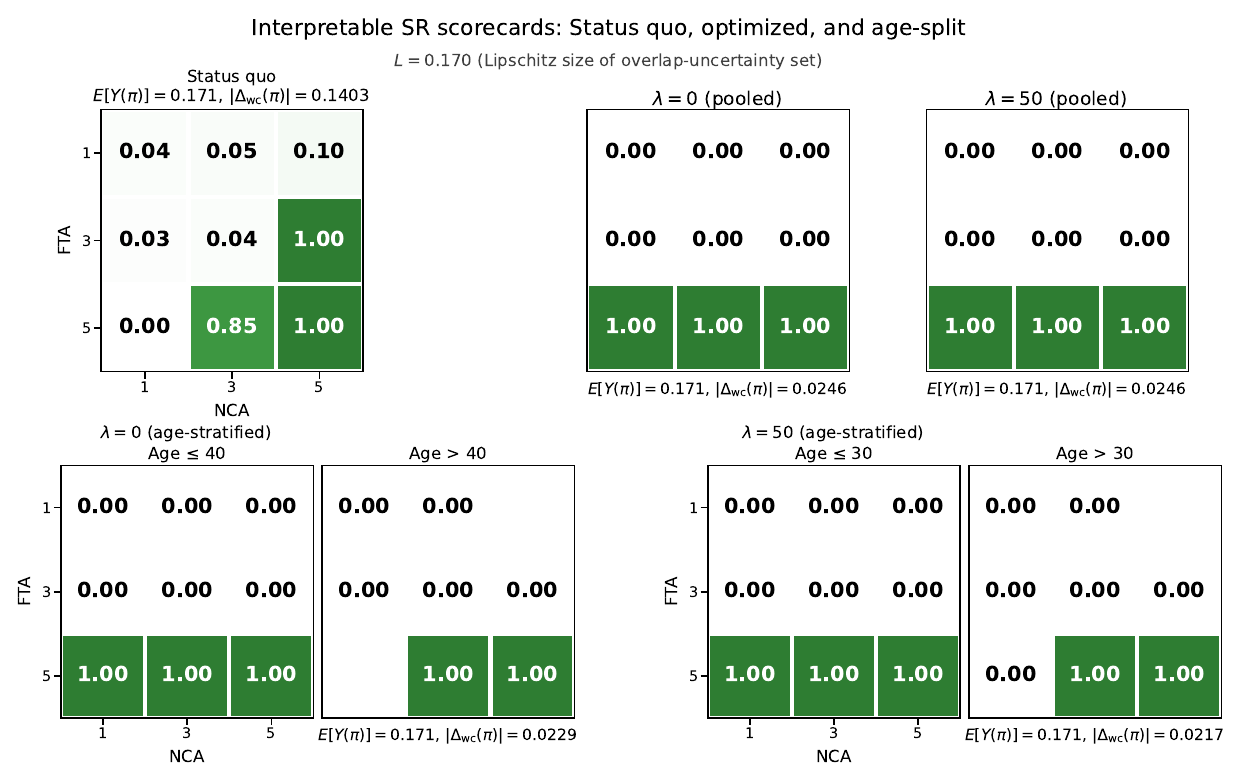}
    \caption{}
    \label{fig:scorecards-psa}
\end{figure}

To optimize over Lipschitz recommendation response models $p_{1\mid r}(x),$ we use a characterization of \citet{khan2024off} which reduces the Lipschitz uncertainty set on $p_{1\mid r}(x)$ to a per-datapoint interval uncertainty set, compatible with \Cref{prop-robustlp}. We calibrate the Lipschitz constant $L=0.17$ based on the largest effective Lipschitz constant observed on the good-overlap region. We also impose the additional compliance monotonicity constraint that $p_{1\mid 1}(x)-p_{1\mid 0}(x)\ge 0,\;\forall x$. To preserve interpretability, we optimize over \textit{monotone scorecards}, i.e. recommendation within one $(FTA, NCA)$ pair requires recommendation for larger $(FTA, NCA)$ pairs in a partial order. This partial order is pre-calculated and imposed as another linear constraint. Finally, under the status quo policy, about $18\%$ are recommended stricter supervised release: we assume this treatment budget holds and impose it as an additional constraint. 
These structural constraints can be collected via linear polytope constraints on \Cref{prop-robustlp}. We consider a scalarized outcome with FTA:treatment at a 100:20 ratio. See the appendix \Cref{apx-addldisc-scorecard} for more details and the full formulation.

\Cref{fig:scorecards-psa} displays the results, where each ``scorecard'' has FTA scores on the y-axis, and NCA scores on the x-axis. Even the pooled cost-optimal scorecard substantially reduces worst-case disparity relative to the status quo (0.0246 versus 0.0572). Allowing age-stratified scorecards improves worst-case disparity slightly further, to 0.0229 or 0.0217. Thus, local changes to simple decision matrices can reduce disparity while remaining interpretable. It is important that in this situation, \textit{even cost-minimizing scorecards without explicit fairness penalty can improve worst-case disparities relative to the current status quo, subject to our robust Lipschitz extrapolation partial identification intervals}. Although our local improvements over interpretable scorecards are more restricted compared to our full policy optimization earlier, we can therefore provide concrete recommendations for local changes to status-quo decision-making matrices that could improve disparities at little cost in FTA. Optimizing over covariate-rich data-driven recommendation policies can mitigate fairness-accuracy concerns by providing Pareto-improving local improvements upon the status quo. 

\begin{figure}
    \centering
\includegraphics[width=\linewidth]{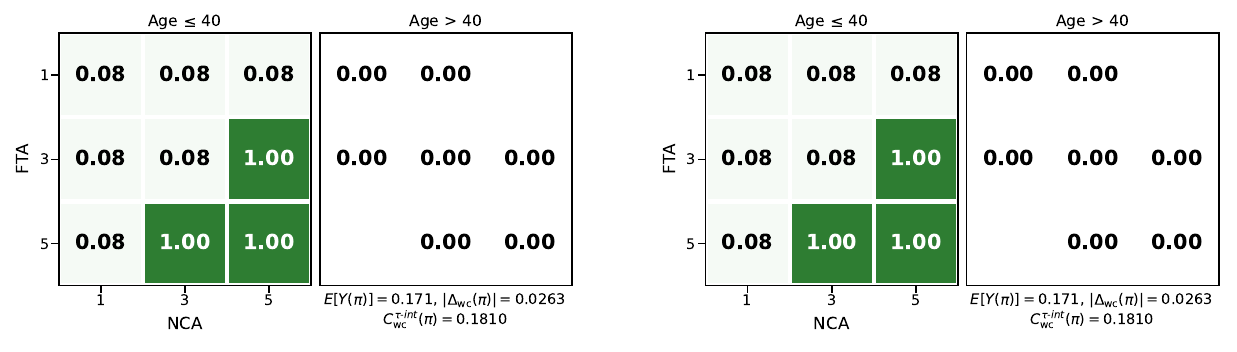}
    \caption{Interpretable supervised release scorecards optimized over uncertainty sets on $\tau$ alone with $\Gamma = 1.247$. Worst-case disparities assessed with $L=0.204$.}
    \label{fig:robust-scorecards-unconf}
\end{figure}

\begin{remark}[Robustness under violations of unconfoundedness]
One concern could be potential violations of \Cref{asn-unconfoundedness} (unconfoundedness), in that decisions about treatment relied on additional information that is unmeasured in our dataset. In \Cref{apx-sensanalysis-unconf}, we leverage the robust linear program \Cref{prop-robustlp} to instead optimize over bounds on heterogeneous treatment effects under bounded violations of treatment unconfoundedness. Further, we audit worst-case disparities of the confounding-robust-optimal policies over the same Lipschitz uncertainty set on overlap. In \Cref{fig:robust-scorecards-unconf} we find that our main findings hold directionally, in that similar \textit{confounding-robust} scorecards can robustly control disparities estimated at $~0.0263$, though the bounds result in more conservative assessment of who stands to benefit.
\end{remark}

\section{Conclusion}
Across these case studies, the main managerial implication is that organizations should explicitly model both treatment take-up and treatment efficacy, and audit disparities in each separately, because gaps in either mechanism suggest different managerial responses. In the SNAP setting, our analysis can quantify the benefits and limits of improvements from targeting. Low-touch interview reminders improve outcomes on average, but for some more than others --- operational redesign of outreach, access, defaults, or service navigation can ensure that those who stand to gain the most do. Our extension to algorithmic advice in the supervised-release case study illustrates that searching over covariate-rich personalized rules can uncover local Pareto improvements over simpler, non-data-driven status-quo rules, reducing disparity with little increase in objective costs. 


\bibliography{beyondfairness,encouragement,auto-semipara}
\bibliographystyle{abbrvnat}

\clearpage
\begin{APPENDICES}

\noindent\textbf{Appendix contents.}
\begin{itemize}[noitemsep,topsep=2pt]
  \item \Cref{sec-addldisc}: Additional Discussion (additional related work; alternative identification)
  \item \Cref{apx-addldisc-fairnessanalysis}: A simple model with low fairness-value tradeoff
  \item \Cref{apx-addldisc-method}: Additional discussion on method (extrapolation; overlap)
  \item \Cref{apx-disc-opt}: Additional discussion on general optimization method (constrained optimization; saddle-point algorithm)
  \item \Cref{sec-proofs}: Proofs (identification and characterization; generalization; robust characterization; general algorithm)
  \item \Cref{sec-addl-experiments}: Additional case studies and experimental details (SNAP; simulation; OHP; PSA-DMF)
\end{itemize}

\section[Additional Discussion]{Additional Discussion}\label{sec-addldisc}

\subsection{Additional Related work}

\textbf{Intention-to-treat analysis.} 
We appeal to intention-to-treat analysis with randomness that either arises from human decision-makers or individual nonadherence/noncompliance, but we generally assume the data does not include information about the \textit{identity} of \textit{different} decision-makers, which is common with publicly available data. Our conditional exclusion restriction also means that certain decomposed effects are zero, so mediation analysis is less relevant. 
A related literature studies principal stratification \citep{jiang2020multiply}, which is less interpretable since stratum membership is unknown. Similarly, even though encouragement effects are driven by compliers, complier-conditional analysis is less policy-relevant since complier identities are unknown. In general, our causal identification arguments are based on covariate-adjusted intention-to-treat analysis and covariate-adjusted as-treated analysis. We avoid estimation of stratum-specific effects, because if complier status is unknown, prescriptive decision rules cannot directly personalize by stratum membership.

\textbf{Fair off-policy learning.}
We highlight some most closely related works in off-policy learning (omitting works in the sequential setting). \citep{metevier2019offline} studies high-probability fairness constraint satisfaction.
\citep{kim2022doubly} studies doubly-robust causal fair classification, while others have imposed deterministic resource constraints on the optimal policy formulation \citep{chohlas2021learning}. Other works study causal or counterfactual risk assessments \citep{mishler2021fairness,coston2020counterfactual}. Our perspective is closer to that of off-policy learning, i.e. approximating direct control over the intervention by assuming stability in decision-maker treatment assignment probabilities.  \citep{kallus2019assessing} studies (robust) bounds for treatment responders in binary outcome settings; this desiderata is coupled to classification notions of direct treatment. Again, our focus is on modeling the fairness implications of non-adherence. Indeed, in order to provide general algorithms and methods, we do build on prior fair classification literature. A different line of work studies ``counterfactual" risk assessments which models a different concern. %

\textbf{Principal stratification and mediation analysis in causal inference. }\citep{liu2021efficient} studies an optimal test-and-treat regime under a no-direct-effect assumption, that assigning a diagnostic test has no effect on outcomes except via propensity to treat, and studies semiparametric efficiency using Structural Nested-Mean Models. Though our exclusion restriction is also a no-direct-effect assumption, our optimal treatment regime is in the space of recommendations only as we do not have control over the final decision-maker, and we consider generally nonparametric models. 

We briefly go into more detail about formal differences, due to our specific assumptions, that delineate the differences to mediation analysis. Namely, our conditional exclusion restriction implies that $Y_{1T_0} = Y_{T_0}$ and that $Y_{0T_1} = Y_{1T_1}$ (in mediation notation with $T_r = T(r)$ in our notation), so that so-called \textit{net direct effects} are identically zero and the \textit{net indirect effect} is the treatment effect (also called average encouragement effect here). 

\textbf{Other work on algorithmic fairness encouraging access in operations.} 
There is extensive work on algorithmic fairness in operations, measured in different ways. \citet{bertsimas} studies efficiency trade-offs versus Nash social welfare fairness. 
On the non-algorithmic advice, but encouragement side, \citet{freund2025fair} consider dynamic resource allocation under fairness constraints, when deciding monetary incentives for retention under stochastic participation. Other models of stochastic resource usage and demand are also relevant, though potentially more specialized than our non-adherence setting. 

\textbf{Human-in-the-loop in consequential domains.}
There is a great deal of interest in designing algorithms for the ``human in the loop" and studying expertise and discretion in human oversight in consequential domains \citep{de2020case}. On the algorithmic side, recent work focuses on frameworks for learning to defer or human-algorithm collaboration. Our focus is \textit{prior} to the design of these procedures for improved human-algorithm collaboration: we primarily hold fixed current human responsiveness to algorithmic recommendations. Therefore, our method can be helpful for optimizing local nudges. Incorporating these algorithmic design ideas would be interesting directions for future work.  

\textbf{Empirical literature on judicial discretion in the pretrial setting.}
Studying a slightly different substantive question, namely causal effects of pretrial decisions on later outcomes, a line of work uses individual judge decision-makers as a leniency IV for the treatment effect of (for example, EM) on pretrial outcomes \citep{arnold2022measuring,arnold2018racial,lum2017causal}. And, judge IVs rely on quasi-random assignment of individual judges. We focus on the prescriptive question of optimal recommendation rules in view of patterns of judicial discretion, rather than the descriptive question of causal impacts of detention on downstream outcomes. 

A number of works have emphasized the role of judicial discretion in pretrial risk assessments in particular \citep{green2021algorithmic,doleac2020algorithmic,ludwig2021fragile}. In contrast to these works, we focus on studying decisions about electronic monitoring, which is an intermediate degree of decision lever to prevent FTA that nonetheless imposes costs.  \citep{imai2020experimental} study a randomized experiment of provision of the PSA and estimate (the sign of) principal causal effects, including potential group-conditional disparities. They are interested in a causal effect on the principal stratum of those marginal defendants who would not commit a new crime if recommended for detention. \citep{ben2021safe} study policy learning in the absence of positivity (since the PSA is a deterministic function of covariates) and consider a case study on determining optimal recommendation/detention decisions; however their observed outcomes are downstream of judicial decision-making. Relative to their approach, we handle lack of overlap via an exclusion restriction so that we only require ambiguity on \textit{treatment responsivity models} rather than causal outcome models.

\subsection{Alternative identification}\label{apx-addldisc-identification}

\paragraph{Alternative IV-style interpretation following Wang and Tchetgen Tchetgen.}
If one prefers to avoid our stronger treatment-unconfoundedness assumption and the corresponding heterogeneous-treatment-effect language, an alternative route available when the recommendation is truly randomized is to interpret the recommendation \(R\) as an instrumental variable for treatment \(T\). In our notation, let
\[
\Delta_Y(x,a)
:=
\E[Y \mid R=1,X=x,A=a]-\E[Y \mid R=0,X=x,A=a]
\]
The corresponding conditional Wald estimand is
\[
\delta_W(x,a):=\frac{\Delta_Y(x,a)}{\kappa(x,a)}.
\]
Under the Wang--Tchetgen Tchetgen IV setup, one assumes: (i) exclusion, \(Y(t,r)=Y(t)\) for all \(t,r\); (ii) IV independence from unmeasured confounding, \(R \indep U \mid X,A\); (iii) IV relevance, \(\E[T\mid R=1,X,A]\neq \E[T\mid R=0,X,A]\); and (iv) \(Y(t)\indep (T,R)\mid X,A,U\). They then show that the conditional ATE is identified by the conditional Wald estimand if one additionally assumes either
\[
\E[T\mid R=1,X,A,U]-\E[T\mid R=0,X,A,U]
=
\E[T\mid R=1,X,A]-\E[T\mid R=0,X,A],
\]
that is, no additive \(U\)--\(R\) interaction in the first stage, or
\[
\E[Y(1)-Y(0)\mid X,A,U]
=
\E[Y(1)-Y(0)\mid X,A],
\]
that is, no additive \(U\)--\(t\) interaction in the treatment effect. Under these assumptions,
\[
\delta_W(x,a)=\E[Y(1)-Y(0)\mid X=x,A=a].
\]

\paragraph{Interpretation relative to our main analysis.}
Under our main text assumptions, we write the encouragement score as
\[
s(x,a)=\Delta_Y(x,a)=\kappa(x,a)\tau(x,a),
\]
where \(\tau(x,a)\) is an as-treated heterogeneous treatment effect identified by treatment unconfoundedness. Under the Wang--Tchetgen Tchetgen alternative, the same reduced-form score is instead written as
\[
s(x,a)=\Delta_Y(x,a)=\kappa(x,a)\delta_W(x,a).
\]
Thus the policy-relevant reduced-form score \(s(x,a)\) and the first-stage term \(\kappa(x,a)\) are unchanged; what changes is the interpretation of the multiplier. In particular, one replaces the language of ``heterogeneous treatment effects'' with the language of conditional Wald effects. If one further strengthens the IV assumptions to a \emph{causal IV} assumption and adds monotonicity,
\begin{equation*}  \textstyle
R \indep \{T(1),T(0),Y(1),Y(0)\}\mid X,A,
\qquad
T(1)\ge T(0)\ \text{a.s.},
\end{equation*}
then the same \(\delta_W(x,a)\) also admits a conditional LATE (CLATE) interpretation, so that one may equivalently write $s(x,a)=\kappa(x,a)\tau_c(x,a),$ with \(\tau_c(x,a)\) denoting a conditional complier-local average treatment effect.Then 
\[
\tau_c(x,a)
:=
\E\!\left[Y(1)-Y(0)\mid T(1)>T(0),\,X=x,\,A=a\right].
\]
Thus, under monotonicity, the term multiplying the first stage is no longer an as-treated heterogeneous treatment effect for the full stratum \((x,a)\), but rather the average treatment effect for the latent subgroup of compliers within that stratum. In this case, $\kappa(x,a)=\Pr(T(1)>T(0)\mid X=x,A=a)$
is the conditional complier mass, and $s(x,a)=\kappa(x,a)\tau_c(x,a)$ decomposes the value of encouragement into an extensive-margin term (how many are induced into treatment) and an intensive-margin term (how much treatment helps those induced units). Therefore the policy ranking is unchanged, but the interpretation of treatment-effect heterogeneity becomes local to compliers rather than population-wide.

Under all such alternative specifications, the core message remains the same: naive ITT targeting depends on a product of a term describing first-stage selection into treatment (whether they are restricted to compliers, as in the standard causal IV case with monotonicity, or not, without monotonicity) and a term describing treatment effectiveness (whether conditional to compliers or not).


\section{Additional discussion on \Cref{sec-fairness-analysis}}

\subsection{A simple model with low fairness-value tradeoff.}\label{apx-addldisc-fairnessanalysis}
\paragraph{The model}
To concretize, consider a simple model with covariates partitioned as $(X^s, X^\kappa, X^\tau)$ where $X$ is a shared factor affecting both compliance and heterogeneous treatment effect, $X^\kappa$ modifies compliance alone, and $X^\tau$ modifies the heterogeneous treatment effect alone. Consider the following model: 
\begin{align}
&X = (X^s, X^\tau, X^\kappa), \;\;\; X^s \sim N(0,1),\;\;\;  X^\tau \sim N(0,1), \;\;\; X^\kappa \mid A=g \sim N(\mu_g,1)
\label{eqn-simple-gaussian-model}\\
    & \tau(X,A) = \log(1+\exp(\beta_s X^s + \beta_\tau X^\tau )),\qquad 
    p_{1\mid 0}(X,A) 
    = \bar{p}_0, \quad \kappa(X,A) = (1-\bar{p}_0) \Phi(\theta^\kappa X^\kappa + \theta^s X^s) 
\end{align}
  Importantly, in the simple model above, the protected attribute $A$ has no direct effect on compliance or the heterogeneous treatment effect. However, $A$ impacts the distribution of a compliance-relevant covariate $X^\kappa$. For example, $X^\kappa$ could be zip code that affects location and therefore distance to medical services. With other health information about a person, zip code wouldn't plausibly biologically affect treatment effectiveness, but could affect adherence/compliance. 
  
The simple model is primarily diagnostic: it shows that large disparities under naive ITT targeting can arise even when treatment benefits are identical across groups. As a secondary implication, it also identifies a regime in which suppressing compliance-only proxies can achieve a favorable fairness–value tradeoff, thereby motivating the fairness-constrained methods developed later. However, it may not be possible to implement group-specific thresholds and it may be \textit{procedurally fairer} and preferable to optimize policies over $x$ alone rather than $x,a$. Let $\pi_{DP}^{blind}(X)$ denote a solution to \cref{eqn-demographic-parity-budget} but optimizing over policies that are measurable with respect to $(X^s,X^\tau)$. 

\begin{proposition}\label{prop-simple-model-results}
   Under the above model of \Cref{eqn-simple-gaussian-model}, $$\textstyle 
   \E[\kappa(X,A)\mid X^s, A=g]  = (1-\bar{p}_0) \Phi \left( \frac{\theta^\kappa \mu_g + \theta^s X^s}{\sqrt{1+\left(\theta^\kappa\right)^2
   }} \right), \;\; \E[\kappa(X,A)\mid  A=g]  = (1-\bar{p}_0) \Phi \left( \frac{\theta^\kappa \mu_g }{\sqrt{1+\left(\theta^\kappa\right)^2+\left(\theta^s\right)^2
   }} \right)$$

   We also have that $\E[ \pi_{DP}^{blind}(X^s,X^\tau)\mid A=a] - \E[ \pi_{DP}^{blind}(X^s,X^\tau)\mid A=b]=0$. 
The fairness-value tradeoff is bounded by:
$$ V(\pi^*) - V(\pi_{DP}^{blind}) \leq \abs{s(X,A) - s(X^s, X^\tau)}\leq  \frac{(1-\bar{p}_0)
       \abs{\theta^\kappa }
       \E[ \abs{\tau(X^s,X^\tau)}]\E[\abs{ (X^\kappa - \E[X^\kappa ])
    }}{\sqrt{2\pi}}.
$$
with the last inequality holding in the above model of \Cref{eqn-simple-gaussian-model}.
\end{proposition}

\begin{remark}
\Cref{prop-simple-model-results} characterizes fairness-value trade-offs in the simple proxy compliance model. The disparity $\Delta_\kappa
=
\left|\E[\kappa(X,A)\mid A=1]-\E[\kappa(X,A)\mid A=0]\right|$
captures the compliance disparity induced by $X^\kappa$, which we interpret as a proxy for the fairness benefit of dropping $X^\kappa$. (It is a proxy because the fairness benefit depends on the budget in general). In contrast, $\theta^\kappa\E|X^\kappa-\E[X^\kappa]|.$
Proxy-blindness achieves low fairness-welfare trade-off when $X^\kappa$ is mostly a marker of group difference, rather than mostly a source of individual-level compliance prediction within each group.\end{remark}

\paragraph{Interpretation}  We set the parameter vector to $\beta_\tau = 1.25$, $\lambda^s = 0.9$ (treatment), $\mu_a = 2$, $\mu_b = -2$, $\lambda^\kappa = 0.18$, $\bar p_0 = 0.2$, and $c_\kappa = 0.5$. The unconstrained policy $\pi^*$ attains value $V(\pi^*) = 0.2159$ at budget $b = 0.3$. Under the unconstrained policy, encouragement disparity is $\E[\pi^*(X,A)\mid A=a]- \E[\pi^*(X,A)\mid A=b] = 0.1232$ and treatment disparity is $|\mathbb{E}[T(\pi^*)\mid A{=}a] - \mathbb{E}[T(\pi^*)\mid A{=}b]| = 0.0652$. A compliance-proxy blind policy $\pi^{\mathrm{blind}}$ (ranking by $(X^s,X^\tau)$ alone) attains $V(\pi^{\mathrm{blind}}) = 0.2128$ with encouragement disparity $0.0088$ and treatment disparity $0.0328$. Therefore the value gap from blinding is $V(\pi^*) - V(\pi^{\mathrm{blind}}) = 0.0031$ (a relative 1.4\% drop) although it eliminates the $0.1232$ encouragement disparity. Therefore, in this simple model, covariate-rich personalized policies offer more degrees of freedom to find decision alternatives that could greatly reduce outreach disparities at modest loss of value. Of course, the extent of such trade-offs in general depends on the data and context in any particular application. 

\begin{table}[ht]
\centering
\caption{Budget--access--value tradeoffs at budget $b = 0.30$. Compliance disparity $\Delta_\kappa = 0.1071$.}
\label{tab:tradeoff}
\begin{tabular}{lcccc}
\toprule
Policy & $V(\cdot)$ & $V(\pi^*){-}V(\cdot)$ & 
Budget disp.
& Treatment disp. \\
\midrule
$\pi^*$ & 0.2159 & 0 & 0.1232 & 0.0652 \\
$\pi^{blind}$ & 0.2128 & 0.0031 & 0.0088 & 0.0328 \\
\bottomrule
\end{tabular}
\end{table}

\section{Additional discussion on method }\label{apx-addldisc-method}

\subsection{Extrapolation in \Cref{setting-algorec}}
A naive approach based on parametric extrapolation estimates $p_{1\mid 1}(X)$, treatment responsivity, on the observed data and simply uses the parametric form to extrapolate to the full dataset.
In the case study later on, the support of $X\mid R=1$ is a superset of the support of $X\mid R=0$ in the observational data. 
Given this, we derive the following alternative identification based on marginal control variates (where $p_t = P(T=t\mid X)$ marginalizes over the distribution of $R$ in the observational data): 
\begin{proposition}[Control variate for alternative identification ]\label{prop-nooverlap} 
Assume $Y(T(r)) \perp T(r) \mid R=r,X$.
  \begin{align*} \textstyle  V(\pi) = \sum_{t\in\tspace, r\in\{0,1\}} \E\left[ \left\{ Y(t)  \frac{\indic{T=t}}{p_{t}(X)} 
 + \left( 1-\frac{\indic{T=t}}{p_{t}(X)}\right) \mu_{t}(X)
 \right\}
\ptmidrx \right] 
 \end{align*} 
\end{proposition}

On the other hand, parametric extrapolation is generally unsatisfactory because conclusions will be driven by model specification rather than observed data. Nonetheless, it can provide a starting point for robust extrapolation of structurally plausible treatment response probabilities.
\section{Additional discussion on general optimization method}\label{apx-disc-opt}
\subsection{Additional discussion on constrained optimization}

\paragraph{Feasibility program }
We can obtain upper/lower bounds on $\epsilon$ in order to obtain a feasible region for $\epsilon$ by solving the below optimization over maximal/minimal values of the constraint: 
\begin{align}
\overline{\epsilon},\underline{\epsilon} 
\in \max_{\pi} / \min_{\pi} \E[T(\pi)\mid A=a]- \E[T(\pi)\mid A=b] \label{eqn-feas-constr}
\end{align}
\begin{align}%
\max_{\pi} \; 
\{V(\pi) \colon \E[T(\pi)\mid A=a]- \E[T(\pi)\mid A=b]  \leq \epsilon \}
\end{align}

\subsection{Additional discussion on \Cref{alg-metaalg2} (general algorithm)}\label{sec-saddle}

\subsubsection{Constrained policy optimization oracle}\label{apx-pol-opt}
Although Agarwal et al.\ (2018) use the more general notation $Mh(\pi)\le d$, in our setting these constraints are encoded directly as coordinates of $\Delta(\pi)$, so we use $\Delta(\pi)\le d$ in the main text.
\citet{agarwal2018reductions} consider a generic formulation of fairness constraints as moment conditions on $\pi$. Their key algorithmic design choices include convexification of the space of policies, that is, they optimize over distributions over policies $Q\in \mathcal{P}(\Pi),$ where $\mathcal{P}(\Pi)$ is the space of distributions over policies. The linear constraints are encoded as coordinates $\Delta_k(\pi)\le d_k$ of the moment vector $\Delta(\pi)$, with observed-data scores $\delta_k(O;\pi,\eta)$ satisfying $\E[\delta_k(O;\pi,\eta_0)]=\Delta_k(\pi)$.
The elements of $\Delta_k(\pi)$ are average functionals (for example, the average treatment takeup in group $k$).
Importantly, the moment function $\delta_k$ depends on $\pi$, while the conditioning event cannot depend on $\pi$. Many important fairness constraints can nonetheless be written in this framework, such as burden/resource parity and parity in true positive rates, but not measures such as calibration whose conditioning event does depend on $\pi$. (See \Cref{sec-saddle} for examples omitted for brevity.) %

\paragraph{Implementing best-response oracle }
Given $\lambda_t,$ the algorithm computes a best response over $Q$ ($\operatorname{BEST}_\beta\left({\lambda}_t\right)$); since the worst-case distribution will place all its weight on one classifier, this step reduces to cost-sensitive/weighted classification \citep{beygelzimer2009offset,zhao2012estimating}, which we describe in further detail below. Computing the best response over $\operatorname{BEST}_{{\lambda}}(\hat{Q}_t))$ selects the most violated constraint. 
\subsubsection{Weighted classification reduction and off-policy estimation}\label{sec-saddle-weighted-class-reduction}

There is a well-known reduction of optimizing the zero-one loss for policy learning to weighted classification. Note that the reductions approach of \citep{agarwal2014taming} works with the Lagrangian relaxation which only further introduces datapoint-dependent additional weights. Notationally, in this section, for policy optimization, $\pi\in\{-1,+1\},T\in\{-1,+1\}$ (for notational convenience alone). We consider parameterized policy classes so that $\pi(x)=\pi(1\mid x) = \operatorname{sign}(f_\beta(x))$ for some index function $f$ depending on a parameter $\beta \in \mathbb{R}^d$.
Consider the centered regret $J(\pi) = \E[Y(\pi)] - \frac 12 \E[\E[ Y\mid R=1,X] + \E[ Y\mid R=0,X]].$ 
We summarize different estimation strategies via the score function $\psi_{(\cdot)}(O),$ where $(\cdot) \in \{DM, IPW, DR\}$: the necessary property is that ${\E[ \psi \mid X] = \E[ Y\mid R=1,X] - \E[ Y\mid R=0,X]}$. The specific functional forms of these different estimators are as follows,  
where $\mu^R_r(X) = \E[Y\mid R=r,X]:$  
$$\textstyle \psi_{DM} = (p_{1\mid 1}(X)-p_{1\mid 0}(X)) (\mu_1(X)-\mu_0(X)), \psi_{IPW} = \frac{RY}{e_R(X)}, \psi_{DR} = \psi_{DM} + \psi_{IPW} + \frac{R\mu^R(X)}{e_R(X)}.$$

\subsubsection{Additional fairness constraints and examples in this framework}
In this section we discuss additional fairness constraints and how to formulate them in the generic framework. Much of this discussion is quite similar to \citep{agarwal2018reductions} (including in notation) and is included in this appendix for completeness only. We additionally provide concrete discussion of the reduction to weighted classification, and concrete descriptions of the causal fairness constraints in the more general framework. 

We first discuss how to impose the treatment parity constraint. This is similar to the demographic parity example in \cite{agarwal2018reductions}, with different coefficients, but included for completeness. (Instead, recommendation parity in $\E[\pi\mid A=a]$ is indeed nearly identical to demographic parity.)
\begin{example}[Writing treatment parity in the general constrained classification framework.]
  We write the constraint  
  \begin{equation}\E[T(\pi)\mid A=a]- \E[T(\pi)\mid A=b] \label{apx-tp-constraint}
  \end{equation}
  in this framework as follows: 
$$
\E[T(\pi)\mid A=a] 
=\E[ \pi_1(X) (p_{1\mid 1}(X,A)-p_{1\mid 0}(X,A)) + p_{1\mid 0}(X,A) \mid A=a]
$$
  For each $u \in \mathcal{A}$ we enforce that 
$$\textstyle \sum_{r\in\{0,1\}} \E\left[ \pi_r(X)p_{1\mid r}(X,A)\mid A=u\right] 
= \sum_{r\in\{0,1\}} \E\left[ \pi_r(X,A)p_{1\mid r}(X,A)\right] 
$$
The score $\delta_1(O;\pi,\eta)$ that identifies the upper constraint $\Delta_1(\pi)=\E[T(\pi)\mid A=a]-\E[T(\pi)\mid A=b]$ is
\[
\delta_1(O;\pi,\eta)
=
\frac{\mathbf{1}\{A=a\}}{p_a}\sum_{r}\pi_r(X)\,p_{1\mid r}(X,A)
-
\frac{\mathbf{1}\{A=b\}}{p_b}\sum_{r}\pi_r(X)\,p_{1\mid r}(X,A),
\]
where $p_u:=P(A=u)$, and $\delta_2(O;\pi,\eta)=-\delta_1(O;\pi,\eta)$ for the lower constraint $\Delta_2(\pi)\le\varepsilon$.
One can verify $\E[\delta_k(O;\pi,\eta_0)]=\Delta_k(\pi)$ coordinatewise.

Both scores are \emph{linear in the policy}: using
$\sum_r\pi_r(X)\,p_{1\mid r}(X,A)
 =\pi_1(X)\,(p_{1\mid 1}(X,A)-p_{1\mid 0}(X,A))+p_{1\mid 0}(X,A)$,
we can write
$\delta_1(O;\pi,\eta)=\pi_1(X)\,\tilde w(O;\eta)+c(O;\eta)$
for known weight $\tilde w$ and offset $c$ that depend only on nuisances, not on $\pi$.
This linearity in $\pi$ implies linearity of $\Delta_k(Q)=\E[\delta_k(O;Q,\eta_0)]$ in the randomized policy $Q$, which is what enables the saddle-point reductions approach of \citet{agarwal2018reductions}.
This yields two coordinates of $\Delta(\pi)$, corresponding to the upper and lower inequality constraints.

\end{example}

In binary monotone settings, the same framework can also encode responder-conditional constraints, but we omit the details because they are not used in the main text or case studies.

\subsubsection{Best-response oracles}

\paragraph{Best-responding classifier $\pi$, given $\lambda$: $\op{BEST}_\pi(\lambda)$}
The best-response oracle, given a particular $\lambda$ value, optimizes the Lagrangian given $\pi$: 

$\mathcal{L}(\pi,\lambda)=\hat V(\pi)+\lambda^\top(\hat\Delta(\pi)-d)$
\paragraph{Best-responding Lagrange multiplier $\lambda$, given $\pi$: }
$\operatorname{BEST}_{\boldsymbol{\lambda}}(Q)$ is the best response of the $\Lambda$ player. It can be chosen to be either $0$ or put all the mass on the most violated constraint. 
$\operatorname{BEST}_{\boldsymbol{\lambda}}(Q)$ returns  $$\begin{cases}\mathbf{0} & \text { if } \widehat{\Delta}(Q) \leq \widehat{\mathbf{c}} \\ B \mathbf{e}_{k^*} & \text { otherwise, where } k^*=\arg \max _k\left[\widehat{\Delta}_k(Q)-\widehat{c}_k\right]\end{cases}$$

\subsubsection{Weighted classification reduction }

There is a well-known reduction of optimizing the zero-one loss for policy learning to weighted classification. A cost-sensitive classification problem is
\begin{align*}
    \underset{\pi_1 }{\arg \min } \sum_{i=1}^n \pi_1 \left(X_i\right) C_i^1+\left(1-\pi_1\left(X_i\right)\right) C_i^0
\end{align*}
The weighted classification error is $\sum_{i=1}^n W_i {1}\left\{h\left(X_i\right) \neq Y_i\right\}$ which is an equivalent formulation if $W_i=\left|C_i^0-C_i^1\right|$ and $Y_i={1}\left\{C_i^0 \geq C_i^1\right\}$.

The reduction to weighted classification is particularly helpful since taking the Lagrangian will introduce datapoint-dependent penalties that can be interpreted as additional weights. 
 We can consider the centered regret $J(\pi) = \E[Y(\pi)] - \frac 12 \E[\E[ Y\mid R=1,X] + \E[ Y\mid R=0,X]]$. Then 
\begin{align*}  J(\theta) = J(\op{sgn}(g_\theta(\cdot))) = \E[ \op{sgn}(g_\theta(X)) \left\{ \psi \right\}]
\end{align*} 
where $\psi$ can be one of, where $\mu^R_r(X) = \E[Y\mid R=r,X],$ 
$$\psi_{DM} = (p_{1\mid 1}(X)-p_{1\mid 0}(X)) (\mu_1(X)-\mu_0(X)), \psi_{IPW} = \frac{RY}{e_R(X)}, \psi_{DR} = \psi_{DM} + \psi_{IPW} + \frac{R\mu^R(X)}{e_R(X)} $$

We can apply the standard reduction to cost-sensitive classification since $\psi_i \op{sgn}(g_\theta(X_i)) = \abs{\psi_i} (1-2 \indic{\op{sgn}(g_\theta(X_i)) \neq \op{sgn}(\psi_i)}$). Then we can use surrogate losses for the zero-one loss, 
$$ \mathcal{L}(\theta) = \E[\abs{\psi} \ell(g_\theta(X), \op{sgn}(\psi))] $$
Although many functional forms for $\ell(\cdot)$ are Fisher-consistent, the logistic (cross-entropy) loss will be particularly relevant: $ l(g,s) = 2 \log(1+\exp(g)) - (s+1)g$.

\begin{example}[Treatment parity, continued (weighted classification reduction)]
    
The cost-sensitive reduction for a vector of Lagrange multipliers can be deduced by applying the weighted classification reduction to the Lagrangian: 
$$ 
\mathcal{L}(\beta)=\mathbb{E}\left[|\tilde{\psi}^\lambda | \ell\left(f_\beta(X), \operatorname{sgn}(\tilde{\psi}^\lambda )\right)\right], \qquad \text{ where }
\tilde{\psi}^\lambda = \psi + \frac{\lambda_{A}}{p_{A}} (p_{1\mid 1}-p_{1\mid 0}) -\sum_{a \in \mathcal{A}} \lambda_a.$$
where $p_a:=\hat P(A=a)$ and $\lambda_a:=\lambda_{(a,+)}-\lambda_{(a,-)}$, effectively replacing two non-negative Lagrange multipliers by a single multiplier, which can be either positive or negative. 
\end{example}

\clearpage

\section{Proofs}\label{sec-proofs}

\subsection{Identification, estimation and characterization}

\begin{proof}{Proof of \Cref{prop-identification}}
    \begin{align*}\E[Y(\pi)]
&=  \textstyle \sum_{t\in\tspace, r\in\{0,1\}}
\E[\pi_r(X) \E[ \indic{T(r)=t}Y(t(r)) \mid R=r,X]]
\\
& \textstyle =\sum_{t\in\tspace, r\in\{0,1\}}
\E[\pi_r(X)  {P}(T=t\mid R=r,X)\E[ Y(t(r)) \mid R=r,X]] \label{eqn-identification} \\
& \textstyle =\sum_{t\in\tspace, r\in\{0,1\}}
\E[\pi_r(X)  {P}(T=t\mid R=r,X)\E[ Y \mid T=t,X]] 
\end{align*} 
where the last line follows by the conditional exclusion restriction (\Cref{asn-exclusionrestriction}) and consistency (\Cref{asn-consistency}). 

\end{proof}

\begin{proof}{Proof of \Cref{prop-simple-model-results}}
We will repeatedly use a simple lemma and Gaussian probit identity. Generically, if $G \sim N(m, v)$, then
$$
\mathbb{E}[\Phi(G)]=\Phi\left(\frac{m}{\sqrt{1+v}}\right)
$$
Indeed, if $Z \sim N(0,1)$ is independent of $G$, then
$$
\mathbb{E}[\Phi(G)]=\E[ P(Z\leq G)] = \E[ P(Z\leq G \mid G)] = \operatorname{Pr}(Z \leq G)=\operatorname{Pr}(G-Z \leq 0)=\Phi\left(\frac{m}{\sqrt{1+v}}\right).
$$

First we consider $\E[\kappa(X,A) \mid X^s,A=g]$. Note that 
$$\E[\kappa(X,A) \mid X^s,A=g] = \theta^\kappa \mu_g + \theta^s X^s,\qquad  \;\;Var[\kappa(X,A) \mid X^s,A=g] = (\theta^\kappa)^2 (\sigma_x^\kappa)^2
$$
Analogously, for $\E[\kappa(X,A) \mid A=g]:$ 
 $$\E[\kappa(X,A) \mid A=g] = \theta^\kappa \mu_g ,\qquad  \;\;Var[\kappa(X,A) \mid X^s,A=g] = (\theta^\kappa)^2 (\sigma_x^\kappa)^2
 + (\theta^s)^2 (\sigma_x^s)^2
$$
In the model, $(\sigma_x^\kappa)^2, (\sigma_x^s)^2=1$, yielding the result.

\textbf{Bounding $ V(\pi^*) - V(\pi_{DP}^{blind})$:}
For the policy bound, note that
\begin{align*}
    V(\pi^*) - V(\pi_{DP}^{blind}) &= 
    \E[ (\pi^*(X,A) - \pi_{DP}^{blind}(X^s,X^\tau) s(X,A) ]\\
&= \E[ \pi^*(X,A) s(X,A) ]- \E[\pi_{DP}^{blind}(X^s,X^\tau)  s(X^s,X^\tau) ] \\
& = \E[ \pi^*(X,A) (s(X,A) - s(X^s, X^\tau))]+ 
\E[(\pi^*(X,A) -\pi_{DP}^{blind}(X^s,X^\tau)) s(X^s,X^\tau)] \\
& \leq \E[ \pi^*(X,A) (s(X,A) - s(X^s, X^\tau))] \tag{by suboptimality of $\pi^*$ w.r.t. $(X^s,X^\tau)$} \\
& \leq \E[ \abs{ (s(X,A) - s(X^s, X^\tau))
}]  \tag{since $\pi^*\in[0,1]$}
\end{align*}
Next we can apply results specific to the Gaussian model of \Cref{eqn-simple-gaussian-model}. By Lipschitzness of the Gaussian probit/cdf, $\Phi'(x) \leq \frac{1}{\sqrt{2\pi}}.$ Therefore:
\begin{align*}
    \E[ \abs{ s(X,A) - s(X^s, X^\tau)}]
    &\leq \frac{(1-\bar{p}_0)\E[ \abs{\tau(X^s,X^\tau) (\theta^\kappa X^\kappa - \E[\theta^\kappa X^\kappa \mid X^s, X^\tau])
    }}{\sqrt{2\pi}}\\
    &= \frac{(1-\bar{p}_0)\E[ \abs{\tau(X^s,X^\tau)}]\E[\abs{ (\theta^\kappa X^\kappa - \E[\theta^\kappa X^\kappa \mid X^s, X^\tau])
    }}{\sqrt{2\pi}}\\
       &\leq \frac{(1-\bar{p}_0)
       \abs{\theta^\kappa }
       \E[ \abs{\tau(X^s,X^\tau)}]\E[\abs{ (X^\kappa - \E[X^\kappa ])
    }}{\sqrt{2\pi}}.
\end{align*}

\end{proof}

\subsection{Estimation - proofs for generalization under unconstrained policies }
\begin{proposition}[Policy value generalization]\label{prop-polgen-unconstr}
Assume the nuisance models $\eta = [p_{1\mid 0},p_{1\mid 1}, \mu_{1},\mu_0, e_r(X)]^\top, \eta \in \mathcal{F}_\eta$ are consistent and well-specified with finite VC-dimension $v_\eta$ over the product function class $H$. We provide a proof for the general case, including doubly-robust estimators, which applies to the statement of \Cref{prop-polgen-unconstr} by taking $\eta =  [p_{1\mid 0},p_{1\mid 1}, \mu_{1},\mu_0].$ 

Let $\Pi = \{\mathbb{I}\{\E[\ell(\lambda,X,A; \eta)\mid X]<0\} \colon  \lambda \in \mathbb{R}; \eta \in \mathcal{F}_\eta \}.$ 
    $$\textstyle \sup_{\pi \in \Pi, \lambda\in\mathbb{R}}\left|
(\E_n[ \pi \ell(\lambda,X,A) ] - \E[ \pi \ell(\lambda,X,A) ])  
    \right| = O_p(n^{-\frac 12} ) $$
\end{proposition}
The generalization bound allows deducing risk bounds on the out-of-sample value: 
\begin{corollary}[]
    $$\E[ \ell(\hat \lambda,X,A)_+ ]\leq \E[ \ell(\lambda^*,X,A)_+ ]  + O_p(n^{-\frac 12} )$$
\end{corollary}
\begin{proof}{Proof of \Cref{prop-polgen-unconstr}}
We study a general Lagrangian, which takes as input pseudo-outcomes $\psi^{t\mid r}(O;\eta),\psi^{y \mid t }(O;\eta), \psi^{1\mid 0,\Delta A}$ where each satisfies that
\begin{align*}
    \E[\psi^{t\mid r}(O;\eta)\mid X,A] &= p_{1\mid 1}(X,A)-p_{1\mid 0}(X,A) \\
     \E[\psi^{y \mid t }(O;\eta)\mid X,A] &= \tau(X,A) \\
     \E[ \psi^{1\mid 0,\Delta A}\mid X]& = p_{1\mid 0}(X,a)-p_{1\mid 0}(X,b)
\end{align*}

We make high-level stability assumptions on pseudooutcomes $\psi$ relative to the nuisance functions $\eta$ (these are satisfied by standard estimators that we will consider): 
\begin{assumption}
    $\psi^{t\mid r}, \psi^{y \mid t }, \psi^{1\mid 0,\Delta A}$ respectively are Lipschitz contractions with respect to $\eta$ and bounded
\end{assumption}

We study a generalized Lagrangian of an optimization problem that took these pseudooutcome estimates as inputs: 
$$\ell(\lambda,X,A; \eta ) =  \psi_{t\mid r}(O;\eta) \left\{ \psi_{y \mid t }(O;\eta) + \frac{\lambda}{p(A)} (\indic{A=a}-\indic{A=b}) \right\} + \lambda (\psi^{1\mid 0,\Delta A}(O;\eta))$$

We will show that  $$
\sup_{\pi \in \Pi, \lambda\in\mathbb{R}}\left|
(\E_n[ \pi \ell(\lambda,X,A) ] - \E[ \pi \ell(\lambda,X,A) ])  
    \right| = O_p(n^{-\frac 12} ) $$
which, by applying the generalization bound twice gives that 
$$
\E_n[ \pi \ell(\lambda,X,A) ]   = \E[ \pi \ell(\lambda,X,A) ])   + O_p(n^{-\frac 12} ) $$

Write the pointwise-score characterization as

$$
\max_\pi \min_\lambda = \min_\lambda \max_\pi = \min_\lambda\E[\ell(O, \lambda; \eta  )_+]
$$

Suppose the Rademacher complexity of $\eta_k$ is given by $\mathcal{R}(H_k),$ so that \citep[Thm. 12]{bartlett2002rademacher} gives that the Rademacher complexity of the product nuisance class $H$ is therefore $\sum_k\mathcal{R}(H_k).$ The main result follows by applying vector-valued extensions of Lipschitz contraction of Rademacher complexity given in \cite{maurer2016vector}. 
Suppose that $\psi^{t\mid r}, \psi^{y \mid t }, \psi^{1\mid 0,\Delta A}$ are Lipschitz with constants $C^L_{t\mid r}, C^L_{y \mid t }, C^L_{1\mid 0,\Delta A}$. 

We establish VC-properties of 
\begin{align*}& \mathcal{F}_{L_1}(O_{1:n}) 
= 
\left\{ 
 (g_\eta(O_1)
 ,
 g_\eta(O_i),
 \dots 
 g_\eta(O_n) )
  \colon  \eta \in \mathcal{F}_\eta 
\right\}, \text{ where } g_\eta(O) = \psi_{t\mid r}(O;\eta) \psi_{y \mid t }(O;\eta)\\
& \mathcal{F}_{L_2}(O_{1:n}) 
= 
\left\{ 
 (h_\eta(O_1)
 ,
 h_\eta(O_i),
 \dots 
 h_\eta(O_n) )
  \colon  \eta \in \mathcal{F}_\eta 
\right\}, \text{ where } h_\eta(O) = \psi_{t\mid r}(O;\eta)  \frac{\lambda}{p(A)} (\indic{A=a}-\indic{A=b}) \\
 & \mathcal{F}_{L_3}(O_{1:n})= 
\left\{ 
 (m_\eta(O_1)
 ,
 m_\eta(O_i),
 \dots 
 m_\eta(O_n) )
  \colon  \eta \in \mathcal{F}_\eta 
\right\}, \text{ where } m_\eta(O) = \lambda (\psi^{1\mid 0,\Delta A}(O;\eta)) 
\end{align*} 
and the function class for the truncated Lagrangian, 
$$
\mathcal{F}_{L_+}
 = 
\left\{ 
\{(g_\eta(O_i)+h_\eta(O_i)+m_\eta(O_i))_+\}_{1:n}  \colon 
g \in \mathcal{F}_{L_1}(O_{1:n}), h \in \mathcal{F}_{L_2}(O_{1:n}), m \in \mathcal{F}_{L_3}(O_{1:n}),
\eta\in \mathcal{F}_\eta
\right\} $$

\citep[Corollary 4]{maurer2016vector} (and discussion of product function classes) gives the following:  
Let $\mathcal{X}$ be any set, $\left(x_1, \ldots, x_n\right) \in \mathcal{X}^n$, let $F$ be a class of functions $f: \mathcal{X} \rightarrow \ell_2$ and let $h_i: \ell_2 \rightarrow \mathbb{R}$ have Lipschitz norm $L$. Then
\begin{equation}\label{eq-vectorvaluedcontraction}
\mathbb{E} \sup_{\eta  \in H} \sum_i \epsilon_i \psi_i\left(\eta\left(O_i\right)\right) \leq \sqrt{2} L \mathbb{E} \sup_{\eta  \in H} \sum_{i, k} \epsilon_{i k} \eta\left(O_i\right)
\leq  \sqrt{2} L \sum_k \mathbb{E} \sup _{\eta_k \in H_k} \sum_i \epsilon_i \eta_k \left(O_i\right)
\end{equation}
where $\epsilon_{i k}$ is an independent doubly indexed Rademacher sequence and $f_k\left(x_i\right)$ is the $k$-th component of $f\left(x_i\right)$. 

Applying \Cref{eq-vectorvaluedcontraction} to each of the component classes $\mathcal{F}_{L_1}(O_{1:n}), \mathcal{F}_{L_2}(O_{1:n}),\mathcal{F}_{L_3}(O_{1:n}),$ and Lipschitz contraction \citep[Thm. 12.4]{bartlett2002rademacher} of the positive part function $\mathcal{F}_{L_+}$, 
we obtain the bound
$$
\sup_{\lambda, \eta} \abs{\E_n [ \ell(O, \lambda; \eta  )_+ ]  - \E[ \ell(O, \lambda; \eta  )_+ ] } 
\leq  \sqrt{2} (C^L_{t\mid r}C^L_{y \mid t} + C^L_{t\mid r}B_{p_a} B
+ B C^L_{1 \mid 0, \Delta A}) \sum_k \mathcal{R}(H_k) 
$$

\end{proof}

\begin{proof}{Proof of \Cref{prop-threshold-treatment-parity}}
By \Cref{prop-identification}, expanding $\sum_{t,r}\pi_r\mu_t p_{t\mid r}$ in the binary case and using $\kappa(X,A)\tau(X,A)=s(X,A)$ (the encouragement-score factoring), $\E[Y(\pi)]=C_Y+\E[\pi(X,A)s(X,A)]$ where $C_Y$ is independent of $\pi$. Similarly, $T(\pi)=\pi(X,A)\kappa(X,A)+p_{1\mid 0}(X,A)$, so $\E[T(\pi)\mid A=u]=\E[\pi(X,A)\kappa(X,A)+p_{1\mid 0}(X,A)\mid A=u]$ and
\[\Delta_T(\pi):=\E[T(\pi)\mid A=a]-\E[T(\pi)\mid A=b]=\E[\pi(X,A)\kappa(X,A)g(A)]+c_0.\]
Hence the Lagrangian for \eqref{eq-opt-treatment-parity-constraint} is
\[L(\pi,\lambda)=C_Y+\lambda(\epsilon-c_0)+\E\!\left[\pi(X,A)\ell(\lambda,X,A)\right],\qquad\lambda\ge 0.\]
For fixed $\lambda$, the integrand is linear in $\pi(x,u)\in\{0,1\}$, so the pointwise maximizer is $\pi_\lambda(x,u)=\mathbb{I}\{\ell(\lambda,x,u)>0\}$, giving $\sup_\pi L(\pi,\lambda)=C_Y+\lambda(\epsilon-c_0)+\E[\ell(\lambda,X,A)_+]$. By strong duality for infinite-dimensional linear programs \citep{shapiro2001duality}, any $\lambda^*$ minimizing this dual objective yields an optimal $\pi^*=\pi_{\lambda^*}$.
If $\pi$ depends only on $X$, then $\E[\pi(X)\ell(\lambda,X,A)]=\E[\pi(X)\E[\ell(\lambda,X,A)\mid X]]$, so the same pointwise argument gives $\pi_\lambda(x)=\mathbb{I}\{\E[\ell(\lambda,X,A)\mid X=x]>0\}$ with dual objective $\E[\E[\ell(\lambda,X,A)\mid X]_+]+\lambda(\epsilon-c_0)$. Strict feasibility can be verified using \cref{eqn-feas-constr}.
\end{proof}

\clearpage
\subsection{Proofs for robust characterization }

\subsubsection{Additional results}\label{apx-opt-overlap}
In the specialized but practically relevant case of binary outcomes/treatments/recommendations, we obtain the following simplifications for bounds on the policy value and the minimax robust policy that optimizes the worst-case overlap extrapolation function. In the special case of constant uniform bounds, it is equivalent (in the case of binary outcomes) to consider marginalizations: 
\begin{lemma}[Binary outcomes, constant bound]\label{lemma-binaryoutcome-constantbound}
Let $\mathcal{U}_{\text{cbnd}} \coloneqq \left\{ q_{t\mid r}(x') \colon 
\underline{B}\leq q_{1\mid r}(x')  \leq \overline{B}
\right\} $ and $\mathcal{U} = \mathcal{U}_{q_{t\mid r}} 
 \cap \mathcal{U}_{\text{cbnd}}.$ Define $ \bar q_{t\mid r} \coloneqq \E[ q_{t\mid r}(X,A)\mid T=t].$%
 If $T\in\{0,1\},$
 \begin{align*}  \overline{V}_{no}(\pi) &= 
\sumrt 
 \E[
c^*_{rt}
 \bar q_{t\mid r}  \E[ Y\pi_r(X)\mid T=t]   \indnov{X} ]  ], \\ 
& \text{ where } \textstyle c^*_{rt}
 = 
 \overline{B} \indic{t=1} + \underline{B}\indic{t=0} \text{ if } \E[ Y\pi_r(X)\mid T=t]  \geq 0, \text{ and } c^*_{rt}= 0 \text{ otherwise. }
\end{align*} 
\end{lemma}

\begin{proof}{Proof of \Cref{prop-nooverlap}}

\begin{align}
 V(\pi) = &  
\sumrt \E[ \pi_r(X) \E[  Y(t) \indic{T(r)=t} \mid R = r,X ]  ]  \nonumber \\
 = &  \sumrt 
 \E[  \pi_r(X) \E[  Y(t)  \mid R=r,X] P( T(r)=t \mid R = r,X )  ] && \text{ unconf. }  %
 \nonumber \\
= &  \sumrt 
\E[  \pi_r(X) \E[ Y(t) \mid X] P( T(r)=t \mid R = r,X ) ] && \text{ \Cref{asn-exclusionrestriction} (ER) } \label{reg-adjustment-identification} \\
 = &  \sumrt 
 \E\left[ \pi_r(X)  \E\left[  Y(t)  \frac{\indic{T(r)=t}}{p_t(X)} \mid X\right] P( T(r)=t \mid R = r,X )  \right] && \text{ unconf.}
 \nonumber \\
= &  \sumrt 
\E\left[  \pi_r(X) 
\left\{ \E\left[ Y(t)  \frac{\indic{T(r)=t}}{p_t(X)} 
 + \left( 1-\frac{T}{p_t(X)}\right) \mu_{t}(X)
 \mid X\right] 
\ptmidrx
  \right\}
  \right] && \text{ control variate }\nonumber   \\
 =& \sumrt 
 \E\left[  \pi_r(X) \left\{ \left\{ Y(t)  \frac{\indic{T(r)=t}}{p_t(X)} 
 + \left( 1-\frac{T}{p_t(X)}\right) \mu_{t}(X)
 \right\}
\ptmidrx \right\} \right] && \text{(LOTE)} \nonumber 
\end{align}
where $p_t( X)=P(T=t\mid X)$ (marginally over $R$ in the observational data) and (LOTE) is an abbreviation for the law of total expectation.
\end{proof}

\begin{proof}{Proof of \Cref{lemma-binaryoutcome-constantbound}}
    \begin{align*}
        \overline{V}_{no}(\pi)&\coloneqq  \underset{q_{tr}(X) \in \mathcal{U} }{\max}
\left\{ 
\sumrt 
 \E[ \pi_r(X)\mu_t(X)  q_{tr}(X)  \indnov{X} ]  ]
 \right\} \\
 & =\underset{q_{tr}(X) \in \mathcal{U} }{\max}
\left\{ 
\sumrt 
 \E[ \pi_r(X)\E[ Y\mid T=t,X] q_{tr}(X)  \indnov{X} ]  ]
 \right\} 
    \end{align*}

    Note the objective function can be reparametrized under a surjection of $q_{t\mid r}(X)$ to its marginalization, i.e. marginal expectation over a $\{ T=t\}$ partition (equivalently $\{T=t, A=a\}$ partition for a fairness-constrained setting). 

    Define \[ \bar q_{t\mid r}(a) \coloneqq \E[ q_{t\mid r}(X,A)\mid T=t, A=a],
    \bar q_{t\mid r} \coloneqq \E[ q_{t\mid r}(X,A)\mid T=t]
    \]
    Therefore we may reparametrize $\overline{V}_{no}(\pi)$ as an optimization over constant coefficients (bounded by B): 

    \begin{align*}
&= {\max}
\left\{ 
\sumrt 
 \E[\{ c_t  \bar q_{t\mid r} \}  \pi_r(X)\E[ Y\mid T=t,X]   \indnov{X} ]  ]
 \colon \underline{B}\leq c_1 \leq \overline{B}, c_0=1-c_1 
 \right\} \\
& = {\max}
\left\{ 
\sumrt 
 \E[\{ c_t  \bar q_{t\mid r} \}  \E[ Y\pi_r(X)\mid T=t]   \indnov{X} ]  ]
 \colon \underline{B}\leq c_1 \leq \overline{B}, c_0=1-c_1 
 \right\} && \text{(LOTE)} \\
 & = 
\sumrt 
 \E[
c^*_{rt}
 \bar q_{t\mid r}  \E[ Y\pi_r(X)\mid T=t]   \indnov{X} ]  ] \\
 &\text{ where } c^*_{rt}
 = \begin{cases}
 \overline{B} \indic{t=1} + \underline{B}\indic{t=0} & \text{ if } \E[ Y\pi_r(X)\mid T=t]  \geq 0 \\ 
  \overline{B} \indic{t=0} + \underline{B}\indic{t=1} & \text{ if }\E[ Y\pi_r(X)\mid T=t]  < 0
 \end{cases} 
    \end{align*}
\end{proof}

\begin{proof}{Proof of \cref{prop-robustlp}}

\begin{remark}
\textbf{Remark 3.} Proposition~\ref{prop-robustlp} in the main text is stated in value-maximization form. For the derivation below we work with the equivalent cost objective $C(\pi):=-\underline{V}_{\mathrm{rob}}(\pi)$ and convert back by multiplying the objective by $-1$.
\end{remark}

    \begin{align}
\min_\pi\; & C(\pi) \nonumber\\
\text{s.t.}\; &
\E[T(\pi)\indov{X}\mid A=a]-\E[T(\pi)\indov{X}\mid A=b] \nonumber\\
&\qquad +\E[T(\pi)\indnov{X}\mid A=a]-\E[T(\pi)\indnov{X}\mid A=b]
\le \epsilon,\quad \forall q_{1\mid r}\in\mathcal{U}
\end{align}

Define
$$
g_r(x,u) =  (\mu_{r1}(x,u) - \mu_{r0}(x,u))
$$

then we can rewrite this further and
apply the standard epigraph transformation: 
\begin{align*}
    \min_{t,\pi} \; & t
 \\
 & t - \int_{x\in \Xnov} \sum_{u\in\{a,b\}} 
 \sum_{r\in\{ 0,1\}} 
 \{ g_r(x,u) \pi_r(x,u) f(x, u) \} q_{1\mid r}(x,u)\}  
dx \geq 
 V_{ov}(\pi) + \E[ \mu_0(X,A)\indnov{X} ] ,  \forall q_{1\mid r}\in \mathcal{U} \\
& \int_{x\in \Xnov}
\{ f(x\mid a) (\sum_r \pi_r(x,a) q_{1\mid r}(x,a) )
- f(x\mid b) (\sum_r \pi_r(x,b) q_{1\mid r}(x,b))\} +\E[\Delta_{ov}T(\pi)] \leq \epsilon ,  \forall q_{1\mid r}\in \mathcal{U}  \\
\end{align*}

Project the uncertainty set onto the direct product of uncertainty sets: 
\begin{align*}
      \min_{t,\pi} \; & t
 \\
 & t - \int_{x\in \Xnov} \sum_{u\in\{a,b\}} 
 \sum_{r\in\{ 0,1\}} 
 \{ g_r(x,u) \pi_r(x,u) f(x, u) \} q_{1\mid r}(x,u)\}  
dx \geq 
 V_{ov}(\pi) + \E[ \mu_0(X,A)\indnov{X} ] ,  \forall q_{1\mid r}\in \mathcal{U_1} \\
& \int_{x\in \Xnov}
\{ f(x\mid a) (\sum_r \pi_r(x,a) q_{1\mid r}(x,a) )
- f(x\mid b) (\sum_r \pi_r(x,b) q_{1\mid r}(x,b))\} +\E[\Delta_{ov}T(\pi)] \leq \epsilon ,  \forall q_{1\mid r}\in \mathcal{U_2}  \\
\end{align*}
Clearly robust feasibility of the resource parity constraint over the interval is obtained by the highest/lowest bounds for groups $a,b$, respectively: 
\begin{align*}
    \min_{t,\pi}  \; & t
 \\
 & t - \int_{x\in \Xnov} \sum_{u\in\{a,b\}} 
 \sum_{r\in\{ 0,1\}} 
 \{ g_r(x,u) \pi_r(x,u) f(x, u) \} q_{1\mid r}(x,u)\}  
dx \geq 
 V_{ov}(\pi)  + \E[ \mu_0(X,A)\indnov{X} ]  ,  \forall q_{1\mid r}\in \mathcal{U_1} \\
& \int_{x\in \Xnov}
\{ f(x\mid a) (\sum_r \pi_r(x,a) \overline{B}_r(x,a) )
- f(x\mid b) (\sum_r \pi_r(x,b) \underline{B}_r(x,u))\} +\E[\Delta_{ov}T(\pi)] \leq \epsilon 
\end{align*}

Define midpoint and radius $m_r(x,u):=\tfrac12(\underline{B}_r(x,u)+\overline{B}_r(x,u))$ and $\rho_r(x,u):=\tfrac12(\overline{B}_r(x,u)-\underline{B}_r(x,u))$, and reparameterize
$$q_{1\mid r}(x,u) = m_r(x,u)+\rho_r(x,u)\,\varphi_{1\mid r}(x,u),\qquad \|\varphi_{1\mid r}(x,u)\|_\infty\le 1,$$
so that $\underline{B}_r\le q_{1\mid r}\le\overline{B}_r \iff \|\varphi_{1\mid r}\|_\infty\le 1$. Substituting,

 \begin{align*}
 &   \min_{t,\pi}  \;  t
 \\
 & t +\min_{
 \substack{\norm{\varphi_{1\mid r}(x,u)}_\infty \leq 1\\ r\in\{0,1\},u\in\{a,b\}}
 }\left\{ - \int_{x\in \Xnov} \sum_{u\in\{a,b\}} 
 \sum_{r\in\{ 0,1\}} 
 \{ g_r(x,u) \pi_r(x,u) f(x, u) \} \rho_r(x,u) \varphi_{1\mid r}(x,u)  
dx  \right\} \\
& \qquad \qquad - \sum_{r}\E\!\left[\pi_r(X,A)\indnov{X}\, m_r(X,A)\tau(X,A)\right]\geq 
 V_{ov}(\pi) + \E[ \mu_0(X,A)\indnov{X} ]  \\
& \int_{x\in \Xnov}
\{ f(x\mid a) (\sum_r \pi_r(x,a) \overline{B}_r(x,a) )
- f(x\mid b) (\sum_r \pi_r(x,b) \underline{B}_r(x,u))\} +\E[\Delta_{ov}T(\pi)] \leq \epsilon,
\end{align*}

This is equivalent to: 
 \begin{align*}
    \min_{t,\pi} \; & t
 \\
 & t +  \int_{x\in \Xnov} \sum_{u\in\{a,b\}} 
 \sum_{r\in\{ 0,1\}} 
 -\abs{g_r(x,u) \pi_r(x,u) f(x, u)} \rho_r(x,u) 
dx   - \sum_{r}\E\!\left[\pi_r(X,A)\indnov{X}\, m_r(X,A)\tau(X,A)\right]\geq 
 V_{ov}(\pi) + \E[ \mu_0(X,A)\indnov{X} ] \\
& \int_{x\in \Xnov}
\{ f(x\mid a) (\sum_r \pi_r(x,a) \overline{B}_r(x,a) )
- f(x\mid b) (\sum_r \pi_r(x,b) \underline{B}_r(x,u))\} +\E[\Delta_{ov}T(\pi)]  \leq \epsilon 
\end{align*}
Multiplying the objective by $-1$ converts this equivalent cost-minimization program back into the value-maximization statement of Proposition~\ref{prop-robustlp} in the main text.
Equivalently, and matching the main-text statement, we obtain:
 \begin{align*}
    \max_\pi \;&
    V_{\mathrm{ov}}(\pi)
    +\E[\mu_0(X,A)\indnov{X}]
    +c_1(\pi)
    -\sum_{r\in\{0,1\}}
    \E\!\left[
    |\tau(X,A)|\,\pi_r(X,A)\,\rho_r(X,A)\,\indnov{X}
    \right]
\\
\text{s.t.}\;&
\sum_{r\in\{0,1\}}
\Big\{
\E[\pi_r(X,A)\overline B_r(X,A)\indnov{X}\mid A=a]
-
\E[\pi_r(X,A)\underline B_r(X,A)\indnov{X}\mid A=b]
\Big\}
+\Delta_{\mathrm{ov}}^T(\pi)
\le \epsilon.
\end{align*}
\end{proof}

\clearpage
\subsection{Proofs for general fairness-constrained policy optimization algorithm and analysis}

We begin by summarizing some notation that will simplify some statements. 
Define, for observation tuples $O\sim (X,A,R,T,Y),$ the value estimate $v(Q;\eta)$ given some pseudo-outcome $\psi(O;\eta)$ dependent on observation information and nuisance functions $\eta$. (We often suppress notation of $\eta$ for brevity). We let estimators sub/super-scripted by $1$ denote estimators from the first dataset. 

\begin{align*}
v_{(\cdot)}(O;Q,\eta) &= \E_{\pi \sim Q} [ v_{(\cdot)}(O;\pi,\eta)]\\
    v_{(\cdot)}(Q) &= \E[v_{(\cdot)}(Q)] \\
\hat V_1^{(\cdot)}(Q) &= \E_{n_1}[v_{(\cdot)}(Q)] \\
\delta_k(O; Q) & = \E_{\pi\sim Q}[\delta_k(O ; \pi) \mid O] \\
  \Delta_k(Q) &= \E[\delta_k(O; Q)] \\
    \hat{\Delta}^1_k(Q) &= \E_{n_1}[\delta_k(O; Q)]
\end{align*}

\subsubsection{Preliminaries: results from other works used without proof}
The original reductions theorem is stated for cost minimization. We apply it to
the equivalent cost objective \(\tilde V=-V\) and state the resulting
sign-flipped corollary below in our value-maximization notation.

\begin{theorem}[Theorem 3, \citep{agarwal2018reductions} (saddle point generalization bound for \Cref{alg-metaalg1}) ]\label{thm-saddle-generalization}
Let $\xi:=\sup_{Q\in\mathcal{Q}}\|\hat\Delta(Q)-\Delta(Q)\|_{\infty}$. Suppose \Cref{asn-polopt-rate} holds for $C^{\prime} \geq 2 C+2+\sqrt{\ln (4 / \delta) / 2}$, where $\delta>0$. Let $Q^{\star}$ maximize $V(Q)$ subject to $\Delta(Q) \leq d$. Then Algorithm 1 with $\nu \propto n^{-\alpha}, B \propto n^\alpha$ and $\omega \propto \xi^{-2} n^{-2 \alpha}$ terminates in $O\left(\xi^2 n^{4 \alpha} \ln |\mathcal{K}|\right)$ iterations and returns $\hat{Q}$. If $n p_j^{\star} \geq 8 \log (2 / \delta)$ for all $j$, then with probability at least $1-(|\mathcal{J}|+1) \delta$ then for all $k$, $\hat{Q}$ satisfies: 
$$
\begin{aligned}
& V(\hat{Q}) \geq V\left(Q^{\star}\right)-\widetilde{O}\left(n^{-\alpha}\right) \\
&\Delta_k(\widehat{Q}) \leq d_k+\frac{1+2 \nu}{B}+\sum_{j \in \mathcal{J}}\left|M_{k, j}\right| \widetilde{O}\left(\left(n p_j^{\star}\right)^{-\alpha}\right)
\end{aligned}
$$    
\end{theorem}

The proof of \cite[Thm. 3]{agarwal2018reductions} is  modular in invoking Rademacher complexity bounds on the objective function and constraint moments, so that invoking standard Rademacher complexity bounds for off-policy evaluation/learning \citep{athey2021policy,sj15} yields the above statement for $V(\pi)$ (and analogously, randomized policies by \citep[Thm. 12.2]{bartlett2002rademacher} giving stability for convex hulls of policy classes). 

More specifically, we use standard local Rademacher complexity bounds.
\begin{definition}[Local Rademacher Complexity]\label{def-local-rademacher-complexity}
    The local Rademacher complexity for a generic $f \in \mathcal{F}$ is: $$\mathcal{R}(r, \mathcal{F})=\mathbb{E}_{\epsilon, X_{1: n}}\left[\sup _{f \in \mathcal{F}:\|f\|_2 \leq r} \frac{1}{n} \sum_{i=1}^n \epsilon_i f\left(X_i\right)\right]$$
\end{definition}
    
The following is a generic concentration inequality for local Rademacher complexity over some radius $r$; see \cite{wainwright2019high} for more background. 
\begin{lemma}[Lemma 5 of \citep{chernozhukov2019semi}/Lemma 4, \cite{foster2019orthogonal}]\label{lemma-localrc}
 Consider any $Q^* \in \mathcal{Q}$. 
 Assume that $v(\pi)$ is L-Lipschitz in its first argument with respect to the $\ell_2$ norm and let:
$$
Z_n(r)=\sup _{Q \in \mathcal{Q}}
\{ \left|
\E_n[ \hat v(Q) - \hat v(Q^*)]
-
\E[v(Q)- v(Q^*)] %
\right|
\colon 
\E[(v(Q)-v(Q^*))^2]^{\frac 12 } 
\leq r
\} 
$$
Then for some constant $C_3$: 
$$Z_n(r) \leq C_3\left(\mathcal{R}\left(r, \mathcal{Q}-Q^*\right)+r \sqrt{\frac{\log (1 / \delta)}{n}}+\frac{\log (1 / \delta)}{n}\right)$$
\end{lemma}

\begin{lemma}[Concentration of empirical constraint moments (\citep{agarwal2018reductions,woodworth2017learning})]\label{lemma-concentration-conditional-moments}
    For any $k \in [K]$, with probability at least $1-\delta$, for all $Q$,
$$
\left|\widehat{\Delta}_k(Q)-\Delta_k(Q)\right| \leq 2 \mathcal{R}_{n_j}(\mathcal{Q})+\frac{2}{\sqrt{n_j}}+\sqrt{\frac{\ln (2 / \delta)}{2 n_j}}
$$
If $n p_j^{\star} \geq 8 \log (2 / \delta)$, then with probability at least $1-\delta$, for all $Q$,
$$
\left|\widehat{\Delta}_k(Q)-\Delta_k(Q)\right| \leq 2 \mathcal{R}_{n p_j^{\star} / 2}(\mathcal{Q})+2 \sqrt{\frac{2}{n p_j^{\star}}}+\sqrt{\frac{\ln (4 / \delta)}{n p_j^{\star}}}
$$
\end{lemma}
\begin{lemma}[Orthogonality (analogous to \citep{chernozhukov2019semi}  (Lemma 8), others)]\label{lemma-chernozhukov-ortho}
    Suppose the nuisance estimates satisfy a mean-squared-error bound 
    $$
    \max_{l} \{ \E[(\hat\eta_l-\eta_l)^2]\}_{l\in [L]} \le \chi^2_{n,\delta}
    $$
    Then w.p. $1-\delta$ over the randomness of the policy sample, 
    $$
    V(Q_0) - V(\hat Q) \leq O(R_{n,\delta} + \chi^2_{n,\delta})
    $$
\end{lemma}

\subsection{Adapted lemmas}
In this subsection we collect results similar to those that have appeared previously, but that require substantial additional argumentation in our specific saddle point setting. 

The next lemma establishes the variance of small-regret policies with estimated vs. true nuisances is close, up to nuisance estimation error. 
\begin{lemma}[Feasible vs. oracle nuisances in low-variance regret slices (\cite{chernozhukov2019semi}, Lemma 9) ]\label{lemma-variance-slice-oracle}
Suppose that the mean squared error of the nuisance estimates is upper bounded w.p. $1-\delta$ by $\chi^2_{n, \delta}$ and suppose $\chi^2_{n, \delta} \leq \epsilon_n$. Then:
$$
V_2^0=\sup _{Q, Q^{\prime} \in \mathcal{Q}^\star\left(\epsilon_n+2\chi^2_{n, \delta};\hat I_1\right)} \operatorname{Var}\left(v_{D R}^0(x ; Q)-v_{D R}^0\left(x ; Q^{\prime}\right)\right)
$$
Then $V_2 \leq V_2^0+O\left(\chi_{n, \delta}\right)$.
\end{lemma}

\subsection{Proof of \Cref{thm-varregret}}

\begin{proof}{Proof of \Cref{thm-varregret}}

We first study the meta-algorithm with ``oracle" nuisance functions $\eta=\eta_0$. For brevity below we notationally suppress the dependence of $v$ on observation $O$.

Define the empirical second-stage feasible set:
\[
\mathcal{Q}_2(\epsilon_n)
:=
\left\{
Q\in\mathcal P(\Pi):
\E_{n_1}\!\left[
v_{\mathrm{DR}}(O;\hat Q_1,\eta_0)
-
v_{\mathrm{DR}}(O;Q,\eta_0)
\right]\le \epsilon_n,\quad
\E_{n_1}\!\left[
\delta_k(O;Q,\eta_0)
-
\delta_k(O;\hat Q_1,\eta_0)
\right]\le \epsilon_n,\ \forall k\in\hat I_1,\quad
\E_{n_1}[\delta_k(O;Q,\eta_0)]\le \hat d_k,\ \forall k\in[K]
\right\}.
\]
By construction, the second-stage optimization localizes only the nearly binding
constraints indexed by \(\hat I_1\), while retaining the slackened feasibility
constraints for all \(k\in[K]\).

In the following, we suppress notational dependence on $\eta_0.$

Note that $\hat Q_1 \in \mathcal{Q}_2\left(\epsilon_n\right).$

\underline{Step 1:} First we argue that w.p. $1-\delta/6,$ $Q^* \in \mathcal{Q}_2$. 

Invoking \Cref{thm-saddle-generalization} on the output of the first stage of the algorithm, yields that with probability $1-\frac{\delta}{6}$ over the randomness in $\mathcal{D}_1,$ by choice of $\epsilon_n$ = $\bar{O}(n^{-\alpha}))$, 
\begin{align*} 
V(\hat{Q}_1) &\geq V(Q^*)- \epsilon_n/2 \\
\Delta_k(\hat{Q}_1) &\leq d_k+\sum_{j \in \mathcal{J}}\left|M_{k, j}\right| \widetilde{O}\left((np_j^*)^{-\alpha}\right) \leq d_k + \epsilon_n/2  \quad \text { for all } k 
\end{align*}

Further, by \Cref{lemma-localrc}, 
\begin{align*}
&\sup_{Q \in \mathcal{Q}} \abs{
\E_{n_1}[(v(Q)-v(Q^*))]
- \E[(v(Q)-v(Q^*))]
} \leq \epsilon_n / 2\\
&\sup_{Q \in \mathcal{Q}} \abs{
\E_{n_1}[(g(O;Q)-g(O;Q^*))]
- \E[(\delta_k(O;Q)-\delta_k(O;Q^*))]
} \leq \epsilon_n / 2
\end{align*} 

Therefore, with high probability on the good event, $Q^\star \in \mathcal{Q}_2(\epsilon_n)$.

\underline{Step 2:} 
\underline{Step 2:} Again invoking \Cref{thm-saddle-generalization} on the output of the second stage of the algorithm, restricted to the localized feasible class $\mathcal{Q}_2(\epsilon_n)$, and conditioning on the good event that $Q^\star \in \mathcal{Q}_2(\epsilon_n)$, we obtain that with probability at least $1-\delta/3$ over the randomness of the second sample $\mathcal{D}_2$,

\begin{align*}
V(\hat{Q}_2) &\geq V(Q^*)-\epsilon_n/2 \\
\Delta_k(\hat{Q}_2) & \leq \Delta_k({Q}^*)+\epsilon_n/2 
\end{align*}

\underline{Step 3:} empirical small-regret slices relate to population small-regret slices, and variance bounds

We show that if $Q \in \mathcal{Q}_2,$ then with high probability $Q \in \mathcal{Q}_2^0$ (defined on small population value- and constraint-regret slices relative to $\hat Q_1$ rather than small empirical regret slices)  
$$
\mathcal{Q}_2^0 = \{ Q \in\op{conv}(\Pi) \colon \abs{V(Q)-V(\hat Q_1)}\leq \epsilon_n/2, 
\E[\delta_k(O;Q)-\delta_k(O; \hat Q_1)] \leq \epsilon_n, \forall k\} 
$$
so that w.h.p. $\mathcal{Q}_2 \subseteq \mathcal{Q}_2^0.$ 

Note that for $Q \in \mathcal{Q}$, w.h.p. $1-\delta/6$ over the first sample, we have that 
\begin{align*}
&\sup_{Q \in \mathcal{Q}} \abs{ \E_n[v(Q)-v(\hat Q_1)] - 
 \E[v(Q)- v( \hat Q_1)]} 
 \leq 2 \sup_{Q \in \mathcal{Q}} \abs{ \E_n[v(Q)]-\E[ v( Q)]} \leq \epsilon, 
\\
&\sup_{Q \in \mathcal{Q}} \abs{ \E_{n_1}[\delta_k(O;Q)-\delta_k(O;\hat Q_1)] - 
\E[\delta_k(O;Q)-\delta_k(O;\hat Q_1)]
 } 
 \\
 & \qquad \qquad  \leq 2 \sup_{Q \in \mathcal{Q}} \abs{ \E_{n_1}[\delta_k(O;Q)]-\E[\delta_k(O;Q)]} \leq \epsilon , \forall j
\end{align*} 

The second bound follows from \cite[Theorem 12.2]{bartlett2002rademacher} (equivalence of Rademacher complexity over convex hull of the policy class) and linearity of the policy value and constraint estimators in $\pi,$ and hence $Q$.

On the other hand since $Q_1$ achieves low policy regret, the triangle inequality implies that we can contain the true policy by increasing the error radius. That is, for all $Q \in \mathcal{Q}_2,$ with high probability $\geq 1 - \delta/3$:
\begin{align*}
& \abs{ 
\E[ (v(Q)-v(Q^*))]
}
\leq 
\abs{ 
\E[(v(Q)-v(\hat Q_1))]
}
+\abs{ 
\E[ (v(\hat Q_1)-v(Q^*))]
}
\leq \epsilon_n \\
&
\abs{\E[\delta_k(O;Q)-\delta_k(O; Q^*)]}
\leq 
\abs{\E[\delta_k(O;Q)-\delta_k(O;\hat Q_1)]}
+ 
\abs{\E[\delta_k(O;\hat Q_1)-\delta_k(O;Q^*)]} \leq \epsilon_n 
\end{align*}

Define the population localized near-optimal slice:
\[
\mathcal Q^\star(\epsilon_n;\hat I_1)
:=
\left\{
Q\in\mathcal P(\Pi):
V(Q^\star)-V(Q)\le \epsilon_n,\quad
\Delta_k(Q)\le d_k+\epsilon_n,\ \forall k\in[K],\quad
\Delta_k(Q)-\Delta_k(Q^\star)\le \epsilon_n,\ \forall k\in\hat I_1
\right\}.
\]
On the same high-probability event, the population counterpart of the empirical
second-stage feasible set satisfies
\[
\mathcal Q_2(\epsilon_n)
\subseteq
\mathcal Q^\star(C\epsilon_n;\hat I_1)
\]
for a universal constant $C<\infty$. The value part follows from
\[
V(Q^\star)-V(Q)
=
\{V(Q^\star)-V(\hat Q_1)\}
+
\{V(\hat Q_1)-V(Q)\}
\le
C\epsilon_n,
\]
while the full feasibility constraints follow from the slackened empirical
constraints and uniform concentration:
\[
\Delta_k(Q)\le d_k+C\epsilon_n,\qquad \forall k\in[K].
\]
For the localized comparison constraints, the definition of \(\hat I_1\), the
first-stage near-binding property, and concentration imply
\[
\Delta_k(Q)-\Delta_k(Q^\star)\le C\epsilon_n,
\qquad \forall k\in\hat I_1.
\]
So that on that high-probability event,
\begin{equation}
    \mathcal{Q}_2(\epsilon_n)\subseteq \mathcal Q^\star(C\epsilon_n;\hat I_1).
\end{equation}
Then on that event with probability $\geq 1 - \delta/3$,
\[
r_2^2
=
\sup_{Q\in\mathcal{Q}_2}
\E[(v(Q)-v(Q^\star))^2]
\le
\sup_{Q\in\mathcal Q^\star(C\epsilon_n;\hat I_1)}
\E[(v(Q)-v(Q^\star))^2]
\]
\[
=
\sup_{Q\in\mathcal Q^\star(C\epsilon_n;\hat I_1)}
\op{Var}(v(Q)-v(Q^\star))
+
\E[(v(Q)-v(Q^\star))]^2
\]
\[
\le
\sup_{Q\in\mathcal Q^\star(C\epsilon_n;\hat I_1)}
\op{Var}(v(Q)-v(Q^\star))
+
\epsilon_n^2.
\]

Therefore:

$$r_2\leq \sqrt{\sup_{Q \in \mathcal Q^\star(C\epsilon_n;\hat I_1)} \operatorname{Var}\left(v(Q)-v(Q^\star)\right)}+2 \epsilon_n=\sqrt{V_2}+2 \epsilon_n$$

Therefore, applying the local Rademacher bound on the localized class
$\mathcal Q^\star(C\epsilon_n;\hat I_1)$ gives
\[
V(Q^\star)-V(\hat Q_2)
=
O\!\left(
\kappa(\bar\sigma_{D_2},\operatorname{conv}(\mathcal F_\Pi))
+
\bar\sigma_{D_2}n_2^{-1/2}\sqrt{\log(3/\delta)}
+
\chi^2_{n,\delta}
\right).
\]
Applying the same argument coordinatewise to the constraint scores gives, for
each $k\in[K]$,
\[
\bigl(\Delta_k(\hat Q_2)-d_k\bigr)_+
=
O\!\left(
\kappa(\bar\sigma_{k,D_2},\operatorname{conv}(\mathcal F^\Delta_k))
+
\bar\sigma_{k,D_2}n_2^{-1/2}\sqrt{\log(3/\delta)}
+
\chi^2_{n,\delta}
\right).
\]

\end{proof}

\subsection{Proofs of auxiliary/adapted lemmas}

\begin{proof}{Proof of \Cref{lemma-variance-slice-oracle}}
    The proof is analogous to that of \citep[Lemma 9]{chernozhukov2019semi} except for the step of establishing that $\pi_* \in \mathcal{Q}^0_{\epsilon_n + O(\chi_{n,\delta}^2  )}$: in our case we must establish relationships between saddlepoints under estimated vs. true nuisances. We show an analogous version below. 

For this proof only, switch to the equivalent cost formulation
\[
\tilde v_{\mathrm{DR}}(O;Q,\eta)
:=
-v_{\mathrm{DR}}(O;Q,\eta),
\qquad
\tilde V(Q):=-V(Q).
\]
Thus the minimization saddle-point problem below is the costized version of the
main text's value-maximization problem.

Define the saddle points to the following problems (with estimated vs. true nuisances):
\begin{align*}
(Q^*_{0,0}, \lambda^*_{0,0}) &\in \arg\min_Q \max_\lambda \E[\tilde v_{\mathrm{DR}}(Q; \eta_0)] + \lambda^\top (\Delta(Q; \eta_0 ) - d) =: \mathcal{L}(Q,\lambda; \eta_0, \eta_0) =: \mathcal{L}(Q,\lambda) ,\\
(Q^*_{\eta,0}, \lambda^*_{\eta,0}) &\in \arg\min_Q \max_\lambda \E[\tilde v_{\mathrm{DR}}(Q; \eta)] + \lambda^\top (\Delta(Q; \eta_0 ) - d),\\
(Q^*, \lambda^*) &\in \arg\min_Q \max_\lambda \E[\tilde v_{\mathrm{DR}}(Q; \eta)] + \lambda^\top (\Delta(Q; \eta ) - d).
\end{align*}

We have that:
\begin{align*}
    &\E[\tilde v_{\mathrm{DR}}(Q^*;\eta)]  \leq \mathcal{L}(Q^*, \lambda^*; \eta,\eta) + \nu \\
    &\leq \mathcal{L}(Q^*, \lambda^*; \eta,\eta_0 ) + \nu + \chi_{n,\delta}^2 \\
    & \leq \mathcal{L}(Q^*, \lambda^*; \eta,\eta_0 ) + \nu + \chi_{n,\delta}^2 && \text{ by \Cref{lemma-chernozhukov-ortho}} \\
    & \leq \mathcal{L}(Q^*, \lambda^*_{\eta,0}; \eta,\eta_0 ) + \nu + \chi_{n,\delta}^2 && \text{ by saddlepoint prop.} \\
    & \leq \mathcal{L}(Q^*_{\eta, 0}, \lambda^*_{\eta,0}; \eta,\eta_0 ) + \abs{ \mathcal{L}(Q^*_{\eta, 0}, \lambda^*_{\eta,0}; \eta,\eta_0 )- \mathcal{L}(Q^*, \lambda^*_{\eta,0}; \eta,\eta_0 ) } + \nu + \chi_{n,\delta}^2 && \text{ triangle ineq.} \\
     &  \leq \mathcal{L}(Q^*_{\eta, 0}, \lambda^*_{\eta,0}; \eta,\eta_0 ) +  \epsilon_n  + \nu + \chi_{n,\delta}^2  && \text{assuming } \epsilon_n \geq \chi_{n,\delta}^2\\
     &  \leq \E[\tilde v_{\mathrm{DR}}(Q^*_{\eta, 0}; \eta)] +  \epsilon_n  + 2\nu + \chi_{n,\delta}^2  && \text{apx. complementary slackness }\\
     &  \leq \E[\tilde v_{\mathrm{DR}}(Q^*_{0, 0}; \eta)] +  \epsilon_n  + 2\nu + \chi_{n,\delta}^2  && \text{suboptimality} 
\end{align*}
Hence
$$
\tilde V(Q^*) - \tilde V(Q^*_{0,0})
\leq   \epsilon_n  + 2\nu + \chi_{n,\delta}^2.
$$
Equivalently, since $\tilde V = -V$,
$$
V(Q^*_{0,0}) - V(Q^*)
\leq   \epsilon_n  + 2\nu + \chi_{n,\delta}^2.
$$
We generally assume that the saddlepoint suboptimality $\nu$ is of lower order than $\epsilon_n$ (since it is under our computational control). 

Applying \Cref{lemma-chernozhukov-ortho} gives:
$$
V(Q^*_{0, 0}) - V(Q^*)
\leq   \epsilon_n  + 2\nu + 2\chi_{n,\delta}^2.
$$

Define policy classes with respect to localized small-population regret slices
(with a nuisance-estimation enlarged radius):
\[
\mathcal{Q}^0(\epsilon;\hat I_1)
=
\left\{
Q\in\mathcal{P}(\Pi):
V(Q^{\star}_{0,0})-V(Q)\le \epsilon,\quad
\Delta_k(Q)\le d_k+\epsilon,\ \forall k\in[K],\quad
\Delta_k(Q)-\Delta_k(Q^{\star}_{0,0})\le \epsilon,\ \forall k\in\hat I_1
\right\}.
\]

Then we have that 

$$V_2^{obj} \leq \sup_{Q \in \mathcal{Q}^0(\epsilon_n;\hat I_1) } \op{Var}(v_{DR}(O;\pi) - v_{DR}(O;\pi^*)),$$
where we have shown that $\pi^* \in\mathcal{Q}^0(\epsilon + 2 \nu + 2 \chi_{n,\delta}^2;\hat I_1).$

Following the rest of the argumentation in \cite[Lemma 9]{chernozhukov2019semi} from here onwards gives the result, i.e. studying the case of estimated nuisances with our \Cref{lemma-variance-slice-oracle} and \Cref{lemma-chernozhukov-ortho}. 
\end{proof}

\clearpage
\section{Additional case studies and additional details on experiments}\label{sec-addl-experiments}

\subsection{SFHSA - SNAP case study}

\subsubsection{Additional descriptive statistics}
\begin{table}[htbp]
\centering
\caption{Regression Results, $E[R\mid T_0, X]$}
\label{tbl-regressionRX}
\begin{tabular}{lcccc}
\toprule
& Coefficient & Std. Error & t-statistic & p-value \\
\midrule
Intercept & 0.000 & & & \\
HH Size & 0.598*** & (0.055) & 10.826 & 0.000 \\
Phone interview & 0.730*** & (0.049) & 14.952 & 0.000 \\
Female & 0.252*** & (0.044) & 5.667 & 0.000 \\
Nonwhite & 0.236*** & (0.047) & 5.014 & 0.000 \\
Age & 0.032*** & (0.002) & 17.943 & 0.000 \\
Citizen Status & -0.290 & (1.001) & -0.290 & 0.772 \\
First SNAP year & -0.001 & (0.024) & -0.022 & 0.983 \\
Any kids & 0.429*** & (0.096) & 4.470 & 0.000 \\
ESL & 0.384 & (1.001) & 0.384 & 0.701 \\
HH receives max amt & -0.185*** & (0.050) & -3.718 & 0.000 \\
No earnings prev quarter & 0.001** & (0.000) & 2.118 & 0.034 \\
Years since first SNAP & 0.012 & (0.022) & 0.563 & 0.573 \\
Interview week in month & -0.290 & (1.001) & -0.290 & 0.772 \\
Interview Day & -0.020*** & (0.003) & -7.244 & 0.000 \\
English Lang. Int. & -0.384 & (1.001) & -0.383 & 0.701 \\
Spanish Lang. Int. & 0.151 & (0.119) & 1.262 & 0.207 \\
reminder & 0.034 & (0.047) & 0.720 & 0.472 \\
\bottomrule
\multicolumn{5}{l}{\textit{Note:} *** p<0.01, ** p<0.05, * p<0.1}
\end{tabular}
\end{table}

\begin{table}[htbp]
\centering
\footnotesize
\renewcommand{\arraystretch}{0.85}
\caption{Regression Results, Interacted logistic regression $P(T\mid X, R, R \times X)$}
\label{tbl-regressionTRX}
\begin{tabular}{lcccc}
\toprule
& Coefficient & Std. Error & t-statistic & p-value \\
\midrule
Intercept & 0.000 & & & \\
HH Size & 0.128*** & (0.037) & 3.516 & 0.000 \\
Phone interview & 0.141** & (0.059) & 2.380 & 0.017 \\
Female & 0.051 & (0.048) & 1.063 & 0.288 \\
Nonwhite & 0.048 & (0.057) & 0.845 & 0.398 \\
Age & 0.045*** & (0.002) & 22.244 & 0.000 \\
Citizen Status & -0.059 & (1.001) & -0.059 & 0.953 \\
First SNAP year & -0.000 & (0.031) & -0.016 & 0.987 \\
Any kids & 0.089 & (0.082) & 1.086 & 0.277 \\
ESL & 0.078 & (1.001) & 0.078 & 0.938 \\
HH receives max amt & -0.039 & (0.052) & -0.756 & 0.450 \\
No earnings prev quarter & -0.005*** & (0.001) & -8.802 & 0.000 \\
Years since first SNAP & 0.088*** & (0.029) & 3.014 & 0.003 \\
Interview week in month & -0.059 & (1.001) & -0.059 & 0.953 \\
Interview Day & -0.009*** & (0.003) & -3.137 & 0.002 \\
English Lang. Int. & -0.078 & (1.001) & -0.078 & 0.938 \\
Spanish Lang. Int. & 0.031 & (0.089) & 0.348 & 0.728 \\
reminder & 0.000 & (1.414) & 0.000 & 1.000 \\
R $\times$ HH Size & 0.033 & (0.079) & 0.411 & 0.681 \\
R $\times$ Phone interview & 0.031 & (0.118) & 0.261 & 0.794 \\
R $\times$ Female & 0.012 & (0.093) & 0.125 & 0.901 \\
R $\times$ Nonwhite & 0.010 & (0.109) & 0.090 & 0.928 \\
R $\times$ Age & -0.067*** & (0.004) & -16.384 & 0.000 \\
R $\times$ Citizen Status & -0.015 & (1.003) & -0.015 & 0.988 \\
R $\times$ First SNAP year & 0.001 & (0.001) & 1.058 & 0.290 \\
R $\times$ Any kids & 0.022 & (0.186) & 0.117 & 0.907 \\
R $\times$ ESL & 0.017 & (1.006) & 0.017 & 0.986 \\
R $\times$ HH receives max amt & -0.007 & (0.106) & -0.068 & 0.946 \\
R $\times$ No earnings prev quarter & 0.009*** & (0.001) & 8.733 & 0.000 \\
R $\times$ Years since first SNAP & 0.045*** & (0.014) & 3.237 & 0.001 \\
R $\times$ Interview week in month & -0.015 & (1.003) & -0.015 & 0.988 \\
R $\times$ Interview Day & 0.003 & (0.006) & 0.516 & 0.606 \\
R $\times$ English Lang. Int. & -0.017 & (1.006) & -0.017 & 0.986 \\
R $\times$ Spanish Lang. Int. & 0.014 & (0.231) & 0.060 & 0.952 \\
\bottomrule
\multicolumn{5}{l}{\textit{Note:} *** p<0.01, ** p<0.05, * p<0.1}
\end{tabular}
\end{table}

\begin{table}[]
    \centering
\begin{tabular}{lllll}
\hline
      & $p_{1\mid 1}(x,a)-p_{1\mid 0}(x,a)$ & $\tau(x)$ & $P(T=1\mid X,A)$ & $E[Y\mid X,A]$ \\
\hline
 Correlation     & -0.08                           & -0.28  & -0.27          & 0.04         \\
 Correlation A=0 & 0.23                            & -0.27  & -0.21          & 0.21         \\
 Correlation A=1 & -0.15                           & -0.28  & -0.27          & 0.04         \\
\hline
\end{tabular}
    \caption{Self-selection and targeting efficiency: Spearman's rank correlation of predicted enrollment in reminder probabilities ($P(R=1\mid X,A)$) with heterogeneous and marginal compliance and treatment effects. }
    \label{tab:selfselection-correlations}
\end{table}

\paragraph{Is self-selection sufficient for targeting?}
An important question is whether individuals self-select into the reminder on the basis of their expected value of benefits or the particular effect of the reminder on compliance. We leverage the estimated effect heterogeneity and predictive models to unpack this relationship. We consider the Spearman correlation between the predicted probability of enrollment in the reminder, $P(R=1\mid X,A)$, and $p_{1\mid 1}(x,a)-p_{1\mid 0}(x,a)$ (the heterogeneous effect of compliance); $\tau(X)$ (the heterogeneous effect of treatment), $P(T=1\mid X,A)$ (the marginal treatment-enrollment probabilities); and $E[Y\mid X,A]$ (the marginal expected benefits). The first two reflect the effect heterogeneity: A strong positive rank correlation of selection into encouragement (the reminder) with the compliance or heterogeneous treatment effects (or the product thereof) would indicate that self-selection is efficient for recommendation/treatment efficacy. Since this may be difficult to justify a priori, as enrolling individuals do not know this estimated heterogeneity in the effects, we also assess rank correlations of $P(R=1\mid X,A)$ with the marginal treatment and outcomes. We find that, overall, self-selection is weakly \textit{negatively} correlated with heterogeneous compliance, treatment effects, and marginal compliance. However, these relationships are heterogeneous by group membership: We find a weakly positive correlation for the white subgroup, where self-selection is weakly correlated with heterogeneous compliance effects and marginal outcome levels. However, generally no such relationship exists for nonwhite beneficiaries. This indicates that current self-selection patterns are overall not efficient for the targeting of reminders to impact outcomes and that, in fact, self-selection is usually negatively correlated with the heterogeneity in the treatment effect: Better-resourced households with lower treatment effects (eligible for lower benefits amounts) are somewhat more likely to enroll in recertification reminders. At a high level, this suggests that there is room for improvement upon self-selection into deadline reminders, whether by changing the \textit{default option} to a reminder \textit{opt-out} or by means of covariate-conditional targeting under a budget constraint, as we explore next. 

\subsubsection{Additional detail on implementations in \Cref{sec-fairness-analysis}}

\paragraph{Additional details on \Cref{fig:sfhsa-density-race,fig:sfhsa-cdfplot-race}}
We split the SFHSA sample into training and test sets with a fixed random seed and use
  the training sample for all nuisance-model fitting and CATE estimation. For the main
  heterogeneous-effect analysis, we estimate $\mathbb{E}[Y\mid X=x]$ with a cross-
  validated regression learner and $\mathbb{P}(T=1\mid X=x)$ with a cross-validated
  classification learner, where the candidate models are selected by $5$-fold CV from
  logistic regression, random forests, and gradient boosting; these fitted nuisances are
  then passed into a \texttt{CausalForestDML} estimator to recover the treatment effect $
  \tau(x)=\mathbb{E}[Y(1)-Y(0)\mid X=x]$. We analogously estimate a compliance effect by
  treating $R$ as the binary treatment and $T$ as the outcome, yielding $
  \kappa(x)=\mathbb{P}(T=1\mid R=1,X=x)-\mathbb{P}(T=1\mid R=0,X=x)$, and rank
  individuals using the score $\kappa(x)\tau(x)$. For policy evaluation, the notebook
  also fits separate propensity and outcome models, using logistic regression for
  treatment propensities and boosted outcome regressions on $\log(Y+1)$, and combines
  them in a doubly robust plug-in estimator to trace expected outcomes under budget-
  constrained targeting rules.
  
\paragraph{Additional details on \Cref{fig:sfhsa-budgeted-allocation-comparison}}
We optimize the policies on a 70\% training split. The off-policy values are obtained via doubly robust estimation, with logistic regression as the propensity score and gradient-boosted regression on log-transformed outcomes for outcome regression. However, because of the zero-inflated/heavy-tailed outcomes, we find some instability in the off-policy estimates computed on the test set alone. Thus, we pool the data and obtain the plotted estimates of $E[Y(\pi)\mid A=1]$ from the entire dataset. %

\subsubsection{Group-fair budget allocations vs. global budget}
  With these findings in mind, we turn to a final investigation: Suppose that we could consider allocations with different budgets per group rather than a global budget. Given that different compliance responses lead to unequal improvements, how much more outreach toward nonwhite populations would be needed under the allocation to equalize the amount of \textit{improvement} over the no-allocation baseline? We analyze this question in \Cref{fig:sfhsa-groupwise-equalimprovement}.

  \begin{minipage}{\textwidth}
  \begin{minipage}[b]{0.5\textwidth}
  \begin{figure}[H]
 \includegraphics[width=\linewidth]{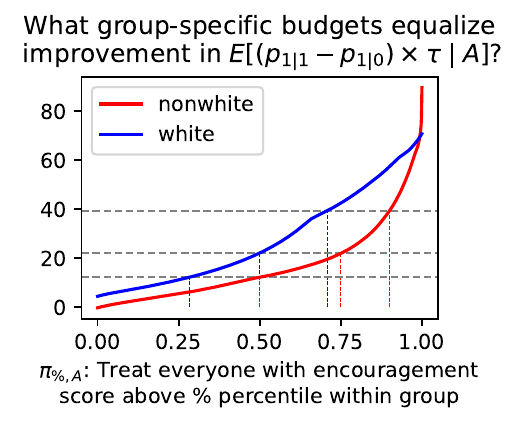}
    \caption{Groupwise separate budgets: Which budget allocations would achieve equal improvements? } 
    \label{fig:sfhsa-groupwise-equalimprovement}
  \end{figure}
  
  \end{minipage}
  \hfill
  \begin{minipage}[b]{0.5\textwidth}
    \begin{table}[H]
    \vspace{-10pt}
          \caption{Equivalent budgets for $A=1$ to equalize improvement at a budget for $A=0$}

\begin{tabular}{lll}
\toprule
 Budget A=0   & Improvement A=0   & Equiv. Budget A=1   \\
\midrule
 50.0\%        & \$12.22            & 71.7\%               \\
 25.0\%        & \$21.97            & 50.2\%               \\
 10.0\%        & \$39.37            & 29.2\%               \\
\bottomrule
\end{tabular}
\label{tbl:sfhsa-budget-equivalence}
      \end{table}
    \end{minipage}
  \end{minipage}
  \vspace{10pt}

 We now examine policies that consider different thresholds for different groups, thereby resulting in different budgets, $\pi_{\%,A} = \mathbb{I}[(p_{1\mid 1}(X,A)-p_{1\mid 0}(X,A))\tau(X,A) > q_A].$ Typically, $q_A$ is a quantile of the score distribution and therefore corresponds to some budget percentile. We let $F_A(p)=P(p_{1\mid 1}(X,A)-p_{1\mid 0}(X,A))\tau(X,A)\leq p) $ denote the cumulative distribution function (CDF) of the heterogeneous encouragement score, and therefore, the groupwise budget at some threshold $q_a$ is $F^{-1}_A(q_A)$. We now range over these thresholds separately for the different groups on the $x$-axis; the $y$-axis plots the average improvement in the allocated subgroup, $E[p_{1\mid 1}(X,A)-p_{1\mid 0}(X,A))\tau(X,A)\mid A,\pi_{\%,A}=1]$. In \Cref{tbl:sfhsa-budget-equivalence}, we illustrate for a few potential groupwise budget percentages of $A=0$ what improvement is achieved for $A=0$ and what percentage of the budget would be required for $A=1$ to achieve equal improvement. The equivalent budget allocations required to target the same percentage of the group population can differ by 20 to 25 percentage points. At thresholds $q=F_0^{-1}(p)$ corresponding to percentiles $p=50\%, 75\%$ and $90\%$ of the $A=0$ group, and therefore budgets of $1-p=50\%, 25\%$ and $10\%$, we observe average improvements of $\$12, \$21.97$ and $\$39.37$ in the $A=0$ group, and it would require an equivalent budget of $~72\%, 50\%$ and $29.2\%$ of the $A=1$ population to achieve equal improvements, respectively.  Given that 80\% of individuals are nonwhite (i.e., the $A=1$ group is four times the size of the $A=0$ group), substantially different groupwise budgets (in terms of number of people covered by the budget allocation) would be required to equalize improvements. 

\subsubsection{Summary and Conclusion} In this case study, we focus on assessing the heterogeneous compliance and treatment effects of text message reminders about recertification deadlines on 1) completion of any interview and 2) yearly benefits. 

\Cref{fig:sfhsa-density-race} and \Cref{fig:sfhsa-cdfplot-race} illustrate that this inequity arises because of the differential effectiveness of the reminder in improving the probability of interview attendance, suggesting either that a design with further communication could be helpful or that other, structural compliance barriers could persist. Proceeding with the analysis of reminder efficacy, we find that while racial disparities in current patterns of self-selection into encouragement are small, self-selection is still inefficient for targeting: Those who would benefit the most from receiving reminders do not sign up for them. We assess budget-optimal targeted allocations that can improve upon the outcome under self-selection with as little as 10\% of the budget. The heterogeneous effects of the reminder on interview completion, i.e., the differences that the nudge makes, are smaller for nonwhite beneficiaries, but this group generally receives higher benefits overall. 
Overall, any efficiency-targeted allocation rule improves groupwise outcomes (i.e., average 12-month benefits for white or nonwhite beneficiaries) relative to those in the no-reminder baseline. 

Our overall findings highlight the need for further research into the roots of the gaps in interview completion and suggest that investing in encouragement could reduce racial gaps and may be overall more effective than constraining allocations for fairness. We consider a hypothetical fairness adjustment in which we vary the encouragement budgets for groups or the thresholds at which groups receive the encouragement, finding that substantially different budgets across groups would be required to equalize improvements---although this could also be perceived as unfair. Overall, we find that naively encouragement-optimal allocations balance fair treatment (less inequality in budget spending) with equity in outcomes.

\subsection{Additional details on two-stage optimization simulation}\label{apx-two-stage}

Our method proceeds as follows: first, we solve the constrained optimization via grid search over Lagrange multipliers, as also suggested and implemented by \citep{agarwal2018reductions} in some of their experiments. We fix a grid of $\lambda$ multipliers beforehand (the grid can be expanded if constrained solutions occur at the edges).

We compare to a number of natural baselines for constrained optimization. We seek to solve, for a fixed $\epsilon >0$: 
$$ \max_\pi \{ V(\pi) \colon \abs{\E[T(\pi) \mid A=a] - \E[T(\pi) \mid A=b]} \leq \epsilon $$
We can obtain the range of feasible $\epsilon$ by solving two unconstrained optimization problems, one that maximizes policy value alone and another that minimizes disparity alone. 

The \textbf{naive} method optimizes the above without any further modification of the constraint, and appears in gray with square markers. As an ablation, the \textbf{one-stage} method solves the constraint with an upper bound of $(\epsilon - \frac{z_{0.9} \sigma_{\max}}{\sqrt{n}} )_+,$ where $z_{0.9}$ is a 90\% two-sided z-score for a standard normal random variable, $\sigma_{\max} = \sqrt{\max_\pi \{ \hat{Var}( \E[T(\pi) \mid A=a]-\E[T(\pi) \mid A=b]  )\}}$ is the worst-case standard deviation of treatment disparity values. It appears in red with triangle markers. Our two-stage methods follow the procedure described in \Cref{alg-metaalg2}. The \textbf{two-stage} method appears in dark purple with a circle marker. 
The sample splitting is required primarily for unbiased estimation of policy value, but in-sample debiasing methods could be further applied to increase effective sample complexity. We introduce a \textbf{two-stage-75\%} method which samples data splits of 75\% of the original dataset size, trading off some potential bias for variance reduction, appearing in light purple with diamond markers. Finally, another constrained formulation of interest is rather that seeking \textit{Pareto improvement} upon a policy, such as a prior deployed policy, or a first-stage optimal unconstrained policy. 

The panels in \Cref{fig:two-stage} indicate the out-of-sample policy regret (compared to a constrained policy optimized on an oracle large dataset), out-of-sample treatment disparity, coverage levels of out-of-sample treatment disparity $\leq \epsilon + z_{0.9} \sigma_{\pi^*} / \sqrt{n_{eff}}$, and the Lagrangian regret (measured relative to the effective $\epsilon_n$ from each method). The second panel of \Cref{fig:two-stage} also plots the value of $\epsilon=0.3$, as well as the effective $\epsilon_n$ series for each method; one-stage, 2-stage, and 2-stage-75\% in dashed gray lines. 
The last panel compares these different methods by scalarizing policy value and constraint violations with a fixed scalar factor. It measures the gap of out-of-sample scalarized policy value and disparity control at the desired $\epsilon$-level, while accounting for different fixed-sample conservative $\epsilon_n$ of different methods. See \Cref{apx-two-stage} for more details on this measure.
We find that the two-stage method is particularly helpful at moderate sample sizes $n=750-2000$ with enough samples to learn nuisances well on half the data, and room to improve via variance regularization.
\paragraph{Multi-objective regret measure.}

First for each different method, potentially with a different effective sample complexity, we define the best-possible policy obtained with population estimates of value and disparity, while taking the data-driven slack as the desired constraint level: 
\begin{align*}
 L(\epsilon_{n_{eff}}) &\coloneqq  \{ V(\pi^*(\lambda)) - \eta^*_{n_{eff}}(\abs{\E[T(\pi^*(\lambda)) \mid A=a]-\mathbb{E}[T(\pi^*(\lambda)) \mid A=b]} - \epsilon)_+ \} \\
       L^*(\epsilon_{n_{eff}}) &\coloneqq \max L(\epsilon_{n_{eff}}),\qquad \eta^*_{n_{eff}}\in \arg\max L(\epsilon_{n_{eff}}) \\
       \hat L(\eta^*_{n_{eff}}) &\coloneqq  V(\hat\pi^*) - \eta^*_{n_{eff}}(\abs{\E[T(\hat\pi^*) \mid A=a]-\mathbb{E}[T(\hat\pi^*) \mid A=b]} - \epsilon)_+.
\end{align*}
The fourth panel plots $    L^*(\epsilon_{n_{eff}}) - \hat L(\eta^*_{n_{eff}})$ for each method, which therefore measures the gap of out-of-sample scalarized policy value and disparity control at the desired $\epsilon$-level, while accounting for best-possible performance given the fixed-sample conservativity at different effective sample complexities. (Note that the oracle benchmark therefore changes in $n$, leading to non-monotonic behavior). We find that the two-stage method is particularly helpful at moderate sample sizes $n=750-2000$ with enough samples to learn nuisances well on half the data, and room to improve via variance regularization. At larger sample sizes, the sample splitting required for two-stage optimization's implicit variance regularization leads to slower convergence in general when one-stage optimization is sufficient. 

\subsection{Oregon Health Insurance Study}

The Oregon Health Insurance Study \citep{finkelstein2012oregon} is an important study on the causal effect of expanding public health insurance on healthcare utilization, outcomes, and other outcomes. It is based on a randomized controlled trial made possible by resource limitations, which enabled the use of a randomized lottery to expand Medicaid eligibility for low-income uninsured adults. Outcomes of interest included health care utilization, financial hardship, health, and labor market outcomes and political participation. 

Because the Oregon Health Insurance Study expanded access to \textit{enroll} in Medicaid, a social safety net program, the effective treatment policy is in the space of \textit{encouragement} to enroll in insurance (via access to Medicaid) rather than direct enrollment. This encouragement structure is shared by many other interventions in social services that may invest in nudges to individuals to enroll, tailored assistance, outreach, etc., but typically do not automatically enroll or automatically initiate transfers. Indeed this so-called \textit{administrative burden} of requiring eligible individuals to undergo a costly enrollment process, rather than automatically enrolling all eligible individuals, is a common policy design lever in social safety net programs. Therefore we expect many beneficial interventions in this consequential domain to have this encouragement structure.

We preprocess the data by partially running the Stata replication file, obtaining a processed data file as input, and then selecting a subset of covariates that could be relevant for personalization. These covariates include household information that affected stratified lottery probabilities, socioeconomic demographics, medical status and other health information. 

In the notation of our framework, the setup of the optimal/fair encouragement policy design question is as follows: 

\begin{itemize}
    \item $X$: covariates (baseline household information, socioeconomic demographics, health information) 
    \item $A$: race (non-white/white), or gender (female/male)

    These protected attributes were binarized. 
    \item $R$: encouragement: lottery status of expanded eligibility (i.e. invitation to enroll when individual was previously ineligible to enroll) 
    \item $T$: whether the individual is enrolled in insurance ever

    Note that for $R=1$ this can be either Medicaid or private insurance while for $R=0$ this is still well-defined as this can be private insurance. 
    \item $Y$: number of doctor visits

    This outcome was used as a measure of healthcare utilization. Overall, the study found statistically significant effects on healthcare utilization. An implicit assumption is that increased healthcare utilization leads to better health outcomes. 
\end{itemize}

We subsetted the data to include complete cases only (i.e. without missing covariates). We learned propensity and treatment propensity models via logistic regression for each group, and used gradient-boosted regression for the outcome model. We first include results for regression adjustment identification. One potential concern is the continued use of the healthcare utilization variable as an outcome measure. From a methodological angle, it displays heterogeneity in treatment effects. From the substantive angle, healthcare utilization remains a proxy outcome measure for other health measures, and interpreting increases in healthcare utilization as beneficial is justified primarily by assuming that individuals were constrained by the costs of uninsured healthcare previously, so that increases in healthcare utilization reflect that access to insurance increases in access to care.

\begin{figure}[t!]
    \centering
    \includegraphics[width=0.25\textwidth]{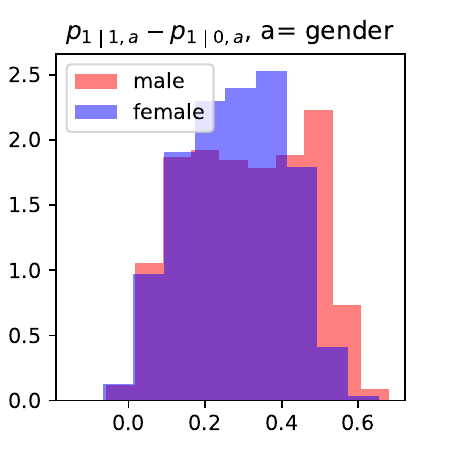}\includegraphics[width=0.25\textwidth]{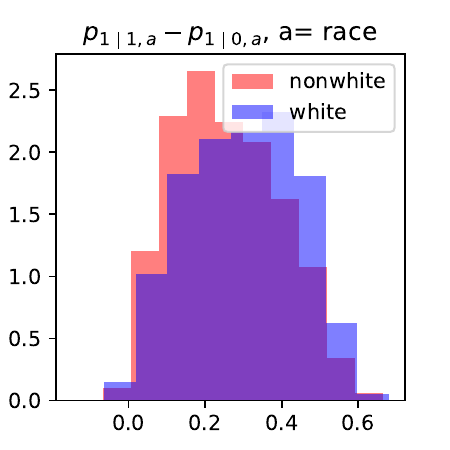}\includegraphics[width=0.25\textwidth]{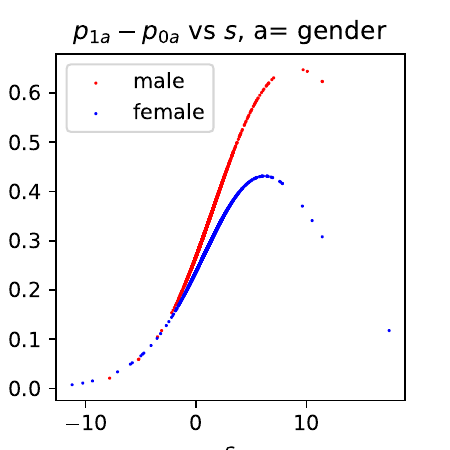}\includegraphics[width=0.25\textwidth]{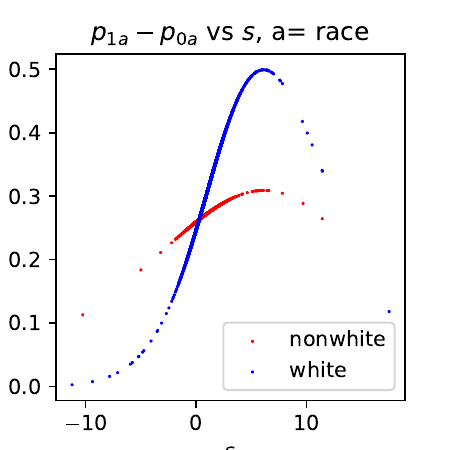}
    \caption{Distribution of lift in treatment probabilities $p_{1 \mid 1,a}-p_{1 \mid 0, a} =P(T=1\mid R=1,A=a,X)-P(T=1\mid R=0,A=a,X)$, and plot of $p_{1\mid 1,a}-p_{1\mid 0,a}$ vs. $\tau.$}
    \label{fig:oregon-desc}
        \includegraphics[width=0.25\textwidth]{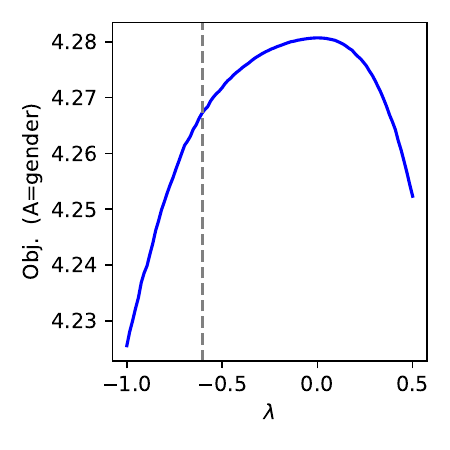}\includegraphics[width=0.25\textwidth]{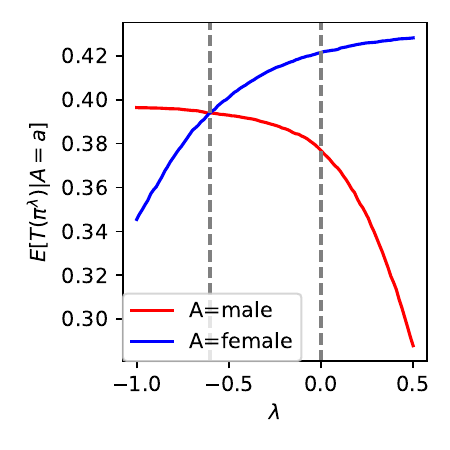}\includegraphics[width=0.25\textwidth]{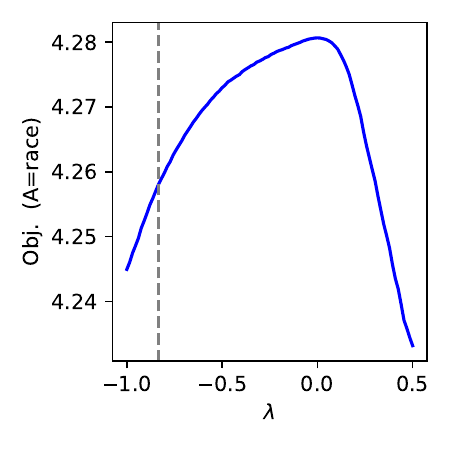}\includegraphics[width=0.25\textwidth]{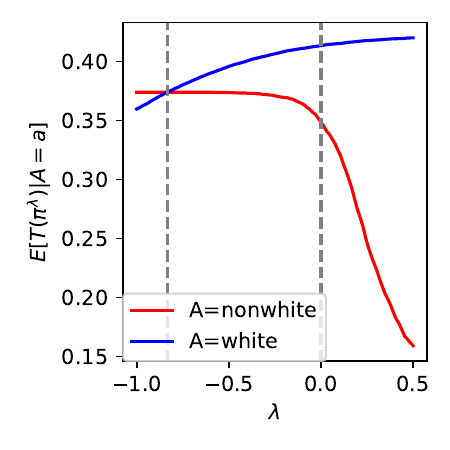}
        \vspace{-10pt}
    \caption{Local scalarized score $2\,\E[Y(\pi^\lambda)] + \E[T(\pi^\lambda)]$ (illustrative weights) and enrollment rate $\E[T(\pi^\lambda)\mid A=a]$, for $A$ = race, gender.}
\label{fig:oregon-objconstr}
\end{figure}
In \Cref{fig:oregon-desc} we plot descriptive statistics. We include histograms of the treatment responsivity lifts $p_{1 \mid 1a}(x,a)-p_{1 \mid 0 a}(x,a).$ We see some differences in distributions of responsivity by gender and race. We then regress treatment responsivity on the outcome-model estimate of $\tau$. We find substantially more heterogeneity in treatment responsivity by race than by gender: whites are substantially more likely to take up insurance when made eligible, conditional on the same expected treatment effect heterogeneity in increase in healthcare utilization. (This is broadly consistent with health policy discussions regarding mistrust of the healthcare system). 

Next we consider imposing treatment parity constraints on an unconstrained optimal policy (defined on these estimates). In \Cref{fig:oregon-objconstr} we plot a \emph{local} scalarized score that combines doctor visits and insurance enrollment,
\[
2\,\E[Y(\pi)] + \E[T(\pi)],
\]
where $Y$ is the number of doctor visits and $T$ indicates ever enrolling in insurance (Medicaid or private when encouraged); the weights $(2,1)$ are for illustration only and are not part of the paper-wide definition $V(\pi)=\E[Y(\pi)]$ used in the formal theory. We also plot $\E[T(\pi)\mid A=a]$ for gender and race, respectively, as $\lambda$ varies. Policies that improve group-conditional enrollment can be found with relatively little change in this scalarized score: disparities in $\E[T(\pi)\mid A=a]$ can be reduced by about $4$ percentage points ($0.04$) for gender and about $6$ percentage points $(0.06)$ for race while lowering $2\,\E[Y(\pi)] + \E[T(\pi)]$ by only about $0.01$--$0.02$ (in the same units, i.e., roughly one--two visits at the visit weighting of $2$). On the other hand, relative improvements and compromises in enrollment for the ``advantaged group'' show different tradeoffs. Plotting the tradeoff curve for race shows that, consistent with the large differences in treatment responsivity we see for whites, improving access for blacks. Looking at this disparity curve given $\lambda$ however, we can also see that small values of $\lambda$ can have relatively large improvements in access for blacks before these improvements saturate, and larger $\lambda$ values lead to smaller increases in access for blacks vs. larger decreases in access for whites. 

\subsection{Additional discussion, PSA-DMF case study}\label{apx-psa}
\begin{figure}
    \centering
    \includegraphics[width=0.25\textwidth]{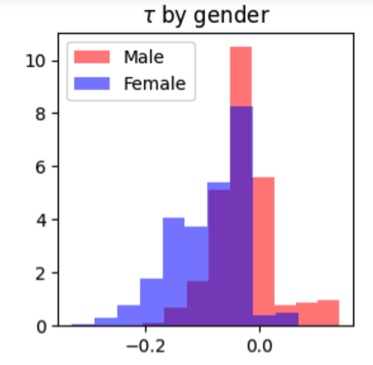}\includegraphics[width=0.25\textwidth]{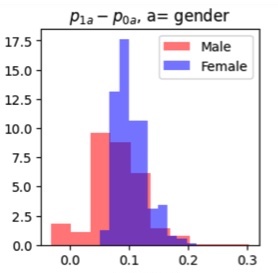}\includegraphics[width=0.25\textwidth]{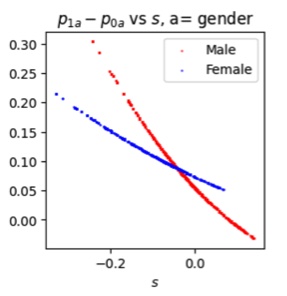}
    \caption{Distribution of treatment effect by gender, lift in treatment probabilities $p_{11a}-p_{01a} =P(T=1\mid R=1,A=a,X)-P(T=1\mid R=0,A=a,X)$, and plot of $p_{11a}-p_{01a}$ vs. $\tau.$}
    \label{fig:desc}

\end{figure}

First, before describing the analysis, we acknowledge important data issues (such as those that commonly arise from the criminal justice system \citep{bao2021s}), in addition to particularities of this data set, so that this analysis should be viewed as exploratory. 

\paragraph{Data setup}
Our analysis proceeds conditional on the non-detained population. This could make sense in a setting where decision-making frameworks for supervised release are unlikely to change judicial decisions to detain or release: our results apply to marginal defendants. Covariate levels (including PSA scores) were discretized for privacy. Moreover, the recorded final supervision decision does not include intensity, but different intensities are recommended in the data, which we collapse into a single level. The PSA-DMF is an algorithmic recommendation so here we are appealing to overlap in treatment recommendations, but using parametric extrapolation in responsivity. So, we are assuming randomness in treatment assignment that arises either from quasi-random assignment to judges or noise/variability in judicial decisions. We strongly appeal to this interpretation of randomness in treatment assignment in the conceptualization of a causal effect of treatment with supervised release. Other accounts and conceptualizations of judicial decision-making could instead argue that judicial decisions such as conditional release are by their very nature discretionary, and do not admit valid counterfactuals. We instead appeal to a hypothetical randomized experiment (if unethical) where individuals could conceivably be randomized into supervised release or not. Finally, unconfoundedness is likely untrue, but sensitivity analysis could address this in ways quite similar to those studied previously in the literature \citep{kallus2019interval,kallus2021minimax,kallus2018confounding}. 
\paragraph{Descriptives}

In \Cref{fig:desc}, we provide descriptive information illustrating heterogeneity (including by protected attribute) in adherence and effectiveness. We observe wide variation in judges' assignment of supervised release beyond the recommendation. We use logistic regression to estimate outcome and treatment response models. The first figure shows our estimates of the causal effect by gender (with similar heterogeneity by race). The outcome is failure to appear, so negative scores are beneficial. The second figure illustrates the difference in responsiveness: how much more likely decision-makers are to assign treatment when there is than when there is not an algorithmic recommendation to do so. The last figure plots a logistic regression of the lift in responsiveness on the causal effect $\tau(x,a) = \mu_1(x,a)-\mu_0(x,a)$. We observe disparities in how responsive decision-makers are conditional on the same treatment effect efficacy. 

\begin{figure}[t!]
    \centering
\includegraphics[width=0.35\textwidth]{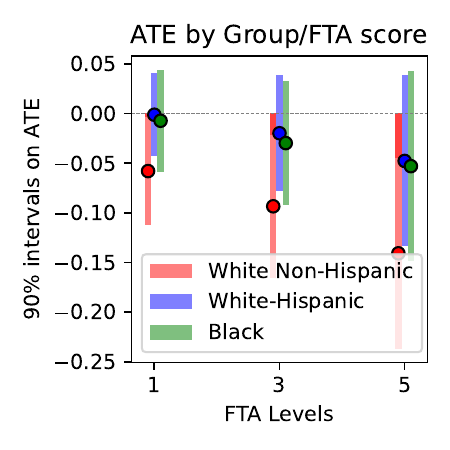}\includegraphics[width=0.55\textwidth]{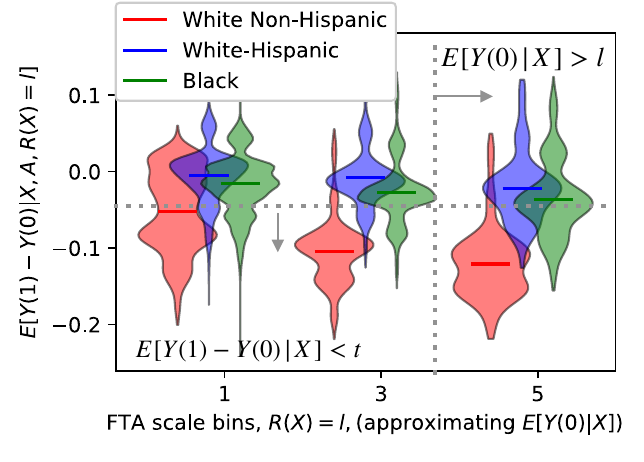}
    \caption{Estimated treatment effects of supervised release by race and FTA risk. (a) Average treatment effect (with 90\% confidence intervals) for different groups, stratified by FTA risk. (b) Distribution of heterogeneous treatment effects by group, stratified by FTA risk.}\label{fig:supervisedrelease-hte}
\end{figure}

In \Cref{fig:superviserelease-ate} we show estimates of the ATE \footnote{Obtained by augmented balancing weights with a logistic regression outcome model and ridge balancing weights.} and 90\% confidence intervals, per race group (excluding an ``Other" category due to smaller sample size) and stratified by underlying FTA risk level. Although the average treatment effect is weakly decreasing (more effective at reducing FTA) for riskier individuals, there is substantial discrepancy in the absolute magnitude of this effect by race. There is also potential heterogeneity in treatment effects. The figure on the right of \Cref{fig:supervisedrelease-ate} superimposes the distributions of estimated heterogeneous treatment effects (obtained by pseudo-outcome regression on the previous AIPW scores), again disaggregated by race and stratified on FTA risk level. Treatment effects are on average larger for white non-Hispanic than others (Black and white-Hispanic). There are many reasons this could be the case, such as different covariate distributions. 

The decision-making matrix thresholded its recommendations based on riskiness, as assessed via the PSA (\Cref{ex:justice}), in the absence of supervision. Although this can recall the riskiest individuals within races, and recall individuals with the largest treatment effects, this is at a targeting efficiency cost due to differences in absolute magnitude of treatment effects. The horizontal gray line in \Cref{fig:supervisedrelease-hte}
illustrates targeting based on heterogeneous treatment effects: choosing a threshold based on the difference that supervised release makes in reducing FTA.

\subsubsection{Additional details on interpretable scorecard optimization}\label{apx-addldisc-scorecard}

\paragraph{Overview of the optimization problem}
\citet{khan2024off} studies robust optimization under overlap via Lipschitz uncertainty sets and establishes a no-interaction property, such that the Lipschitz uncertainty sets between datapoints can be reduced to a per-datapoint interval uncertainty for an \textit{anchor-point} in the overlap region. We also impose the additional compliance monotonicity constraint that $p_{1\mid 1}(x)-p_{1\mid 0}(x)\ge 0,\;\forall x$. We set the Lipschitz constant width of the uncertainty set to $L=0.17,$ which is the largest effective Lipschitz constant observed on the overlap region. We collect these constraints on robust $p_{1\mid 0}(x),p_{1\mid 1}(x)$ in the product uncertainty set $\mathcal{P}_{\mathrm{wc}-\mathrm{cov}},$ see the appendix \Cref{apx-addldisc-scorecard} for more details and the full formulation.

Individuals $i=1,\ldots,n$ are assigned to LP cells $j(i)\in\mathcal{J}$ (a pretrial scorecard bin).

To preserve interpretability, we also require optimization over \textit{monotone scorecards}, i.e. recommendation within one $(FTA, NCA)$ pair requires recommendation for larger $(FTA, NCA)$ pairs in a partial order, a linear constraint. The scorecard polytope is therefore
\[
\Pi\coloneqq
\{
\pi=(\pi_{j})_{j\in\mathcal{J}}:
\ 0\le \pi_{j}\le 1,\ 
\E[\pi(\mathrm{FTA}, \mathrm{NCA})]\le B,\ 
\pi(i,j)\ge \pi(k,l)\ \text{whenever }(i,j) \succ \pi_{(k,l)}
\},
\]
where $m_{j}$ are cell masses, $B$ is a supervision budget cap, and $\succ$ is the fixed PSA partial order on the ordered risk levels (e.g.\ $j,k\in\{1,3,5\}$). Under the status quo policy, about $18\%$ are recommended stricter supervised release. 

We therefore solve the following local scalarization, which is specific to this application:
$$
\min_{\pi \in \Pi} \max _{p \in \mathcal{P}_{\mathrm{wc}-\mathrm{cov}}}\{ 100\,\E[Y(\pi)] + 20\,\E[T(\pi)]
+\lambda \Delta(\pi, p)
\} .
$$

\paragraph{Parametrizing the Lipschitz uncertainty sets}
For compactness in the scorecard LP below, we write $p_{10}(x):=p_{1\mid 0}(x)$ and $p_{11}(x):=p_{1\mid 1}(x)$, so the monotonicity constraint is equivalently $p_{11}(x)\ge p_{10}(x)$.
\citep{khan2024off} gives Lipschitz uncertainty intervals for each datapoint:
$$
\begin{aligned}
\ell_{10, i}:=0 \vee \max _{k \in \mathcal{O}}\left(\hat{p}_{10, k}-L d\left(x_i, x_k\right)\right), & u_{10, i}:=1 \wedge \min _{k \in \mathcal{O}}\left(\hat{p}_{10, k}+L d\left(x_i, x_k\right)\right), \\
\ell_{11, i}:=0 \vee \max _{k \in \mathcal{O}}\left(\hat{p}_{11, k}-L d\left(x_i, x_k\right)\right), & u_{11, i}:=1 \wedge \min _{k \in \mathcal{O}}\left(\hat{p}_{11, k}+L d\left(x_i, x_k\right)\right) .
\end{aligned}
$$

Due to the ordinal nature of the data, the problem is equivalently parameterized over problems with equivalent truncated interval overlap sets. Partition $\{1,\ldots,n\}$ into nonempty groups $g=1,\ldots,G$ such that all $i\in g$ share the same cell index $j(g)\coloneqq j(i)$ and the same partial-identification rectangle $[\ell_{10,g},u_{10,g}]\times[\ell_{11,g},u_{11,g}]$
for feasible compliance pairs $(p_{10},p_{11})$.
Let
\[
\mathcal{V}^{(g)}\coloneqq
\bigl\{(p_{10},p_{11})\in[\ell_{10,g},u_{10,g}]\times[\ell_{11,g},u_{11,g}]:\ p_{11}\ge p_{10}\bigr\},
\qquad
\mathcal{P}_{\mathrm{wc\text{-}cov}}\coloneqq \prod_{g=1}^{G}\mathcal{V}^{(g)}.
\]

\paragraph{Compliance monotonicity constraint }
In the supervised release case study, for robust optimization over Lipschitz uncertainty sets, we also impose an additional monotonicity constraint that $p_{11}(X)\ge p_{10}(X)$. This is nonetheless reasonable, as it corresponds to assuming that recommending stricter conditions for supervised release does not make a judge (stochastically) \textit{less} likely to impose it. As a result, the resulting optimization problem differs somewhat from our pointwise interval solutions used in \Cref{prop-robustlp}, as we have the intersection of the partial identification rectangle with the monotonicity constraint. Nonetheless, we take a vertex enumeration approach to enumerate over the vertices of the partial identification polytope, in order to arrive at the following equivalent finite linear program. 

\paragraph{Finite linear program (vertex reformulation).}
For each $g$, let $\mathcal{K}_{g}=\{(p_{0,g,k},p_{1,g,k})\}_{k=1}^{K_{g}}$ be the vertex set of the polygon $\mathcal{V}^{(g)}$ (equivalently, the extreme points of $\mathcal{V}^{(g)}$).
Nature chooses one pair $(p_{10}^{(g)},p_{11}^{(g)})\in\mathcal{V}^{(g)}$ per group, and we set $(p_{10,i},p_{11,i})=(p_{10}^{(g)},p_{11}^{(g)})$ for every $i\in g$.
\begin{align*}
\min_{\pi,\,d_{+},\,\{u_g,w_g,\zeta_g\}_{g=1}^{G}}\quad
& \frac{1}{n}\sum_{i=1}^{n}\hat\mu_{0,i}
+\sum_{g=1}^{G}\zeta_{g}
+\lambda d_{+}
\\
\text{s.t. }\quad
& \zeta_{g}\;\ge\;\sum_{i\in g}\frac{\hat\tau_{i}+\alpha}{n}
\left\{
(1-\pi_{j(g)})\,p_{0,g,k}+\pi_{j(g)}\,p_{1,g,k}
\right\},
\qquad  \forall\,g\in\{1,\ldots,G\},\ \forall\,k\in\{1,\ldots,K_{g}\},
\\
& u_{g}\;\ge\;
\sum_{i\in g}\left(\frac{A_i}{p_a}-\frac{1-A_i}{p_b}\right)p_{0,g,k}
+\pi_{j(g)}
\left(\sum_{i\in g}\left(\frac{A_i}{n_{\mathrm{nw}}}-\frac{1-A_i}{n_{\mathrm{w}}}\right)\right)
\left(p_{1,g,k}-p_{0,g,k}\right),
\qquad  \forall\,g,\ \forall\,k,
\\
& w_{g}\;\ge\;
-\sum_{i\in g}\left(\frac{A_i}{p_a}-\frac{1-A_i}{p_b}\right)p_{0,g,k}
-\pi_{j(g)}
\left(\sum_{i\in g}\left(\frac{A_i}{n_{\mathrm{nw}}}-\frac{1-A_i}{n_{\mathrm{w}}}\right)\right)
\left(p_{1,g,k}-p_{0,g,k}\right),
\qquad  \forall\,g,\ \forall\,k,
\\
& \sum_{g=1}^{G}u_{g}\le d_{+},\qquad
\sum_{g=1}^{G}w_{g}\le d_{+}, \qquad d_{+}\ge 0, \qquad \pi \in \Pi
\end{align*}

\paragraph{Results discussion}

\Cref{fig:scorecards-psa} displays the results. Each ``scorecard'' has FTA scores on the y-axis, and NCA scores on the x-axis, increasing as one goes further down or right. The optimization uses this local scalarization of FTA and supervision burden, with an additional $\lambda$-weighted objective of treatment disparity $\E[T(\pi) \mid A=a ] - \E[T(\pi) \mid A=b]$; each scorecard also includes an estimate of the worst-case disparity. The left figures on the top row include cost-optimal and fairness-penalized scorecards; by only recommending supervision to the individuals at highest-risk of FTA, even the cost-minimizing scorecard alone can control the worst-case disparity at $0.0246$ below the current status-quo realized values, $0.0572.$ The pooled scorecard is too simple for explicit fairness penalization to move further in the space of simple scorecards. The next row considers a slightly richer parametrization: different FTA-NCA scorecards depending on an age cutoff, which we optimize by enumerating over potential cutoffs, given that age was already discretized. Such scorecards remain interpretable, although blank cells refer to never-seen combinations of covariates. We find that cost-optimized and fair rules tend to withhold recommendations for older individuals with low NCA activity, improving worst-case disparity control to $0.0229$ or $0.0217$ with fairness penalization. It is important that in this situation, \textit{even cost-minimizing scorecards without explicit fairness penalty can improve worst-case disparities relative to the current status quo, subject to our robust Lipschitz extrapolation partial identification intervals}. By restricting the policy parametrization to match closely with the current status quo, our quantitative estimates of improvement shrink relative to our full policy optimization - on the other hand, we can provide concrete recommendations for local changes to status-quo decision-making matrices that could improve disparities at little cost in FTA. Optimizing over covariate-rich data-driven recommendation policies can mitigate fairness-accuracy concerns by providing Pareto-improving local improvements upon the status quo.

\begin{figure}[ht!]
    \centering
    \includegraphics[width=\linewidth]{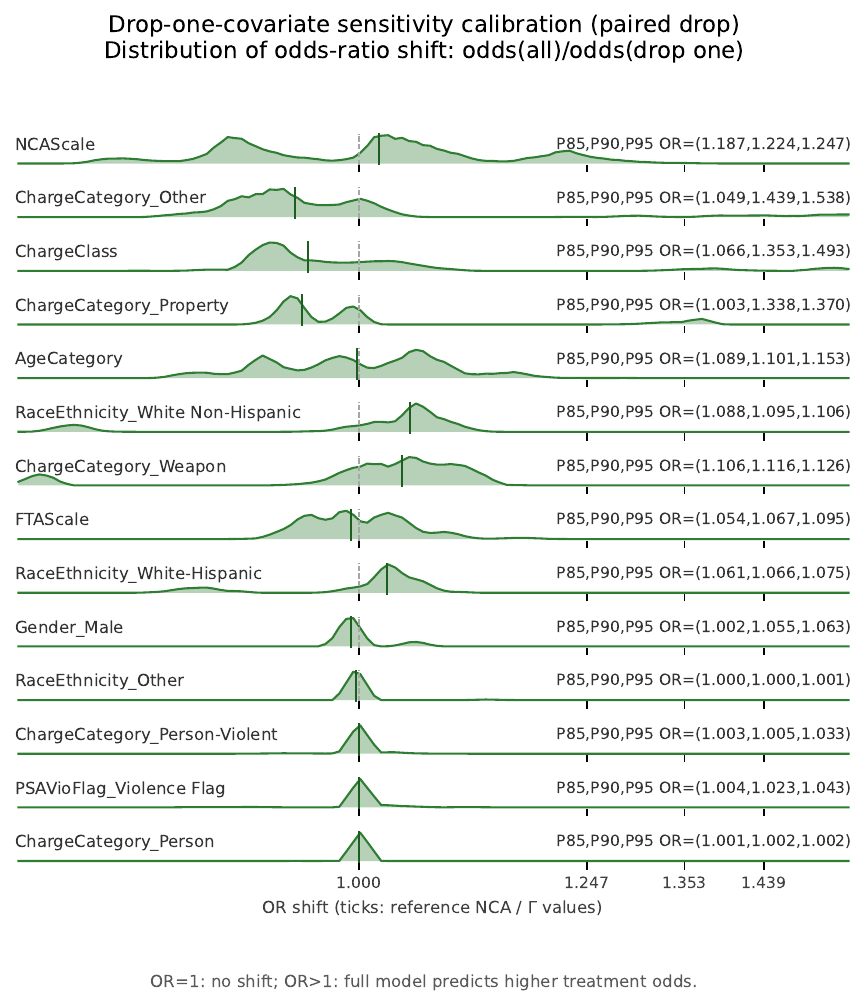}
    \caption{``Sparklines'' plot for calibrating size of uncertainty set of robustness to unobserved confounders. Density comparison of odds ratios induced by training propensities with dropped covariates (one per
line, in order). x-axis is the odds ratio, while y-axis (for each subplot) is a density plot of the odds ratios. Most informative categories include NCA, FTA, age, and charge category. }
    \label{fig:oddsratio-sparklines}
\end{figure}

\subsubsection{Sensitivity analysis: robustness checks under violations of unconfoundedness.}\label{apx-sensanalysis-unconf}

In \Cref{fig:oddsratio-sparklines} we see a sparklines plot to diagnose the odds-ratios that would be induced, if we had dropped each of the observed covariates, between selection into treatment based on all of the covariates vs. omitting each covariate. Such a plot is often used to interpret what different sizes of the uncertainty set mean, in relation to how informative \textit{currently observed} covariates are of selection into treatment. The most informative covariates of treatment are NCA, FTA, age and charge category. However, charge category has many discrete levels indicating different charges, where we see certain (rare) charges have an outsize effect on selection into treatment. It is not a good fit for the uniform control implied by the marginal sensitivity model and its $\Gamma$ parameter. Therefore, we turn to the most informative covariate, NCA, to guide our choice of $\Gamma.$ The $95\%$ quantile of odds ratios obtained by dropping NCA is 1.247. We choose $\Gamma=1.247$ as the size of uncertainty set.

\paragraph{Interval uncertainty sets on $\tau$ for the robust LP (MSM sensitivity).}
We use \citet{oprescu2023b} to obtain bounds on the heterogeneous treatment effect for the pretrial encouragement design using a B-Learner implementation
with logistic nuisances and a random-forest second stage. \citet{oprescu2023b} optimizes a marginal sensitivity model (MSM) odds-ratio bound $\Gamma\ge 1$. We follow our earlier specification with logistic regression propensity scores. Estimating robust bounds additionally requires quantile regression and ultimately we use a random forest regression for the final stage for the estimation of bounds.

\begin{figure}
    \centering
\includegraphics[width=\linewidth]{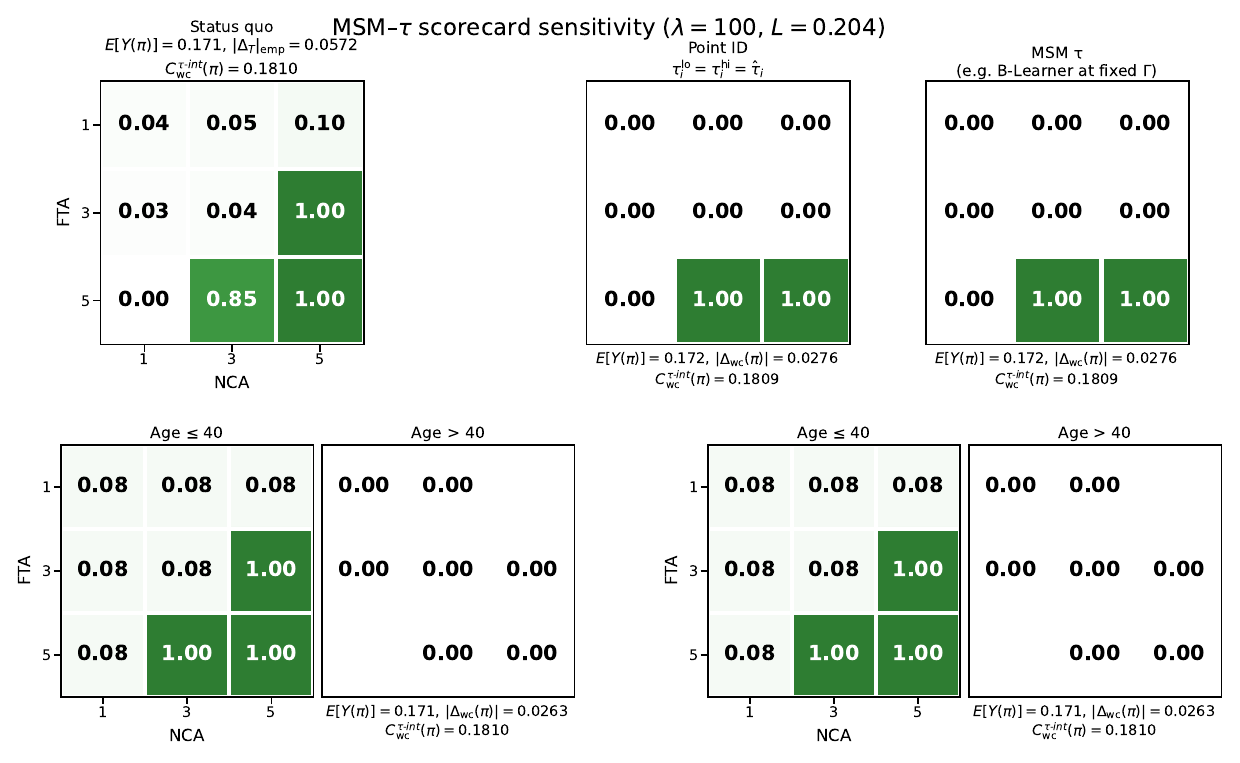}
    \caption{Interpretable supervised release scorecards optimized over uncertainty sets on $\tau$ alone with $\Gamma = 1.247$. Worst-case disparities assessed with $L=0.204$.}
    \label{fig:robust-scorecards}
\end{figure}

\end{APPENDICES}

\ACKNOWLEDGMENT{}

\end{document}